\documentclass{article}
\usepackage[letterpaper, left=1in, right=1in, top=1in, bottom=1in]{geometry}

\usepackage[breaklinks=true]{hyperref}
\hypersetup{
  colorlinks,
  citecolor=blue!70!black,
  linkcolor=blue!70!black,
  urlcolor=blue!70!black
}
\usepackage[algo2e,ruled]{algorithm2e}
\usepackage[round,compress]{natbib}

\usepackage[utf8]{inputenc} %
\usepackage[T1]{fontenc}    %

\usepackage{parskip}

\usepackage{amsmath}
\usepackage{amsxtra,amsfonts,amscd,amssymb,bm,bbm,mathrsfs}
\usepackage{nicefrac}       %
\usepackage{mathtools}
\usepackage{mathabx}
\newcommand\labelAndRemember[2]
  {\expandafter\gdef\csname labeled:#1\endcsname{#2}%
   \label{#1}#2}
\newcommand\recallLabel[1]
   {\csname labeled:#1\endcsname\tag{\ref{#1}}}
\newcommand\labelr[2]
  {\expandafter\gdef\csname labeled:#1\endcsname{#2}%
   \label{#1}#2}
\newcommand\recall[1]
   {\csname labeled:#1\endcsname}

\newcommand{\qed}{\hfill\ensuremath{\Box}}
\newenvironment{proof}{
\textbf{Proof.} 
}{
\qed
}
\newenvironment{proofof}[1]{
\textbf{Proof of #1.} 
}{
\qed
}

\usepackage{thmtools}

\newtheorem{theorem}{Theorem}[section]
\newtheorem{lemma}[theorem]{Lemma}
\newtheorem{definition}[theorem]{Definition}
\newtheorem{corollary}[theorem]{Corollary}
\newtheorem{proposition}[theorem]{Proposition}

\newtheorem{example}[theorem]{Example}
\newcommand{\exend}{{~$\Diamond$}}

\newtheorem{assumption}[theorem]{Assumption}
\newtheorem{conjecture}[theorem]{Conjecture}

\usepackage{url}
\usepackage[capitalize]{cleveref}
\usepackage[titletoc,title]{appendix}
\usepackage{cancel}
\usepackage{enumerate}
\usepackage{enumitem}
\usepackage{comment}
\crefformat{equation}{#2(#1)#3}
\Crefformat{equation}{#2(#1)#3}
\crefformat{assumption}{Assumption #2#1#3}
\crefformat{conjecture}{Conjecture #2#1#3}

\usepackage{algorithm,algorithm2e}

\SetCommentSty{mycommfont}

\usepackage{booktabs}       %
\usepackage{multirow}
\usepackage{multicol}
\usepackage{makecell}
\usepackage{array}
\newcolumntype{H}{>{\setbox0=\hbox\bgroup}c<{\egroup}@{}}
\newcolumntype{Z}{>{\setbox0=\hbox\bgroup}c<{\egroup}@{\hspace*{-\tabcolsep}}}

\usepackage[normalem]{ulem}
\usepackage{exscale}
\usepackage{tikz}
\usepackage{float}
\usepackage{graphicx,graphics}
\graphicspath{ {./fig/} }
\allowdisplaybreaks

\renewcommand{\hat}{\widehat}
\renewcommand{\epsilon}{\varepsilon}

\usepackage{pifont}

\newcommand{\id}{I}

\newcommand{\tM}{\tilde{M}}

\newcommand{\RR}{\mathbb{R}}
\newcommand{\NN}{\mathbb{N}}
\newcommand{\eps}{\varepsilon}

\newcommand{\II}{\mathbb{I}}
\newcommand{\EE}{\mathbb{E}}
\newcommand{\PP}{\mathbb{P}}
\newcommand{\MM}{\mathbb{M}}

\newcommand{\QQ}{\mathbb{Q}}

\newcounter{cnt}
\foreach \num in {1,2,...,26}{%
  \stepcounter{cnt}%
  \expandafter\xdef \csname c\Alph{cnt}\endcsname {\noexpand\mathcal{\Alph{cnt}}}%
  \expandafter\xdef \csname b\Alph{cnt}\endcsname {\noexpand\mathbb{\Alph{cnt}}}%
  \expandafter\xdef \csname f\Alph{cnt}\endcsname {\noexpand\mathfrak{\Alph{cnt}}}%
}

\DeclareMathOperator*{\argmin}{arg\,min}
\DeclareMathOperator*{\argmax}{arg\,max}

\newcommand{\diag}{\operatorname{diag}}

\newcommand{\leqsim}{\lesssim}
\newcommand{\geqsim}{\gtrsim}

\newcommand{\st}{\mathop{\textrm{s.t.}\ }}  
\newcommand{\T}{\top}  
\newcommand{\iprod}[2]{\left\langle #1, #2 \right\rangle}
\newcommand{\nrm}[1]{\left\|#1\right\|}

\newcommand{\abs}[1]{\left|#1\right|}

\newcommand{\brcond}[2]{\left[\left.#1\right|#2\right]}
\newcommand{\cond}[2]{\mathbb{E}\left[\left.#1\right|#2\right]}

\newcommand{\Om}[1]{\Omega\left(#1\right)}

\newcommand{\ceil}[1]{\left\lceil #1\right\rceil}
\newcommand{\floor}[1]{\left\lfloor #1\right\rfloor}
\DeclarePairedDelimiterX{\ddiv}[2]{(}{)}{%
  #1\;\delimsize\|\;#2%
}

\newcommand{\Unif}{\mathrm{Unif}}

\newcommand{\paren}[1]{{\left( #1 \right)}}
\newcommand{\brac}[1]{{\left[ #1 \right]}}
\newcommand{\set}[1]{{\left\{ #1 \right\}}}

\newcommand{\defeq}{:=}

\newcommand{\be}{\mathbf{e}}

\newcommand{\rank}{\mathrm{rank}}

\newcommand{\ms}{m^\star}

\newcommand{\Piall}{\Pi_{\rm RND}}
\newcommand{\Pimarkov}{\Pi_{\rm DM}}

\newcommand{\om}{\varpi}
\newcommand{\DB}[1]{D_{\rm B}\paren{#1}}
\newcommand{\delsep}{\delta}

\renewcommand{\sp}{s_\oplus}

\newcommand{\Asat}[1]{\bA_{{\sf sat},#1}}

\newcommand{\bo}{\mathbf{o}}
\newcommand{\bj}{\boldsymbol{j}}
\newcommand{\bx}{{\mathbf{x}}}
\newcommand{\by}{{\mathbf{y}}}
\newcommand{\bz}{{\mathbf{0}}}
\newcommand{\bk}{{\boldsymbol{k}}}
\newcommand{\bDelta}{\boldsymbol{\Delta}}
\newcommand{\ts}{\tilde{s}}

\newcommand{\perr}{e}
\newcommand{\belief}{\mathbf{b}}
\newcommand{\epssep}{\eps_{\rm s}}
\newcommand{\logNt}{\log N_{\Theta}}
\newcommand{\ths}{\theta^\star}

\newcommand{\modf}[1]{\mathsf{p}(#1)}
\newcommand{\modp}[1]{\phi(#1)}
\newcommand{\modq}[1]{\varphi(#1)}

\newcommand{\supp}{\mathrm{supp}}
\newcommand{\TT}{\mathbb{T}}
\newcommand{\poly}{\mathrm{poly}}

\newcommand{\hQ}{\hat{Q}}
\newcommand{\hV}{\hat{V}}
\newcommand{\hpi}{\hat{\pi}}

\newcommand{\ol}{\bar{l}}

\newcommand{\Wexp}{W_{\exp}}

\newcommand{\hPP}{\hat{\PP}}

\renewcommand{\O}{\mathbb{O}}
\renewcommand{\T}{\mathbb{T}}
\newcommand{\tcS}{\Tilde{\cS}}

\newcommand{\Nt}{N_{\Theta}}
\newcommand{\betaN}{2\logNt(1/T)+2\log(1/p)+2}

\newcommand{\pitau}[1]{\pi|_{\tau_{#1}}}
\newcommand{\mthtau}{m_{\theta}(\otau_W)}
\newcommand{\mothtau}{m_{\otheta}(\otau_W)}
\newcommand{\mthstau}{m_{\ths}(\otau_W)}
\newcommand{\oW}{\Bar{W}}
\newcommand{\MMop}[5]{\MM^{#1}_{#2,#4}(#5,#3)}
\newcommand{\MMth}[1]{\MM^\theta_{\mthtau,\oW}(#1,s_W)}
\newcommand{\MMoth}[1]{\MM^\otheta_{\mothtau,\oW}(#1,s_W)}
\newcommand{\perrtau}[1]{e_{#1}(\otau_W)}
\newcommand{\perrtauh}[1]{e_{#1}(\otau_h)}

\newcommand{\tmu}{\Tilde{\mu}}

\newcommand{\norm}[1]{\left\|{#1}\right\|} %
\newcommand{\lone}[1]{\norm{#1}_1} %
\newcommand{\ltwo}[1]{\norm{#1}_2} %
\newcommand{\linf}[1]{\norm{#1}_\infty} %

\renewcommand{\cO}{\mathcal{O}}
\newcommand{\tO}{\widetilde{\cO}}

\newcommand{\indic}[1]{\mathbf{1}\left\{#1\right\}} %

\renewcommand{\bQ}{\mathbf{Q}}

\newcommand{\otau}{\overline{\tau}}
\renewcommand{\II}{\mathbbm{1}}

\newcommand{\unif}{{\rm Unif}}

\newcommand{\otheta}{{\bar{\theta}}}

\newcommand{\dH}{D_{\mathrm{H}}}
\newcommand{\dTV}{D_{\mathrm{TV}}}

\renewcommand{\DH}[1]{D_{\mathrm{H}}^2\left(#1\right)}

\newcommand{\DTV}[1]{D_{\mathrm{TV}}\left(#1\right)}
\newcommand{\DTVt}[1]{D_{\mathrm{TV}}^2\left(#1\right)}

\newcommand{\expl}{\mathrm{sep}}

\newcommand{\doac}{\mathrm{do}}
\usetikzlibrary{matrix,decorations.pathreplacing,calc}

\newcommand{\tPP}{\widetilde{\PP}}

\newcommand{\ba}{\mathbf{a}}

\newcommand{\piexp}{\pi_{\expl}}

\newcommand{\BB}{\mathbf{B}}
\newcommand{\bq}{\mathbf{q}}

\newcommand{\pis}{\pi_\star}

\usepackage{amsmath,amsfonts,bm}

\def\eqref#1{equation~\ref{#1}}

\def\ceil#1{\lceil #1 \rceil}
\def\floor#1{\lfloor #1 \rfloor}
\def\1{\bm{1}}

\def\eps{{\epsilon}}

\DeclareMathAlphabet{\mathsfit}{\encodingdefault}{\sfdefault}{m}{sl}
\SetMathAlphabet{\mathsfit}{bold}{\encodingdefault}{\sfdefault}{bx}{n}

\def\tM{{\tens{M}}}

\def\tR{{\tens{R}}}

\def\tT{{\tens{T}}}

\def\sR{{\mathbb{R}}}

\newcommand{\E}{\mathbb{E}}

\newcommand{\R}{\mathbb{R}}

\usepackage[suppress]{color-edits}
\addauthor{sr}{orange}
\addauthor{cf}{red}
\addauthor{ng}{purple}

\newcommand{\sepstr}{$\delta$-strongly separated}
\newcommand{\twopol}{two-policy}

\title{Near-Optimal Learning and Planning in Separated Latent MDPs}
\usepackage{times}
\author{
Fan Chen\\{\small MIT}\\{\small\texttt{fanchen@mit.edu}} \and Constantinos Daskalakis\\{\small MIT \& Archimedes AI}\\{\small\texttt{costis@csail.mit.edu}} \and Noah Golowich\\{\small MIT}\\{\small \texttt{nzg@mit.edu}} \and Alexander Rakhlin \\\small{MIT}\\{\small \texttt{rakhlin@mit.edu}}
}

\begin{document}

\maketitle

\begin{abstract}%
We study computational and statistical aspects of learning  Latent Markov Decision Processes (LMDPs). In this model, the learner interacts with an MDP drawn at the beginning of each epoch from an unknown mixture of MDPs. To sidestep known impossibility results, we consider several notions of separation of the constituent MDPs. The main thrust of this paper is in establishing a nearly-sharp \textit{statistical threshold} for the horizon length necessary for efficient learning. 
On the computational side, we show that under a weaker assumption of separability under the optimal policy, there is a quasi-polynomial algorithm with time complexity scaling in terms of the statistical threshold. We further show a near-matching time complexity lower bound under the exponential time hypothesis.
\end{abstract}

\section{Introduction}

Reinforcement Learning~\citep{kaelbling1996reinforcement,sutton2018reinforcement} captures the common challenge of learning a good policy for an agent taking a sequence of actions in an unknown, dynamic environment, whose state transitions and reward emissions are influenced by the actions taken by the agent. Reinforcement learning has recently contributed to several headline results in Deep Learning, including Atari~\citep{mnih2013playing}, Go~\citep{silver2016mastering}, and the development of Large Language Models~\citep{christiano2017deep,stiennon2020learning,ouyang2022training}. This practical success has also sparked a burst of recent work on expanding its algorithmic, statistical and learning-theoretic foundations, towards bridging the gap between theoretical understanding and practical success.

In general, the agent might not fully observe the state of the environment, instead having imperfect observations of its state. Such a setting is captured by the general framework of Partially Observable Markov Decision Processes (POMDPs)~\citep{smallwood1973optimal}.  In contrast to the fully-observable special case of Markov Decision Processes (MDPs)~\citep{bellman1957markovian}, the setting of POMDPs is rife with statistical and computational barriers. In particular, there are exponential sample lower bounds for learning an approximately optimal policy~\citep{krishnamurthy2016pac,jin2020sample}, and it is PSPACE-hard to compute an approximately optimal policy even when the transition dynamics and reward function are known to the agent~\citep{papadimitriou1987complexity,littman1994memoryless,burago1996complexity,lusena2001nonapproximability}. In view of these intractability results, a fruitful research avenue has been to identify conditions under which statistical and/or computational tractability can be resurrected. This is the avenue taken in this paper.

In particular, we study {\em Latent Markov Decision Processes (LMDPs)}, a learning setting wherein, as its name suggests, prior to the agent's  interaction with the environment over an episode of~$H$ steps, nature samples an MDP, i.e.~the state transition dynamics and the reward function, from a distribution~$\rho$ over MDPs, which share the same state and action sets. The learner can fully observe the state, but cannot observe which MDP was sampled, and she also does not know the distribution~$\rho$. However, she can interact with the environment over several episodes for which, at the beginning of each episode, a fresh MDP is independently sampled from~$\rho$. The learner's goal is to learn a policy that optimizes her reward in expectation when this policy is used on a random MDP sampled from~$\rho$.

LMDPs are a special case of (overcomplete) POMDPs,\footnote{Indeed, if $\cal S$ is the state space shared by all MDPs in the support $\cal M$ of the distribution $\rho$ over MDPs, we may view this LMDP as a POMDP with state space ${\cal S} \times {\cal M}$. The state transition dynamics of this POMDP only allow transitions from state $(s,m)$ to state $(s',m')$ when $m=m'$, and the transition probability from $(s,m)$ to $(s',m)$ on action $a$ is determined by the transition probability from $s$ to $s'$ on action $a$ in MDP $m$. The observation model of this POMPD drops $m$ when observing the state $(s,m)$, and the initial state $(s_0,m)$ is sampled by first sampling $m \sim \rho$, and then sampling $s_0$ from the initialization distribution of MDP $m$.} which capture many natural scenarios. For example, learning in an LMDP can model the task facing a robot that is moving around in a city but has no sensors to observe the weather conditions each day, which affect the pavement conditions and therefore the dynamics. Other examples include optimizing the experience of users drawn from some population in a web platform~\citep{hallak2015contextual}, optimizing the outcomes of patients drawn from some population in healthcare provision~\citep{steimle2021multi}, and developing an optimal strategy against a population of possible opponents in a dynamic strategic interaction~\citep{wurman2022outracing}. More broadly, LMDPs and the challenge of learning in LMDPs have been studied in a variety of settings under various names, including hidden-model MDPs~\citep{chades2012momdps}, multi-task RL~\citep{brunskill2013sample,liu2016pac}, contextual MDPs~\citep{hallak2015contextual}, hidden-parameter MDPs~\citep{doshi2016hidden}, concurrent MDPs~\citep{buchholz2019computation}, multi-model MDPs~\citep{steimle2021multi}, and latent MDPs~\citep{kwon2021rl, zhan2022pac,chen2022partially,zhou2023horizon}.

Despite this work, we lack a complete understanding of what conditions enable computationally and/or sample efficient learning of optimal policies in LMDPs. We {\em do} know that some conditions must be placed, as in general, the problem is both computationally and statistically intractable. Indeed, it is known that an exponential number of episodes in the size $L$ of the support of $\rho$, is necessary to learn an approximately optimal policy~\citep{kwon2021rl}, and even when the LMDP is known, computing an optimal policy is PSPACE-hard~\citep{steimle2021multi}. 

A commonly studied and intuitively simpler setting, which is a main focus of this paper, is that of  {\em \sepstr~LMDPs}, where every pair of MDPs in the support of $\rho$ are $\delta$-separated in the sense that for every state-action pair their transition distributions differ by at least $\delta$ in total variation distance. 
Even in this setting, however, we lack a sharp characterization of the horizon length that is necessary and sufficient for sample-efficient learning. %
Previous works either require a very long horizon\footnote{Even under such a long horizon, \citet{brunskill2013sample,hallak2015contextual,liu2016pac} have to require additional restrictive assumptions, e.g. the diameter of each MDP instance is bounded.} (i.e.~$H\gg SA$, \citet{brunskill2013sample,hallak2015contextual,liu2016pac}) %
or impose extra assumptions on the predictive state representation of the underlying LMDP~\citep{kwon2021rl}.%
Other simplifying assumptions that have been studied include hindsight observability, i.e.~observability of the index of the sampled MDP at the end of  each episode, under which near-optimal regret guarantees have been obtained in certain parameter regimes~\citep{kwon2021rl,zhou2023horizon}, as well as test-sufficiency~\citep{zhan2022pac,chen2022partially} and decodability~\citep{efroni2022provable}, but here the known sample complexity bounds scale exponentially with the test-sufficiency/decodability window.

\paragraph{Our Contributions.} In this paper, we nearly settle the challenge of learning in \sepstr~LMDPs, by providing a near-sharp characterization of the horizon length necessary for efficient learnability. 

Our lower bound (\cref{thm:lower-demo}) shows that, for there to be an algorithm that learns an $\eps$-optimal policy in a \sepstr~LMDP from a polynomial number of samples, it must be that the horizon scales as
\begin{align*}
    H\geqsim \frac{\log(L/\eps)}{\delsep^2},
\end{align*}
where $L$ is the number of MDPs in the mixture. The threshold $H_\star\asymp \frac{\log(L/\eps)}{\delta^2}$ has a fairly intuitive interpretation: when $H\geq H_\star$, we can use the history up to step $H_\star$ to recover the unobservable index of the underlying MDP instance with error probability at most $\eps$ (\cref{prop:latent-MLE}).

We complement our lower bound by proposing a sample-efficient algorithm (\cref{alg:OMLE}) for learning an $\eps$-optimal policy in a $\delsep$-strongly separated LMDP when
\begin{align*}
    H\geqsim \frac{\log(LS/\eps\delta)}{\delta^2}.
\end{align*}
Our sample complexity guarantees also hold beyond the strong separation condition. We study the setting where the MDP instances are separated under \emph{every policy} (\cref{sec:stat-upper}), a condition that is comparably less restrictive than the strong separation condition. We  relax this  separation assumption even further to separation under an \emph{optimal policy}, although we need to make some extra assumptions in this case to preserve sample-efficiency (\cref{sec:single-policy}).

As a further application, we consider learning $N$-step decodable LMDPs, which is a natural class of structured LMDPs where strong separation does not hold. For such a class of LMDPs, we provide a sample-efficiency guarantee when $H\geq 2N$, and we also provide a lower bound which shows that this threshold is \emph{sharp}.

Finally, we study the computational complexity of computing an optimal policy in a known separated LMDP, i.e. the problem of planning. We show that the threshold $H_\star$ tightly captures the time complexity of planning: it gives rise to a natural planning algorithm (\cref{alg:plan}) with near-optimal time complexity under the exponential time hypothesis (ETH).

\subsection{Related works}

\paragraph{Planning in partially observable environment.}
Planning in a known POMDP has long been known to be PSPACE-compete~\citep{papadimitriou1987complexity, littman1994memoryless,burago1996complexity, lusena2001nonapproximability}, and planning in LMDP inherits such hardness~\citep{chades2012momdps,steimle2021multi}. 
The recent work of~\citet{golowich2022planning, golowich2022learning} established a property called ``belief contraction'' in POMDPs under an observability condition \citep{even2007value}, which leads to algorithms with quasi-polynomial statistical and computational efficiency. 

\paragraph{Learning in partially observable environment.}
It is well-known that learning a near-optimal policy in an unknown POMDP is statistically hard in the worst-case: in particular, the sample complexity must scale at least exponentially in the horizon~\citep{liu2022partially,krishnamurthy2016pac}. Algorithms achieving such upper bounds are developed in~\citep{kearns1999approximate, even2005reinforcement}. 
Under strong assumptions, such as full-rankness of the transition and observation matrices or availability of exploratory data, several algorithms based on spectral methods~\citep{hsu2012spectral, azizzadenesheli2016reinforcement, guo2016pac,xiong2021sublinear} and posterior sampling \citep{jahromi2022online} have also been proven to be sample-efficient. However, due to the nature of their strong assumptions, these works fall short of addressing the challenge of exploration in an unknown partially observable environment.

Towards addressing this challenge, a line of recent works proposed various structural problem classes that can be learned sample-efficiently, including reactive POMDPs~\citep{jiang2017contextual}, revealing POMDPs~\citep{jin2020sample, liu2022partially, liu2022sample}, low-rank POMDPs with invertible emission operators~\citep{cai2022reinforcement, wang2022embed}, decodable POMDPs~\citep{efroni2022provable}, regular PSRs~\citep{zhan2022pac}, reward-mixing MDPs%
~\citep{kwon2021reinforcement,kwon2023reward}, PO-bilinear classes~\citep{uehara2022provably}, POMDPs with deterministic latent transition~\citep{uehara2022computationally}, and POMDPs with hindsight observability \citep{lee2023learning}. Based on the formulation of predictive state representation (PSR), \citet{chen2022partially,liu2022optimistic} proposed (similar) unified structural conditions which encompass most of these conditions, with a unified sample-efficient algorithm Optimistic Maximum Likelihood Estimation (OMLE). %
As LMDPs are a subclass of POMDPs, all of these results can be applied to LMDPs to provide structural conditions that enable learnability. However, when instantiated to LMDPs, these structural conditions are less intuitive, and in general they are incomparable to our separability assumptions and do not capture the full generality of the latter.

\paragraph{RL with function approximation.} RL with general function approximation in fully observable environment has been extensively investigated in a recent line of work \citep[etc.]{jiang2017contextual,sun2019model,du2021bilinear,jin2021bellman,foster2021statistical,agarwal2022model,chen2022unified,xie2022role,liu2023optimistic}, and some of the proposed complexity measures and algorithms (e.g. Model-based Optimistic Posterior Sampling \citep{agarwal2022model, chen2022unified}, and Estimation-to-Decision \citep{foster2021statistical}) also apply to partially observable RL. In this work, our analysis of OMLE utilizes several tools developed in \citet{liu2022partially,chen2022unified,chen2022partially,xie2022role}.

\section{Preliminaries}\label{sec:prelim}

\paragraph{Latent Markov Decision Process.} An LMDP $M$ is specified by a tuple $\set{\cS,\cA,(M_m)_{m=1}^L,H,\rho,R}$, where $M_1,\cdots,M_L$ are $L$ MDP instances with joint state space $\cS$, joint action space $\cA$, horizon $H$, and $\rho\in\Delta([L])$ is the mixing distribution over $M_1,\cdots,M_L$, and $R=(R_h:\cS\times\cA\to [0,1])_{h=1}^H$ is the reward function. For $m\in[L]$, the MDP $M_m$ is specified by $\TT_{m}:\cS\times\cA\to\Delta(\cS)$ along with the initial state distribution $\nu_m\in\Delta(\cS)$. In what follows, we will parametrize each LMDP by a parameter $\theta$ (\cref{sec:model-based}), but for now we provide a few definitions without overburdening the notation.

In an LMDP, the latent index of the current MDP is hidden from the agent: the agent can only see the resulting transition trajectory. Formally speaking, at the start of each episode, the environment randomly draws a latent index $\ms\sim \rho$ (which is unobservable) and an initial state $s_1\sim \nu_{\ms}$, and then at each step $h$, after the agent takes action $a_h$, the environment generates the next state $s_{h+1}\sim \TT_{\ms}(\cdot|s_h,a_h)$ following the dynamics of MDP $M_{\ms}$. The episode terminates immediately after $a_{H}$ is taken. %

\paragraph{Policies.} A policy $\pi = \{\pi_h: (\cS\times\cA)^{h-1}\times\cS\to\Delta(\cA) \}_{h \in [H]}$ is a collection of $H$ functions. At step $h\in[H]$, an agent running policy $\pi$ observes the current state $s_h$ and takes action $a_{h}\sim \pi_h(\cdot|\otau_{h})\in\Delta(\cA)$ based on the whole history $\otau_h=(\tau_{h-1},s_h)=(s_1,a_1,\dots,s_{h-1},a_{h-1},s_h)$.  (In particular, we have written $\tau_{h-1} = (s_1, a_1, \ldots, s_{h-1}, a_{h-1})$.) 
The policy class $\Piall$ is the set of all such history-dependent policies, and $\Pimarkov$ is the set of all deterministic Markov policies, namely tuples $\pi = \{ \pi_h : \cS \to \cA \}_{h \in [H]}$. 

For any policy $\pi\in\Piall$, the interaction between $\pi$ and the LMDP $M$ induces a distribution $\PP^{\pi}$ of the whole trajectory $\tau_{H}=(s_1,a_1,\cdots,s_H,a_H)$. The value of $\pi$ is defined as
\begin{align*}
    V(\pi)=\EE^{\pi}\brac{\sum_{h=1}^H R_h(s_h,a_h)}.
\end{align*}
We also use $\tPP^{\pi}$ to denote the joint probability distribution of the latent index $\ms$ and trajectory $\tau_H$ under policy $\pi$.

\newcommand{\Bhd}{Bhattacharyya divergence}

\paragraph{Miscellaneous notations}
For probability distributions $p,q$ on a discrete measure space $\cX$, the Hellinger distance and Bhattacharyya divergence are defined as
\begin{align*}
    \textstyle
    \DH{ p,q } \defeq \frac{1}{2}\sum_{x\in\cX} (\sqrt{p(x)}-\sqrt{q(x)})^2, \qquad
    \DB{p,q}=-\log \sum_{x\in\cX} \sqrt{p(x)q(x)}.
\end{align*}
For expression $f,g$, we write $f\leqsim g$ if there is an absolute constant $C$ such that $f\leq Cg$. We also use $f=\cO(g)$ to signify the same thing.

\subsection{Strong separation and separation under policies}
In this section we introduce the various notions of separability we consider in this paper. 
\begin{definition}[{Strong separation, \citet{kwon2021rl}}]
\label{def:strong-sep}
    An LMDP is $\delsep$-strongly separated if for all $m,l\in\supp(\rho)$ such that $m\neq l$, 
    \begin{align*}
        \DTV{ \TT_{m}(\cdot|s,a), \TT_{l}(\cdot|s,a) }\geq \delsep, \qquad \forall s\in\cS, a\in\cA.
    \end{align*}
\end{definition}

\begin{definition}[Decodability, \citet{efroni2022provable}]
\label{def:decode}
An LMDP $M$ is $N$-step decodable if for any trajectory $\otau_N=(s_1,a_1,\cdots,s_N)$, there is at most one latent index $m\in\supp(\rho)$ such that $\otau_N$ is reachable starting from $s_1$ in the MDP instance $M_m$ (i.e., the probability of observing $s_2,\cdots,s_N$ in $M_m$ starting at $s_1$ and taking actions $a_1,\cdots,a_{N-1}$ is non-zero). In other words, there exists a \emph{decoding function} $\phi_M$ that maps any reachable trajectory $\otau_N$ to the latent index $m$. %
\end{definition}

More generally, we can consider separability under the induced distributions over a trajectory. For any policy $\pi$, we define
\begin{align}
    \MM_{m,h}(\pi,s)\defeq \brac{ \TT_m^{\pi}((a_1,s_2,\cdots,a_{h-1},s_h)=\cdot|s_1=s) } \in \Delta((\cA\times\cS)^{h-1}),\label{eq:mm-def}
\end{align}
where $\TT_m^{\pi}$ is the probability distribution of the trajectory in the MDP instance $M_m$ and under policy $\pi$.

For any increasing function $\om:\NN\to\R$, we can define $\om$-separation as follows, which requires that the separation between any two MDP instances grow as $\om$.

\begin{definition}[Separation with respect to a policy]\label{def:sep-weak}
    An LMDP is $\om$-separated under $\pi$ if for all $m,l\in\supp(\rho)$ such that $m\neq l$,
    \begin{align*}
        \DB{ \MM_{m,h}(\pi,s), \MM_{l,h}(\pi,s) }\geq \om(h), \qquad \forall h\geq 1,s\in\cS.
    \end{align*}
\end{definition}

We also define $\om^{-1}(x)\defeq \min\set{h\geq 1: \om(h)\geq x}$. In \cref{sec:stat-upper}, we show that if the LMDP is $\om$-separated under \emph{all policies} and $H\geqsim \om^{-1}(\log(\text{problem parameters}))$, then a near-optimal policy can be learned sample-efficiently.

In particular, strong separation indeed implies separation under all policies.

\begin{proposition}\label{prop:strong-sep-to-weak}
    If the LMDP $M$ is $\delta$-strongly separated, then it is $\om_{\delta}$-separated under any policy $\pi\in\Piall$, where $\om_{\delta}(h)=\frac{\delta^2}{2}(h-1)$.
\end{proposition}

\begin{proposition}\label{prop:decode-to-weak}
    The LMDP $M$ is $N$-step decodable if and only if it is $\om_N$-separated under all policy $\pi\in\Piall$, where $\om_N(h)=\begin{cases}
        0, & h<N, \\
        \infty, & h\geq N.
    \end{cases}$
\end{proposition}

The proof of \cref{prop:decode-to-weak} is provided in \cref{appdx:proof-strong-sep-to-weak}. More generally, the following lemma gives a simple criteria for all-policy separation.

\begin{lemma}\label{lem:markov-to-all}
If an LMDP is $\om$-separated under any Markov policy $\pi\in\Pimarkov$, then it is $\om$-separated under any general policy $\pi\in\Piall$.
\end{lemma}

\subsection{Model-based function approximation}\label{sec:model-based}

In this paper, we consider the standard model-based learning setting, where we are given an LMDP model class $\Theta$ and a policy class $\Pi\subseteq \Piall$. Each $\theta\in\Theta$ parameterizes an LMDP $M_\theta=\set{\cS,\cA,(M_{\theta,m})_{m=1}^L,H,\rho_\theta,R}$, where the state space $\cS$, action space $\cA$, horizon $H$, integer $L$ representing the number of MDPs, and reward function $R$ are shared across all models, $\rho_\theta$ specifies the mixing weights for the $L$ MDP instances under $\theta$, and the MDP instance $M_{\theta,m}$ is specified by $(\TT_{\theta,m}, \nu_{\theta,m})$ for each $m\in[L]$. 
For each model $\theta\in\Theta$ and policy $\pi\in\Piall$, we denote $\PP_\theta^\pi$ to be the distribution of $\tau_H$ in $M_\theta$ under policy $\pi$, and let $V_\theta(\pi)$ be the value of $\pi$ under $M_\theta$.

We further assume that (a) the ground truth LMDP is parameterized by a model $\theta^\star \in \Theta$ (\emph{realizability}); (b) the model class $\Theta$ admits a bounded log covering number $\log\Nt(\cdot)$ (\cref{def:optimistic-cover}); %
(c) the reward function $R$ is known and bounded, $\sum_{h=1}^H \sup_{s,a}R_h(s,a)\leq 1$. \footnote{For simplicity, we only consider deterministic known reward in this paper. For random reward $r_{h}\in\set{0,1}$ that possibly depends on the latent index $m$, we can consider the ``augmented'' LMDP with the augmented state $\tilde{s}_{h+1}=(s_{h+1},r_{h})$ similar to \citet{kwon2021rl}.} %

In addition to the assumptions stated above, we also introduce the following assumption that the ground truth LMDP admits certain low-rank structure, which is a common assumption for sample-efficient partially observable RL \citep{wang2022embed,chen2022partially,liu2022optimistic}. %
\begin{assumption}[Rank]\label{def:low-rank}
The rank of an LMDP $M_\theta$ is defined as $d_{\theta}\defeq \max_{m\in[L]} \rank(\TT_{\theta,m})$. We assume that the ground truth model $\ths$ has rank $d<\infty$. 
\end{assumption}

\paragraph{Learning goal.} The learner's goal is to output an $\eps$-optimal policy $\hpi$, i.e. a policy with sub-optimality $V_\star-V_{\ths}(\hpi)\leq \eps$, where $V_\star=\max_{\pi\in\Pi} V_{\ths}(\pi)$ is the optimal value of the ground truth LMDP.

\section{Intractability of separated LMDP with horizon below threshold}\label{sec:stat-lower}

Given the exponential hardness of learning general LMDPs, \citet{kwon2021rl} explore several structural conditions under which a near-optimal policy can be learned sample-efficiently. The core assumptions there include a strong separation condition (\cref{def:strong-sep}) together with the bound %
\begin{align}\label{eqn:kwon}
\textstyle    H\geq \delsep^{-4}\log^2(S/\delsep)\log(LSA\eps^{-1}\delsep^{-1}).
\end{align}
A natural question is whether such an assumption on the horizon is necessary. The main result of this section demonstrates the necessity of a moderately long horizon, i.e. in order to learn a \sepstr~LMDP in polynomial samples, it is necessary to have a horizon length that (asymptotically) exceeds $\frac{\log(L/\eps)}{\delta^2}$.

\newcommand{\Hthre}{H_{\rm thre}}
\newcommand{\degC}{\mathsf{d}}
\begin{theorem}[Corollary of \cref{thm:log-L,thm:log-eps}]\label{thm:lower-demo}
Suppose that there exists an integer $\degC\geq 1$ and an algorithm $\fA$ with sample complexity $\max\{S,A,H,L,\eps^{-1},\delsep^{-1}\}^{\degC}$ 
that learns an $\eps$-optimal policy with probability at least $3/4$ in any \sepstr~LMDP with $H\geq \Hthre(L,\eps,\delta)$, for some function $\Hthre(L, \eps, \delta)$. Then there exists constants $c_\degC,\eps_\degC,L_\degC$ (depending on $\degC$) and an absolute constant $\delta_0$ such that %
\begin{align*}
    \Hthre(L,\eps,\delta) \geq \frac{c_\degC\log(L/\eps)}{\delta^2}, \qquad \forall \delta\leq\delta_0,\eps\leq \eps_\degC, L\geq \max(L_\degC,\delta^{-1}).
\end{align*}
\end{theorem}

The proof of \cref{thm:lower-demo} is presented in \cref{appdx:stat-lower}, where we also provide a more precise characterization of the sample complexity lower bounds in terms of $H$ (\cref{thm:log-L,thm:log-eps}). The lower bound of the threshold $\Hthre$ is nearly optimal, in the sense that it almost matches the learnable range (as per \cref{cor:strong-sep-upper} below).

The following theorem provides a simpler lower bound for horizon length $H=\tilde\Theta\paren{\delta^{-1}\log L}$. For such a short horizon, we show that we can recover the exponential lower bound developed in \citet{kwon2021rl} for learning non-separated LMDPs.
\begin{theorem}\label{thm:A-exp}
Suppose that $\delta\in(0,\frac{1}{4e^2}]$, $H\geq 3$, $A\geq 2$, $L\geq 2^{C\log^2(1/\delta)}$ are given such that
\begin{align}\label{eqn:A-exp-constraints}
    CH\log H \log(1/\delta) \leq \frac{\log L}{\delta}.
\end{align}
Then there exists a class of $\delta$-strongly separated LMDPs, each LMDP has $L$ MDP instances, $S=(\log L)^{\cO(\log H)}$ states, $A$ actions, and horizon $H$, such that any algorithm requires $\Om{A^{H-2}}$ samples to learn an $\frac{1}{4H}$-optimal policy with probability at least $\frac{3}{4}$. %
\end{theorem}

\paragraph{Proof idea for \cref{thm:lower-demo}.} \cref{thm:lower-demo} is proved by transforming the known hard instances of general LMDPs (\cref{appdx:comb-lock}) to hard instances of $\delsep$-strong separated LMDPs. In particular, given a LMDP $M$, we transform it to a \sepstr~LMDP $M'$, so that each MDP instance $M_m$ of $M$ is transformed to a mixture of MDPs $\{ M_{m,j} \}$, where each $M_{m,j}=M_i\otimes \mu_{m,j}$ is an MDP obtained by augmenting $M_i$ with a distribution $\mu_{m,j}$ of the auxiliary observation (this operation $\otimes$ is formally defined in \cref{def:MDP-tensor}). The \sepstr~property of $M'$ is ensured as long as $\DTV{\mu_{m,j},\mu_{m',j'}}\geq\delta$ for different pairs of $(m,j)\neq(m',j')$, and intuitively, $M'$ is still a hard instance if the auxiliary observation does not reveal much information of the latent index. %

Such a transformation is possible as long as $H=\frac{o(\log L)}{\delsep^2}$. Here, we briefly illustrate how the transformation works for LMDP $M$ consisted of only 2 MDP instances $M_1, M_2$. Using \cref{prop:2-family-demo}, we define the augmented MDPs $M_{1,j}=M_1\otimes \mu_{j}$ for $j\in\supp(\nu_1)$ and $M_{2,j}=M_2\otimes \mu_{j}$ for $j\in\supp(\nu_2)$, and assigning the mixing weights based on $\nu_1, \nu_2$. Then, result (1) ensures the transformed LMDP is \sepstr, and result (2) ensures the auxiliary observation does not reveal much information of the latent index. The details of our transformation for general LMDPs is presented in \cref{appdx:lower-tool}.

\begin{proposition}[Simplified version of \cref{prop:2-family}]\label{prop:2-family-demo}
Suppose that parameter $\delta, c>0$ and integer $n\geq 2$ satisfy $Cn\log^2 n\leq \min\set{c^{-1},\delta^{-1}}$. Then for $L\geq n^2$, $H\leq \frac{c\log L}{\delta^2}$, there exists $L'\leq L$ distributions $\mu_1,\cdots,\mu_{L'}$ over a set $\cO$ satisfying $|\cO| \leq O(\log L)$, such that:

(1) $\DTV{\mu_i,\mu_j}\geq \delta$ for $i\neq j$.

(2) There exists $\nu_1,\nu_2\in\Delta([L'])$ such that $\supp(\nu_1)$ and $\supp(\nu_2)$ are disjoint, and
\begin{align*}
    \DTV{ \EE_{i\sim \nu_1} \mu_i^{\otimes H}, \EE_{j\sim \nu_2} \mu_j^{\otimes H} } \leq L^{-n},
\end{align*}
where for any distribution $\mu$, $\mu^{\otimes H}$ is the distribution of $(o_1,\cdots,o_H)$ where $o_h \sim \mu$ independently.
\end{proposition}

\paragraph{Tighter threshold for decodable LMDPs} For \sepstr~LMDP, \cref{thm:lower-demo} gives a lower bound of $\Hthre$ that scales as $\frac{\log(L/\eps)}{\delsep^2}$ and nearly matches the upper bounds (\cref{cor:strong-sep-upper}). The following result shows that, for $N$-step decodable LMDPs, we can identify the even tighter threshold of $H$: when $H\leq 2N-\omega(1)$, there is no sample-efficient algorithm; by contrast, when $H\geq 2N$, OMLE is sample-efficient (\cref{cor:decode-upper}). %

\begin{theorem}\label{thm:decode-lower}
Suppose that integers $N\geq n\geq 2$, $A\geq 2$ are given. Then for $H=2N-n$, there exists a class of $N$-step decodable LMDPs with $L=n$, $S=3N-1$ states, $A$ actions, and horizon $H$, such that any algorithm requires $\Om{A^{n-1}}$ samples to learn an $\frac{1}{4n}$-optimal policy with probability at least $\frac{3}{4}$.
\end{theorem}

\section{Learning separated LMDPs with horizon above threshold}\label{sec:stat-upper}

In this section, we show that \sepstr~LMDP, or more generally, any LMDP under suitable policy separation assumptions, can be learned sample-efficiently, as long as the horizon $H$ exceeds a threshold that depends on the separation condition and the logarithm of other problem parameters.

A crucial observation is that if that an LMDP $M_\theta$ is $\om$-separated under policy $\pi$, then the agent can ``decode'' the latent index from the trajectory $\otau_h$, with error probability decaying exponentially in $\om(h)$. 

\begin{proposition}
\label{prop:latent-MLE}
Given an LMDP $M_\theta$ and parameter $W\geq 1$, for any trajectory $\otau_W=(s_1,a_1,\cdots,s_{W})$, we consider the latent index with maximum likelihood under $\otau_W$:
\begin{align}\label{eqn:MLE-traj}
    m_\theta(\otau_W)\defeq \argmax_{m\in\supp(\rho_\theta)}~ \log \rho_\theta(m)+ \log \nu_{\theta,m}(s_1)+\sum_{h=1}^{W-1} \log \TT_{\theta,m}(s_{h+1}|s_{h},a_{h}).
\end{align}
Then as long as $M_\theta$ is $\om$-separated under $\pi$, the \emph{decoding error} can be bounded as
\begin{align}\label{eqn:def-perr}
    \perr_{\theta,W}(\pi)\defeq \tPP_\theta^{\pi}(m_\theta(\otau_W)\neq m^\star)
    \leq L\exp\paren{-\om(W)},
\end{align}
where we recall that $\tPP_\theta^{\pi}$ is the joint probability distribution of the latent index $\ms$ and trajectory $\tau_H$ in the LMDP $M_\theta$ under policy $\pi$.
\end{proposition}

The OMLE algorithm was originally proposed by \citet{liu2022partially} for learning revealing POMDPs, and it was later adapted for a broad class of model-based RL problems \citep{zhan2022pac,chen2022unified,chen2022partially,liu2023optimistic}. Based on the observation above, we propose a variant of the OMLE algorithm for learning separated LMDPs.

\begin{algorithm}[t]
\DontPrintSemicolon
\SetNoFillComment
\caption{\textsc{Optimistic Maximum Likelihood Estimation (OMLE)}}\label{alg:OMLE}
    \textbf{Input:} Model class $\Theta$, policy class $\Pi$, exploration strategy $\modf{\cdot}:\Pi\to\Piall$, parameter $\beta>0, \epssep\in(0,1]$, $W\geq 1$.

\textbf{Initialize:} $\Theta^1=\Theta$, $\cD=\{\}$.

\For{$k=1,\ldots,K$}{
    Set \tcp*[f]{\footnotesize See \cref{eqn:def-perr} for definition of $e_{\theta,W}$.}
    \begin{align*}
    (\theta^k,\pi^k) = \argmax_{(\theta,\pi)} V_\theta(\pi), \qquad\st \theta\in\Theta^k, \perr_{\theta,W}(\pi)\leq \epssep.
    \end{align*}
    
    Execute $\piexp^k=\modf{\pi^k}$ to collect a trajectory $\tau^k_H$, and add  $(\piexp^k,\tau^k_H)$ into $\cD$.
    
    Update confidence set
    \begin{equation*}
    \textstyle
    \Theta^{k+1} = \bigg\{\hat\theta \in \Theta: \sum_{(\pi,\tau)\in\cD} \log \PP_{{\hat\theta}}^{\pi} (\tau)
    \ge \max_{ \theta \in\Theta} \sum_{(\pi,\tau)\in\cD} \log \PP^{\pi}_{{\theta}}(\tau) -\beta \bigg\}. 
    \end{equation*} 
}

\textbf{Output:} $\hat{\pi}\defeq\unif(\set{\pi^1,\cdots,\pi^K})$.
\end{algorithm}

\paragraph{Algorithm.} On a given class $\Theta$ of LMDPs, the OMLE algorithm (\cref{alg:OMLE}) iteratively performs the following steps while building up a dataset $\cD$ consisting of trajectories drawn from the unknown LMDP: %
\begin{enumerate}
\item (Optimism) Construct a confidence set $\Theta^k \subseteq \Theta$ based on the log-likelihood of all trajectories within dataset $\cD$. The optimistic (model, policy) pair $(\theta^k, \pi^k)$ is then chosen greedily while ensuring that the decoding error %
$\perr_{\theta^k,W}(\pi^k)$ is small. %
\item (Data collection) For an appropriate choice of \emph{exploration strategy} $\modf{\cdot}$ (described in \cref{def:policy-mod}), execute the explorative policy $\piexp^k=\modf{\pi^k}$, %
  and then collect the trajectory into $\cD$. %
\end{enumerate}

\paragraph{Guarantees.}
Under the following assumption on all-policy separation with a specific growth function $\om$, the OMLE algorithm can learn a near-optimal policy sample efficiently. In particular, when $\Theta$ is the class of all \sepstr~LMDPs, then \cref{def:all-policy-sep} is fulfilled automatically with $\Pi=\Piall$ and $\om(h)=\frac{\delta^2}{2}(h-1)$ (\cref{prop:strong-sep-to-weak}).

\begin{assumption}[Separation under all policies]\label{def:all-policy-sep}
For any $\theta\in\Theta$ and any $\pi\in\Pi$, $\theta$ is $\om$-separated under $\pi$.
\end{assumption}

\begin{theorem}\label{thm:all-policy-sep-demo}
Suppose that \cref{def:low-rank} and \cref{def:all-policy-sep} hold. We fix any $\piexp\in\Pi$, set $\modf{\cdot}$ as in \cref{def:policy-mod}, and choose the parameters of \cref{alg:OMLE} so that
\begin{align*} 
    &\beta\geq \betaN, \qquad
    K=C_0\frac{Ld^2AH^2\iota\beta}{\eps^2}, \qquad
    \epssep=\frac{\eps^2}{C_0Ld^2H^2\iota},
\end{align*}
where $\iota=\log(LdH/\eps)$ is a log factor, $C_0$ is a large absolute constant. Then, as long as $W$ is suitably chosen so that
\begin{align}\label{eqn:W-H-om}
    W\geq \om^{-1}(\log (L/\epssep)),\qquad
    H-W\geq \om^{-1}(\log (2L)),
\end{align}
\cref{alg:OMLE} outputs an $\eps$-optimal policy $\hpi$ with probability at least $1-p$ after observing $K$ trajectories.
\end{theorem}
Note that the parameter $W$ can always be found satisfying the conditions of \cref{thm:all-policy-sep-demo} as long as $H \geq \om^{-1}(\log(2L)) + \om^{-1}(\log(L/\epssep))$.  
In particular, OMLE is sample-efficient for learning \sepstr~LMDPs with a moderate requirement on the horizon $H$ (which nearly matches the lower bound of \cref{thm:lower-demo}).

\begin{corollary}\label{cor:strong-sep-upper}
Suppose that $\abs{\cS}=S$ and $\Theta$ is the class of all $\delsep$-strongly separated LMDPs. Then as long as
\begin{align}\label{eqn:strong-upper-H}
    H\geq \frac{10\log(LS\eps^{-1}\delsep^{-1})+C}{\delta^2}
\end{align}
for some absolute constant $C$, we can suitably instantiate \cref{alg:OMLE} so that it outputs an $\eps$-optimal policy $\hpi$ with high probability using $K=\tO(\frac{L^2S^4A^2H^4}{\eps^2})$ episodes.
\end{corollary}

Compared to the results of \citet{kwon2021rl}, \cref{cor:strong-sep-upper} requires neither a good initialization that is close to the ground truth model, nor does it require additional assumptions, e.g. test-sufficiency, which is also needed in \citet{zhan2022pac,chen2022partially}. Furthermore, \citet{kwon2021rl} also requires \cref{eqn:kwon}, while the range of tractable horizon \cref{eqn:strong-upper-H} here is wider, and it nearly matches the threshold in \cref{thm:lower-demo}. A more detailed discussion is deferred to \cref{appdx:discuss}.

Furthermore, OMLE is also sample-efficient for learning $N$-step decodable LMDPs, as long as $H\geq 2N$.
\begin{corollary}[Learning decodable LMDPs] \label{cor:decode-upper}
Suppose that $\Theta$ is a class of $N$-step decodable LMDPs with horizon length $H\geq 2N$. Then we can suitably instantiate \cref{alg:OMLE} so that it outputs an $\eps$-optimal policy $\hpi$ with high probability using $K=\tO(\frac{Ld^2AH^2\logNt}{\eps^2})$ episodes.
\end{corollary}

In \citet{efroni2022provable}, a sample complexity that scales with $A^N$ is established for learning general $N$-step decodable POMDPs. %
By contrast, \cref{cor:decode-upper} demonstrates that for $N$-step decodable LMDPs, a horizon length of $H\geq 2N$ suffices to ensure polynomial learnability. As \cref{thm:decode-lower} indicates, requiring $H\geq 2N-\cO(1)$ is also necessary for polynomial sample complexity, and hence the threshold $H\geq 2N$ is nearly \emph{sharp} for $N$-step decodable LMDPs. This result also demonstrates that the condition \cref{eqn:W-H-om} (and our two-phase analysis; see \cref{appdx:upper-overview}) is generally necessary for \cref{thm:all-policy-sep-demo}. %

\subsection{Sample-efficient learning with \twopol~separation}\label{sec:single-policy}

In general, requiring separation under all policies is a relatively restrictive assumption, because it is possible that the LMDP is well-behaved under only a small subset of policies that contains the optimal policy. In this section, we discuss the sample-efficiency of OMLE under the following assumption of separation under an optimal policy.

\begin{assumption}[Separation under an optimal policy]\label{def:single-policy-sep}
There exists an optimal policy $\pis$ of the LMDP $M_{\ths}$, such that $M_{\ths}$ is $\om$-separated under $\pis$.
\end{assumption}

In order to obtain sample-efficiency guarantee, we also need the following technical assumption on a prior-known \emph{separating policy} $\piexp$. Basically, we assume that in each LMDP, the MDP instances are sufficient ``diverse'' under $\piexp$, so that any mixture of them is qualitatively different from any MDP model. 
\newcommand{\tref}{\mathrm{ref}}
\begin{assumption}[Prior knowledge of a suitable policy for exploration]\label{assmp:pi-exp}
There exists a known policy $\piexp$ and parameters $(\Wexp,\alpha)$ such that for any model $\theta\in\Theta$, the following holds: %

(a) $M_\theta$ is $\om$-separated under $\piexp$.

(b) For any MDP model $\TT_{\tref}$ and state $s\in\cS$, it holds that for any $\lambda\in\Delta(\supp(\rho_\theta))$, %
\begin{align}\label{assmp:reg}
    \DTV{ \EE_{m\sim \lambda}\brac{ \MM_{m,\Wexp}^\theta(\piexp,s)}, \MM_{\tref,\Wexp}(\piexp,s) } \geq \alpha(1-\max_m \lambda_m),
\end{align}
where
\begin{align*}
    \MM_{\tref,h}(\piexp,s)=\brac{\TT_{\tref}^{\piexp}((a_1,s_2,\cdots,s_h)=\cdot|s_1=s) } \in \Delta((\cS\times\cA)^{h-1})
\end{align*}
is the distribution of trajectory induced by running $\piexp$ on the MDP with transition $\TT_{\tref}$.
\end{assumption}

\begin{theorem}\label{thm:single-policy-sep-demo}
Suppose that \cref{def:low-rank}, \cref{def:single-policy-sep}, and \cref{assmp:pi-exp} hold. We set $\modf{\cdot}$ based on $\piexp$ as in \cref{def:policy-mod}, and choose the parameters of \cref{alg:OMLE} so that
\begin{align*} 
    &\beta\geq \betaN, \qquad
    K=C_0\frac{L^3d^5AH^6\iota^3\beta}{\alpha^2\eps^4}, \qquad
    \epssep=\frac{\alpha\eps^2}{C_0Ld^2H^2\iota},
\end{align*}
where $\iota=\log(LdH\alpha^{-1}\eps^{-1})$ is a log factor, $C_0$ is a large absolute constant. Then, as long as $W$ is suitably chosen so that
\begin{align*}
    W\geq \om^{-1}(\log (L/\epssep)),\qquad
    H-W\geq \Wexp,
\end{align*}
\cref{alg:OMLE} outputs an $\eps$-optimal policy $\hpi$ with probability at least $1-p$.
\end{theorem}

In \cref{appdx:pi-exp-example}, we also provide a sufficient condition of \cref{assmp:pi-exp}, which is more intuitive. %

\section{Computation complexity of separated LMDPs}\label{sec:comp}

In this section, we investigate the computational complexity of planning in a \emph{given} LMDP, i.e. a description of the ground truth model $\ths$ is provided to the learner.\footnote{In this section, we omit the subscript of $\ths$ for notational simplicity, because the LMDP $M=M_{\ths}$ is given and fixed.} For planning, a longer horizon does not reduce the time complexity (in contrast to learning, where a longer horizon does help). %

In general, we cannot expect a polynomial time planning algorithm for \sepstr~LMDP, because even the problem of computing an approximate optimal value in any given \sepstr~LMDP is NP-hard.
\begin{proposition}\label{prop:NP-hard}
If there is an algorithm that computes the $\eps$-approximate optimal value of any given $\delsep$-strongly separated LMDP in $\poly(L,S,A,H,\eps^{-1},\delsep^{-1})$ time, then P=NP.
\end{proposition}

On the other hand, utilizing the \cref{prop:latent-MLE}, we propose a simple planning algorithm (\cref{alg:plan}) for any LMDP that is separated under its optimal policy. The algorithm design is inspired by the Short Memory Planning algorithm proposed by \citet{golowich2022planning}.

\newcommand{\Vind}[2]{\hV_{#1,#2}}
\newcommand{\Qind}[2]{\hQ_{#1,#2}}
\newcommand{\pind}[2]{\pi_{#1,#2}}

\begin{algorithm}[t]
\caption{Short Memory Planning with Context Inference}\label{alg:plan}
\KwData{ $W\geq 1$, LMDP model $M=M_{\ths}$ }
Set $\Vind{m}{H+1}(\emptyset)=0$ for all $m\in[L]$\;

\For{$h=H,H-1,\cdots,W$}{
    For each pair $(s_h,a_h,m)\in\cS\times\cA\times[L]$, update
    \begin{align*}
        \Qind{m}{h}(s_h,a_h)=\EE_{s_{h+1}\sim \TT_m(\cdot|s_h,a_h)}\brac{ \Vind{m}{h+1}(s_{h+1}) }+R_h(s_h,a_h).
    \end{align*}
    Set $\Vind{m}{h}(s_h)=\max_{a_h} \Qind{m}{h}(s_h,a_h)$ and store $\pind{m}{h}(s_h)=\argmax_{a_h} \Qind{m}{h}(s_h,a_h)$.
}

\For{each $\otau_W=(s_1,a_1,\cdots,s_W)$}{
    Compute $m=m(\otau_W)$ and set
    \begin{align*}
        \hV(\otau_W)=\PP(m|\otau_W)\cdot \Vind{m}{W}(s_W).
    \end{align*}
}
\For{$h=W-1,\cdots,1$}{
    For each $(\otau_h,a_h)\in (\cS\times\cA)^h$, update
    \begin{align*}
        \hQ(\otau_h,a_{h})=\EE_{s_{h+1}|\otau_h,a_h}\brac{ \hV(\otau_h,a_h,s_{h+1}) }+R_h(s_h,a_h), \qquad \forall \otau_h,a_h
    \end{align*}
    Set $\hV(\otau_h)=\max_{a_h} \hQ(\otau_h,a_h)$ and store $\pi_h(\otau_h)=\argmax_{a_h} \hQ(\otau_h,a_h)$.
}

\KwResult{ description of the determinstic policy $\hpi$ given by 
\begin{align*}
    \hpi(\otau_h)=\begin{cases}
        \pi_h(\otau_h), & h< W, \\
        \pi_{h,m(\otau_W)}(s_h), & h\geq W.
    \end{cases}
\end{align*}}
\end{algorithm}

\begin{theorem}\label{thm:plan}
Suppose that in the LMDP $M$, there exists an optimal policy $\pis$ such that $M$ is $\om$-separated under $\pis$. Then \cref{alg:plan} with $W\geq \om^{-1}(\log(L/\eps))$ outputs an $\eps$-optimal policy $\hpi$ in time
\begin{align*}
    (SA)^{W}\times \poly(S,A,H,L).
\end{align*}
\end{theorem}

As a corollary, \cref{alg:plan} can output an $\eps$-optimal policy (along with an $\eps$-approximate optimal value) of any given $\delsep$-strongly separated LMDP in time
\begin{align*}
    (SA)^{2\delsep^{-2}\log(L/\eps)}\times \poly(L,S,A,H).
\end{align*}
In the following, we demonstrate such a time complexity is nearly optimal for planning in \sepstr~LMDP, under the Exponential Time Hypothesis (ETH):%

\begin{conjecture}[ETH, \citet{impagliazzo2001complexity}]\label{ETH}
There is no $2^{o(n)}$-time algorithm which can determine whether a given 3SAT formula on $n$ variables is satisfiable.
\end{conjecture}

\newcommand{\Alogeps}{A^{o(\delta^{-2}\log(1/\eps))}}
\newcommand{\AlogL}{A^{o\paren{\delta^{-2}\frac{\log L}{\log\log L}}}}

In the following theorems, we provide quasi-polynomial time lower bounds for planning in \sepstr~LMDP, assuming ETH. %
In order to provide a more precise characterization of the time complexity lower bound in terms of all the parameters $(L,\eps,\delta,A)$, we state our hardness results in with varying $(L,\eps,\delta,A)$ pair, with mild assumptions of their growth. To this end, we consider $\cF=\set{ (b_t)_{t\geq 1}, b_t\leq b_{t+1}\leq 2b_t }$, the set of all increasing sequences with moderate growth.

\begin{theorem}\label{thm:comp-log-eps}
Suppose that we are given a sequence of parameters $\cC=\{(\eps_t,A_t,\delta_t)\}_{t\geq 1}$, such that the sequences $(\log \eps_t^{-1})_{t \geq 1}$, $(\delta_t^{-1})_{t \geq 1}$, $(\log A_t)_{t \geq 1} \in \cF$, and%
\begin{align}\label{eqn:comp-log-eps-constraints}
    \eps_t\leq \frac{\delta_t^{10}}{(\log A_t)^5}, \qquad \eps_t\leq \frac{1}{t}, \qquad\qquad \forall t\geq 1.
\end{align}
Then, under Exponential Time Hypothesis (\cref{ETH}), no $\Alogeps$-time algorithm can determine the $\eps$-optimal value of any given $\delta$-strongly separated LMDP with $(\eps,\delta,A)\in\cC$ whose parameters $H,L,S$ satisfy $H\leq \frac{\log(1/\eps)}{\delta^2}$ and $\max\set{ L,S} = \poly(\log (1/\eps), \log A, \delta^{-1})$.
\end{theorem}

\begin{theorem}\label{thm:comp-log-L}
Suppose that we are given a sequence of parameters $\cC=\{(L_t,A_t,\delta_t)\}_{t\geq 1}$, such that the sequences $(\log L_t)_{t \geq 1}, (\delta_t^{-1})_{t \geq 1}, (\log A_t)_{t \geq 1} \in \cF$, $(L_t)_{t\geq 1}$ is strictly increasing, and
\begin{align}\label{eqn:comp-log-L-constraints}
    \log\log L_t \ll \frac{\log A_t}{\delta_t^2} \leq \poly\log L_t, \qquad\qquad \forall t\geq 1.
\end{align}
Then, under Exponential Time Hypothesis (\cref{ETH}), no $\AlogL$-time algorithm can determine the $\eps$-optimal value of any given $\delta$-strongly separated LMDP with $(L,A,\delta)\in\cC$ whose parameters $H,L,S$ satisfy $H\leq \frac{\log L}{\delta^2}$, and $\eps = \frac{1}{\poly(\log L)}$, $S=\exp\paren{\cO(\log^2\log L)}$. %
\end{theorem}

In particular, the results above show that under ETH, a time complexity that scales with $A^{\delta^{-2}\log(L/\eps)}$ is hard to avoid for planning in \sepstr~LMDP, in the sense that our iteration complexity lower bounds apply to any planning algorithm that works for general parameters $(L,A,\delta,\eps)$. Therefore, the threshold $H_\star \asymp \frac{\log(L/\eps)}{\delta^2}$ indeed also captures the computational complexity of planning.

\section*{Acknowledgements}
CD is supported by NSF Awards CCF-1901292, DMS-2022448, and DMS2134108, a
Simons Investigator Award, and the Simons Collaboration on the Theory of Algorithmic Fairness. NG is supported by a Fannie \& John Hertz Foundation Fellowship and an NSF Graduate Fellowship. FC and AR acknowledge support from 
ARO through award W911NF-21-1-0328, DOE through award DE-SC0022199, and the Simons Foundation and the NSF through award DMS-2031883.

\bibliographystyle{abbrvnat}
\bibliography{ref.bib}

\newpage
\appendix

\section{Technical tools}\label{appdx:tools}

\subsection{Covering number}

\begin{definition}[Covering]
\label{def:optimistic-cover}
A $\rho$-cover of the LMDP model class $\Theta$ is a tuple $(\hPP,\Theta_0)$, where $\Theta_0\subset \Theta$ is a finite set, and for each $\theta_0\in\Theta_0$, $\pi\in\Piall$, $\hPP_{\theta_0}^\pi(\cdot)\in\R_{\ge 0}^{\cT}$ specifies an \emph{optimistic likelihood function} such that the following holds: 

(1) For $\theta\in\Theta$, there exists a $\theta_0\in\Theta_0$ satisfying: for all $\tau\in\cT^H$ and $\pi\in\Piall$, it holds that $\hPP_{\theta_0}^{\pi}(\tau)\geq \PP_{\theta}^{\pi}(\tau)$.

(2) For $\theta\in\Theta_0$, $\pi\in\Piall$, it holds $\nrm{\PP_{\theta}^{\pi}(\tau_H=\cdot)-\hPP_{\theta}^{\pi}(\tau_H=\cdot)}_1\leq\rho^2$.

The optimistic covering number $\Nt(\rho)$ is defined as the minimal cardinality of $\Theta_0$ such that there exists $\tPP$ such that $(\tPP,\Theta_0)$ is an optimistic $\rho$-cover of $\Theta$. 
\end{definition}
The above definition of covering is taken from \citet{chen2022unified}. It is known that the covering number defined above can be upper bounded by the bracket number adopted in \citet{zhan2022pac, liu2022optimistic}. In particular, when $\Theta$ is a class of LMDPs with $\abs{\cS}=S, \abs{\cA} = A$, horizon $H$, and with $L$ latent contexts, we have
\begin{align*}
    \logNt(\rho) \leq CLS^2A\log(CLSAH/\rho),
\end{align*}
where $C$ is an absolute constant (see e.g. \citet{chen2022partially,liu2022partially}).

\subsection{Information theory}

In this section, we summarize several basic inequalities related to TV distance, Hellinger distance and Bhattacharyya divergence.

\begin{lemma}\label{lem:TV-Hellinger}
For any two distribution $\PP,\QQ$ over $\cX$, it holds that $\DTV{\PP,\QQ}\leq \sqrt{2}\dH(\PP,\QQ)$, and
\begin{align}\label{eqn:TV-DB}
    \DTV{\PP,\QQ}\geq \DH{\PP,\QQ}=1-\exp\paren{-\DB{\PP,\QQ}}.
\end{align}
Conversely, we also have (Pinsker inequality)
\begin{align}\label{eqn:DB-TV}
    \DB{\PP,\QQ}\geq -\frac12\log(1-\dTV^2(\PP,\QQ)) \geq \frac{1}{2}\dTV^2(\PP,\QQ).
\end{align}
\end{lemma}

\begin{lemma}[{\citet[Lemma A.11]{foster2021statistical}}]
\label{lemma:multiplicative-hellinger}
For distributions $\PP,\QQ$ defined on $\cX$ and function $h:\cX\to[0,R]$, we have
\begin{align*}
    \E_{\PP}\brac{h(X)} \leq 3\E_{\QQ}\brac{h(X)} +2R\dH^2(\PP, \QQ).
\end{align*}
\end{lemma}

\begin{lemma}\label{lemma:TV-cond} 
    For any pair of random variable $(X,Y)$, it holds that
    \begin{align*}
        \EE_{X\sim\PP_X}\brac{\DTV{\PP_{Y|X}, \QQ_{Y|X}}}\leq 2\DTV{\PP_{X,Y}, \QQ_{X,Y}}.
    \end{align*}
    Conversely, it holds that
    \begin{align*}
    \DTV{\PP_{X,Y}, \QQ_{X,Y}}\leq \DTV{\PP_{X}, \QQ_{X}}+\EE_{X\sim\PP_X}\brac{\DTV{\PP_{Y|X}, \QQ_{Y|X}}}.
    \end{align*}
\end{lemma}

\begin{lemma}[{\citet[Lemma A.4]{chen2022unified}}]\label{lemma:Hellinger-cond}
    For any pair of random variable $(X,Y)$, it holds that
    \begin{align*}
        \EE_{X\sim\PP_X}\brac{\DH{\PP_{Y|X}, \QQ_{Y|X}}}\leq 2\DH{\PP_{X,Y}, \QQ_{X,Y}}.
    \end{align*}
    Conversely, it holds that
    \begin{align*}
    \DH{\PP_{X,Y}, \QQ_{X,Y}}\leq 3\DH{\PP_{X}, \QQ_{X}}+2\EE_{X\sim\PP_X}\brac{\DH{\PP_{Y|X}, \QQ_{Y|X}}}.
    \end{align*}
\end{lemma}

\subsection{Technical inequalities}

\begin{lemma}\label{lem:DB-mixture-dist}
For distributions $\PP_1,\cdots,\PP_L\in\Delta(\cO)$ and $\mu,\nu\in\Delta([L])$ so that $\supp(\mu)\cap \supp(\nu)=\emptyset$, we have
\begin{align*}
    \DB{ \EE_{i\sim \mu}\brac{\PP_i}, \EE_{j\sim \nu}\brac{\PP_j} }\geq \min_{i\neq j} \DB{ \PP_i, \PP_j }-\log (L/2).
\end{align*}
As a corollary, if $\DB{ \PP_i, \PP_j }\geq \log L$ for all $i\neq j$, then for any $\mu,\nu\in\Delta([L])$, we have
\begin{align*}
    \DTV{ \EE_{i\sim \mu}\brac{\PP_i}, \EE_{j\sim \nu}\brac{\PP_j} } \geq \frac12 \DTV{\mu,\nu}.
\end{align*}
\end{lemma}

\begin{proof}
By definition,
\begin{align*}
    \exp\paren{ - \DB{ \EE_{i\sim \mu}\brac{\PP_i}, \EE_{j\sim \nu}\brac{\PP_j} } }
    =&~\sum_{x} \sqrt{\EE_{i\sim \mu}\brac{\PP_i(x)}\EE_{j\sim \nu}\brac{\PP_j(x)} } \\
    \leq&~ \sum_{x} \sum_{i,j} \sqrt{ \mu(i)\nu(j) \PP_i(x)\PP_j(x)  } \\
    =&~ \sum_{i,j} \sqrt{\mu(i)\nu(j)} \exp\paren{ -\DB{\PP_i,\PP_j} } \\
    \leq&~ \paren{ \sum_{i}\sqrt{\mu(i)} } \paren{\sum_j \sqrt{\nu(j)}} \max_{i\neq j} \exp\paren{ -\DB{\PP_i,\PP_j} }  \\
    \leq&~ \frac{L}{2}\exp\paren{ -\min_{i\neq j}\DB{\PP_i,\PP_j} },
\end{align*}
where the last inequality follows from the fact that $\sum_{i}\sqrt{\mu(i)} \leq \sqrt{\#\supp(\mu)}$ and $\sum_{j}\sqrt{\nu(j)} \leq \sqrt{\#\supp(\nu)}$. Taking $-\log$ on both sides completes the proof.
\end{proof}

\begin{lemma}\label{lem:DB-inv}
Suppose that for distributions $\PP_1,\cdots,\PP_L\in\Delta(\cO)$, we have $\DB{ \PP_i, \PP_j }\geq \log (2L)$ for all $i\neq j$. Then for the matrix $\MM=[\PP_1,\cdots,\PP_L]\in\R^{\cO\times L}$, there exists $\MM^+\in\R^{L\times\cO}$ such that $\lone{\MM^+}\leq 2$ and $\MM^+\MM=\id_L$.
\end{lemma}

\begin{proof}
We construct $\MM^+$ explicitly. Consider the matrix $Z\in\R^{L\times\cO}$ given by
\begin{align*}
    [Z]_{m,o}=\frac{\PP_m(o)}{\sum_{i\in[L]} \PP_i(o)}.
\end{align*}
Then clearly $\lone{Z}\leq 1$, and the matrix $Y=Z\MM$ is given by
\begin{align*}
    [Y]_{l,m}=\sum_{o\in\cO} \frac{\PP_l(o)\PP_m(o)}{\sum_{i\in[L]} \PP_i(o)}.
\end{align*}
For $l\neq m$, we know
\begin{align*}
    0\leq [Y]_{l,m}\leq \sum_{o\in\cO} \frac{\PP_l(o)\PP_m(o)}{2\sqrt{\PP_l(o)\PP_m(o)}} 
    = \frac12 \sum_{o\in\cO} \sqrt{\PP_l(o)\PP_m(o)}
    = \frac12\exp\paren{ -\DB{\PP_l,\PP_m} }
    \leq \frac{1}{4L}.
\end{align*}
Furthermore,
\begin{align*}
    0\leq 1-[Y]_{m,m}=
    \sum_{o\in\cO} \sum_{l\neq m}\frac{\PP_l(o)\PP_m(o)}{\sum_{i\in[L]} \PP_i(o)} = \sum_{l\neq m} [Y]_{l,m} \leq \frac{1}{4}.
\end{align*}
Combining these two inequalities, we know $\lone{\id_L-Y}\leq \frac{1}{2}$, and hence $\lone{Y^{-1}}\leq 2$. Therefore, we can take $\MM^+=Y^{-1}Z$ so that $\lone{\MM^+}\leq \lone{Y^{-1}}\lone{Z}\leq 2$ and $\MM^+\MM=\id_L$.
\end{proof}

\subsection{Eluder arguments}
\newcommand{\Ccov}{C_{\rm cov}}

In this section, we present the eluder arguments that are necessary for our analysis in \cref{appdx:stat-upper}. 
The following proposition is from \citet[Corollary E.2]{chen2022partially} (with suitable rescaling).
\begin{proposition}[{\citet{chen2022partially}}]\label{prop:semi-linear-eluder}
Suppose we have a sequence of functions $\{ f_k:\R^n\to \R \}_{k \in [K]}$:
\begin{align*}
    f_k(x):=\max_{r \in \cR}\sum_{j=1}^J \abs{\iprod{x}{y_{k,j,r}}},
\end{align*}
which is given by the family of vectors $\set{y_{k,j,r}}_{(k,j,r)\in[K]\times[J]\times\cR}\subset\R^n$.
Further assume that there exists $L_1>0$ such that $f_{k}(x)\leq L_1\nrm{x}_1$. 

Consider further a sequence of vectors $(x_{i})_{i\in\cI}\subset \R^n$ such that the subspace spanned by  $(x_{i})_{i\in\cI}$ has dimension at most $d$.
Then for any sequence of $p_1,\cdots,p_K\in\Delta(\cI)$ and constant $M>0$, it holds that 
\begin{align*}
\sum_{k=1}^K M\wedge \EE_{i\sim p_k} \brac{ f_k(x_{i}) }
\leq
\sqrt{4d\log\paren{1+\frac{KdL_1\max_i \lone{x_i}}{M}}\brac{ KM+\sum_{k=1}^K\sum_{t<k}\EE_{i\sim p_t}\brac{ f_k(x_{i})^2 } }}.
\end{align*}
\end{proposition}

The following proposition is an generalized version of the results in \citet[Appendix D]{xie2022role}. We provide a proof for the sake of completeness.
\begin{proposition}[{\cite{xie2022role}}]\label{prop:coverage-eluder}
Suppose that $p_1,\cdots,p_K$ is a sequence of distributions over $\cX$, and there exists $\mu\in\Delta(\cX)$ such that $p_k(x)/\mu(x)\leq \Ccov$ for all $x\in\cX$, $k\in[K]$. Then for any sequence $f_1,\cdots,f_K$ of functions $\cX\to[0,1]$ and constant $M\geq 1$, it holds that
\begin{align*}
    \sum_{k=1}^K \EE_{x\sim p_k} f_k(x) \leq \sqrt{2\Ccov\log\paren{1+\frac{\Ccov K}{M}}\brac{2KM+\sum_{k=1}^K \sum_{t<k} \EE_{x\sim p_t} f_k(x)^2} }
\end{align*}
\end{proposition}

\begin{proof}
\newcommand{\tp}{\Tilde{p}}
For any $x\in\cX$, define
\begin{align*}
    \tp_k(x)=M\mu(x)+\sum_{t\leq k} p_t(x).
\end{align*}
Then by Cauchy inequality,
\begin{align*}
    \EE_{x\sim p_k} f_k(x) = \sum_{x\in\cX} p_k(x) f_k(x) 
    \leq \sqrt{\sum_{x\in\cX} \frac{p_k(x)^2}{\tp_k(x)} \sum_{x\in\cX} \tp_k(x) f_k(x)^2 }.
\end{align*}
Applying Cauchy inequality again, we obtain
\begin{align*}
    \sum_{k=1}^K \EE_{x\sim p_k} f_k(x) 
    \leq \sqrt{\sum_{k=1}^K\sum_{x\in\cX} \frac{p_k(x)^2}{\tp_k(x)} }
    \cdot\sqrt{ \sum_{k=1}^K \sum_{x\in\cX} \tp_k(x) f_k(x)^2 }
\end{align*}
Notice that 
\begin{align*}
    \sum_{x\in\cX} \tp_k(x) f_k(x)^2 \leq M+1+ \sum_{t<k} \EE_{x\sim p_t} f_k(x)^2,
\end{align*}
and hence it remains to bound
\begin{align*}
    \sum_{k=1}^K\sum_{x\in\cX} \frac{p_k(x)^2}{\tp_k(x)}
    \leq \sum_{x\in\cX} \Ccov\mu(x) \cdot \sum_{k=1}^K \frac{p_k(x)}{\tp_k(x)}.
\end{align*}
Using the fact that $u\leq 2\log(1+u) \forall u\in[0,1]$, we have
\begin{align*}
    \sum_{k=1}^K \frac{p_k(x)}{\tp_k(x)} \leq&~ 2\sum_{k=1}^K \log\paren{1+\frac{p_k(x)}{\tp_k(x)}} \\
    \leq&~ 2\sum_{k=1}^K \log\paren{1+\frac{p_k(x)}{M\mu(x)+\sum_{t<k}p_t(x)}} \\
    =&~ 2 \log\paren{\frac{M\mu(x)+\sum_{t\leq K}p_t(x)}{M\mu(x)}} \\
    \leq&~ 2\log\paren{1+\frac{\Ccov K}{M}}
\end{align*}
Combining the inequalities above completes the proof.
\end{proof}

\begin{proposition}\label{prop:rank-to-cov}
Suppose that $\TT\in\RR^{\cS\times(\cS\times\cA)}$ is a transition matrix such that $\rank(\TT)=d$. Then there exists a distribution $\nu\in\Delta(\cS)$ such that $\TT(s'|s,a)\leq d\cdot \nu(s')\  \forall (s,a,s')\in\cS\times\cA\times\cS$.
\end{proposition}

\begin{proof}
Consider the set
\begin{align*}
    \cP=\set{ \TT(\cdot|s,a): s\in\cS,a\in\cA }\subset \R^{\cS}.
\end{align*}
Then $\rank(\TT)=d$ implies that $\cP$ spans a $d$-dimensional subspace of $\R^{\cS}$. Clearly, $\cP$ is compact, and hence it has a barycentric spanner \citep{awerbuch2008online}, i.e. there exists $\set{\nu_1,\cdots,\nu_d} \subseteq \cP$, such that for any $\mu\in\cP$, there are $\lambda_1,\cdots,\lambda_d\in[-1,1]$ such that
\begin{align*}
    \mu=\lambda_1\nu_1+\cdots+\lambda_d\nu_d.
\end{align*}
Therefore, we can take $\nu=\frac{1}{d}\sum_{i=1}^d \nu_i$.
\end{proof}

\section{Further comparison with related work}\label{appdx:discuss}

In \citet{kwon2021rl}, to learn a \sepstr~LMDP, the proposed algorithms require a horizon $H\geqsim \delsep^{-4}\log^2(S/\delsep)\log(LSA\eps^{-1}\delsep^{-1})$, and also one of the following assumptions:
\begin{itemize}
\item a good initialization, i.e. an initial approximation of the latent dynamics of the ground truth model, with error bounded by $o(\delta^2)$ \citep[Theorem 3.4]{kwon2021rl}.
\item The so-called sufficient-test condition and sufficient-history condition, along with the reachability of states \citep[Theorem 3.5]{kwon2021rl}.
\end{itemize}

\newcommand{\Ems}{\mathbb{K}}
\newcommand{\bss}{\mathbf{s}}
\citet{chen2022partially} further show that, for general LMDPs (not necessarily \sepstr), the sufficient-test condition itself implies that the OMLE algorithm is sample-efficient. More concretely, their result applies to any \emph{$W$-step revealing LMDP}. A LMDP is $W$-step $\alpha$-revealing if the $W$-step emission matrix
\begin{align*}
    \Ems(s)\defeq \brac{ \TT_m(s_{2:W}=\bss|s_1=s,a_{1:W-1}=\ba) }_{(\bss,\ba),m} \in \RR^{(\cA\times\cS)^{W-1}\times [L]}
\end{align*}
admits a left inverse $\Ems^+(s)$ for all $s\in\cS$ such that $\lone{\Ems^+(s)}\leq \alpha^{-1}$. This condition implies the standard $W$-step revealing condition of POMDPs~\citep{liu2022partially,chen2022partially} because the state $s$ is observable in LMDPs\footnote{see, e.g. \citet[Proposition B.10]{chen2022partially} or the proof of \cref{thm:psr} in \cref{appdx:psr}.}.
In particular, the following theorem now follows from \citet[Theorem 9]{chen2022partially}.
\begin{theorem}\label{thm:rev-LMDP}
The class of $W$-step $\alpha$-revealing LMDPs can be learning using $\poly(A^W,\alpha^{-1},L,S,H,\eps^{-1})$ samples.
\end{theorem}
Without additional assumption, it is only known that a \sepstr~LMDP is $W$-step $\alpha$-revealing with $W=\ceil{\frac{2\log(2L)}{\delsep^2}}$ and $\alpha=2$. \footnote{
This result can be obtained by applying \cref{lem:DB-inv} to the distributions of trajectories induced by policy $\unif(\cA^{W-1})$.
}
Therefore, when applied to \sepstr~LMDPs, \cref{thm:rev-LMDP} gives a sample complexity bound that scales with $A^{\delsep^{-2}\log L}$, which is quasi-polynomial in $(A,L)$. %
Further, as \cref{thm:A-exp} indicates, such a quasi-polynomial sample complexity is also unavoidable if the analysis only relies on the revealing structure of \sepstr~LMDP and does not take the horizon length $H$ into account. %

On the other hand, our analysis in \cref{appdx:stat-upper} is indeed built upon the revealing structure of \sepstr~LMDP. However, we also leverage the special structure of separated LMDP, so that we can avoid using the brute-force exploration strategy that essentially samples $a_{H-W+1:H-1}\sim \Unif(\cA^{W-1})$ in the course of the algorithm. Such a uniform-sampling exploration %
approach for learning the system dynamics of the last $W$ steps is generally necessary in learning revealing POMDPs, as the lower bounds of \citet{chen2023lower} indicate. It turns out to be unnecessary for separated LMDP. \cref{appdx:upper-overview} provides a technical overview with more details. %

\section{Proofs for Section~\ref{sec:prelim}}\label{appdx:proof-prelim}

\subsection{Proof of Proposition~\ref{prop:strong-sep-to-weak}}\label{appdx:proof-strong-sep-to-weak}

Fix $m,l\in\supp(\rho)$, $m\neq l$. By definition,
{\small
\begin{align*}
    &~ \DB{ \MM_{m,h+1}(\pi,s), \MM_{l,h+1}(\pi,s) } \\
    =&~ -\log \sum_{a_{1:h},s_{2:h+1}} \sqrt{ \TT_m^\pi(a_1,s_2,\cdots,s_{h+1}|s_1=s)\TT_l^\pi(a_1,s_2,\cdots,s_{h+1}|s_1=s) } \\
    =&~
    -\log \sum_{a_{1:h},s_{2:h}} \sqrt{ \TT_m^\pi(a_1,s_2,\cdots,s_h,a_h|s_1=s)\TT_l^\pi(a_1,s_2,\cdots,s_h,a_h|s_1=s) } \cdot \sum_{s_h} \sqrt{\TT_m(s_{h+1}|s_h,a_h)\TT_l(s_{h+1}|s_h,a_h)} \\
    =&~
    -\log \sum_{a_{1:h},s_{2:h}} \sqrt{ \TT_m^\pi(a_1,s_2,\cdots,s_h,a_h|s_1=s)\TT_l^\pi(a_1,s_2,\cdots,s_h,a_h|s_1=s) } \cdot \exp\paren{ -\DB{ \TT_m(\cdot|s_h,a_h),\TT_l(\cdot|s_h,a_h)} }.
\end{align*}}
Because $M$ is a $\delsep$-strongly separated LMDP, using \cref{eqn:DB-TV}, we know
\begin{align*}
    \DB{ \TT_m(\cdot|s,a),\TT_l(\cdot|s,a)} \geq \frac{1}{2}\DTVt{ \TT_m(\cdot|s,a),\TT_l(\cdot|s,a)}
    \geq \frac{\delsep^2}{2}, \qquad \forall (s,a)\in\cS\times\cA.
\end{align*}
Therefore, we can proceed to bound
\begin{align*}
    &~ \DB{ \MM_{m,h+1}(\pi,s), \MM_{l,h+1}(\pi,s) } \\
    \geq &~ 
    \frac{\delsep^2}{2}-\log \sum_{a_{1:h},s_{2:h}} \sqrt{ \TT_m^\pi(a_1,s_2,\cdots,s_h,a_h|s_1=s)\TT_l^\pi(a_1,s_2,\cdots,s_h,a_h|s_1=s) } \\
    =&~
    \frac{\delsep^2}{2}-\log \sum_{a_{1:h-1},s_{2:h}} \sqrt{ \TT_m^\pi(a_1,s_2,\cdots,s_h|s_1=s)\TT_l^\pi(a_1,s_2,\cdots,s_h|s_1=s) } \cdot \sum_{a_h}\pi(a_h|s,a_1,s_2,\cdots,s_h) \\
    =&~
    \frac{\delsep^2}{2}-\log \sum_{a_{1:h-1},s_{2:h}} \sqrt{ \TT_m^\pi(a_1,s_2,\cdots,s_h|s_1=s)\TT_l^\pi(a_1,s_2,\cdots,s_h|s_1=s) }\\
    =&~
    \frac{\delsep^2}{2}+\DB{ \MM_{m,h}(\pi,s), \MM_{l,h}(\pi,s) }.
\end{align*}
Applying the inequality above recursively, we obtain $\DB{ \MM_{m,h+1}(\pi,s), \MM_{l,h+1}(\pi,s) }\geq \frac{\delsep^2}{2} h$, the desired result.
\qed

\subsection{Proof of Proposition~\ref{prop:decode-to-weak}}\label{appdx:proof-decode-to-weak}

Suppose that $M$ is a $N$-step decodable LMDP. By definition of $\om_N$-separation, we only need to show that for any $m,l\in\supp(\rho)$, $m\neq l$ and policy $\pi\in\Piall$, it holds that
\begin{align*}
    \supp(\MM_{m,h}(\pi,s)) \cap \supp(\MM_{l,h}(\pi,s)) = \emptyset, \qquad \forall h\geq N, s\in\cS,
\end{align*}
or equivalently,
\begin{align*}
    \TT_m^\pi(a_1,s_2,\cdots,s_{h}|s_1=s)\TT_l^\pi(a_1,s_2,\cdots,s_{h}|s_1=s)=0, \qquad \forall h\geq N, \forall \otau_h=(s_1,a_1,\cdots,s_h).
\end{align*}
This is because the $N$-step decoability of $M$ implies that for any $\otau_h=(s_1,a_1,\cdots,s_h)$, there exists at most one $\ms\in\supp(\rho)$ such that
\begin{align*}
    \TT_{\ms}(s_2|s_1,a_1)\cdots\TT_{\ms}(s_h|s_{h-1},a_{h-1})>0.
\end{align*}
The desired result follows immediately.
\qed

\subsection{Proof of Lemma~\ref{lem:markov-to-all}}\label{appdx:proof-markov-to-all}
\newcommand{\BC}[1]{\mathrm{BC}\paren{#1}}

For notational simplicity, we denote
\begin{align*}
    \BC{\PP,\QQ}=\exp(-\DB{\PP,\QQ}).
\end{align*}
Fix $h\geq 1$ and $m,l\in\supp(\rho)$, $m\neq l$. We only need to show that the following policy optimization problem
\begin{align}\label{eqn:max-BC}
    \max_{\pi\in\Piall}\BC{ \MM_{m,h+1}(\pi,s), \MM_{l,h+1}(\pi,s) }
\end{align}
is attained at a deterministic Markov policy.
Recall that
\begin{align*}
    &~\BC{ \MM_{m,h+1}(\pi,s), \MM_{l,h+1}(\pi,s) } \\
    =&~\sum_{a_{1:h},s_{2:h}} \sqrt{ \TT_m^\pi(a_1,s_2,\cdots,s_h,a_h|s_1=s)\TT_l^\pi(a_1,s_2,\cdots,s_h,a_h|s_1=s) } \cdot \BC{ \TT_m(\cdot|s_h,a_h),\TT_l(\cdot|s_h,a_h)}.
\end{align*}
Therefore, \cref{eqn:max-BC} is attained at a policy $\pi$ with
\begin{align*}
    \pi_h(s_h)=\argmax_{a\in\cA} ~\BC{ \TT_m(\cdot|s_h,a)\TT_l(\cdot|s_h,a)}.
\end{align*}
Inductively repeating the argument above for $h'=h,h-1,\cdots,1$ completes the proof.
\qed

\subsection{Proof of Proposition~\ref{prop:latent-MLE}}\label{appdx:proof-latent-MLE}

Notice that $m_\theta(\otau_W)=\argmax_{m\in\supp(\rho)} \tPP_\theta(m|\otau_W)$. Therefore,
\begin{align*}
    \tPP_\theta^{\pi}(m^\star\neq m_\theta(\otau_W))
    =&~ \sum_{\otau_W} \tPP_\theta(m^\star\neq m_\theta(\otau_W)|\otau_W)\cdot \tPP_\theta^{\pi}(\otau_W) \\
    =&~ \sum_{\otau_W} \sum_{m\neq m_\theta(\otau_W)} \tPP_\theta(m|\otau_W)\cdot \tPP_\theta^{\pi}(\otau_W) \\
    =&~ \sum_{m^\star,\otau} \sum_{m\neq m_\theta(\otau_W)} \tPP_\theta(m|\otau_W)\cdot \tPP_\theta^{\pi}(m^\star,\otau_W) \\
    \leq&~ 
    \sum_{m^\star,\otau} \sum_{m\neq m^\star} \tPP_\theta(m|\otau_W)\cdot \tPP_\theta^{\pi}(m^\star,\otau_W) \\
    =&~
    \sum_{m\neq l} \sum_{\otau_W} \frac{\tPP_\theta^{\pi}(m,\otau_W)\tPP_\theta^{\pi}(l,\otau_W)}{\tPP_\theta^{\pi}(\otau_W)} \\
    =&~
    \sum_{m\neq l} \sum_{\otau_W} \frac{\tPP_\theta^{\pi}(m,\otau_W|s_1)\tPP_\theta^{\pi}(l,\otau_W|s_1)}{\tPP_\theta^{\pi}(\otau_W|s_1)} \tPP_\theta(s_1).
\end{align*}
For any $s\in\cS$ and $m\in[L]$, we denote $\rho_{m|s}=\tPP_\theta(m|s_1=s)$, and then
\begin{align*}
    \tPP_\theta^{\pi}(m,\otau_W|s_1=s)=\rho_{m|s} \TT_{\theta,m}^{\pi}(\otau_W|s_1=s), \qquad \tPP_\theta^{\pi}(\otau_W|s_1=s)=\sum_{m} \rho_{m|s} \TT_{\theta,m}^{\pi}(\otau_W|s_1=s),
\end{align*}
Therefore, using the fact that %
\begin{align*}
    \tPP_\theta^{\pi}(\otau_W|s_1=s)\geq 2\sqrt{\rho_{m|s}\rho_{l|s}}\cdot \sqrt{\TT_{\theta,m}^{\pi}(\otau_W|s_1=s)\TT_{\theta,l}^{\pi}(\otau_W|s_1=s)},
\end{align*}
we have
\begin{align*}
    \sum_{\otau_W} \frac{\tPP_\theta^{\pi}(m,\otau_W|s_1)\tPP_\theta^{\pi}(l,\otau_W|s_1)}{\tPP_\theta^{\pi}(\otau_W|s_1)}
    \leq &~
    \frac{\sqrt{\rho_{m|s}\rho_{l|s}}}{2} \sum_{\otau_W} \sqrt{\TT_{\theta,m}^{\pi}(\otau_W|s_1)\TT_{\theta,l}^{\pi}(\otau_W|s_1)} \\
    =&~
    \frac{\sqrt{\rho_{m|s}\rho_{l|s}}}{2}\exp\paren{ -\DB{\MM_{m,W}^\theta(\pi,s_1), \MM_{l,W}^\theta(\pi,s_1)} }.
\end{align*}
Thus, taking summation over $m\neq l$ and using $\sum_{m\neq l} \sqrt{\rho_{m|s}\rho_{l|s}}\leq L-1$ gives
\begin{align*}
    \tPP_\theta^{\pi}(m_\theta(\otau_W)\neq m^\star)\leq L\exp(-\om(W)).
\end{align*}
\qed

\section{Proofs for Section~\ref{sec:stat-lower}}\label{appdx:stat-lower}

We first present two theorems that provide a more precise statement of our sample complexity lower bounds.
\begin{theorem}\label{thm:log-L}
There are constants $c, C$ so that for any $H\geq 1$, $\delta\in(0,\frac{1}{4e^2}]$, $L\geq 2$ and integer $2\leq n\leq H-1$ satisfying
\begin{align}\label{eqn:log-L-constraints}
    Cn\log^4 n \leq \min\set{ \frac{\log L}{H\delta^2}, \delta^{-1}, 2^{c\sqrt{\log L}} },
\end{align}
there exists a class of $\delta$-strongly separated LMDPs with $L$ hidden MDPs, $S=(\log L)^{\cO(\log n)}$ states, $A$ actions, and horizon $H$, so that any algorithm requires $\Om{\min\set{A,L}^{n-1}}$ samples to learn an $\frac{1}{4n}$-optimal policy.
\end{theorem}

\newcommand{\Neps}{N_{n,\delta}}

\begin{theorem}\label{thm:log-eps}
For any $\delta\in(0,\frac{1}{4e^2}]$ and integer $n\geq 2$, there is $\Neps\leq 2^{\cO((1+\delta n)\log^2 n)}$ so that for any $\eps>0$, integer $H, A\geq 2$ satisfying
\begin{align}\label{eqn:log-eps-constraints}
    n< H\leq \frac{\log(1/\eps)}{40\delta^2}+n, \qquad
    \eps\leq \frac{1}{\Neps},
\end{align}
there exists a class of $\delta$-strongly separated LMDPs with parameters $(L,S,A,H)$, where 
\begin{align*}
    L\leq \Neps, \qquad
    S\leq H^{\cO((1+\delta n)\log^2 n)},
\end{align*}
such that any algorithm requires $\Om{A^{n-1}}$ samples to learn an $\eps$-optimal policy.
\end{theorem}

We also present a slightly more general version of \cref{thm:A-exp}, as follows.
\begin{theorem}\label{thm:A-exp-full}
Suppose that $\delta\in(0,\frac{1}{4e^2}]$, $H\geq n+1\geq 3$, $A\geq 2$, $L\geq 2^{C\log n\log(1/\delta)}$ are given such that
\begin{align}\label{eqn:A-exp-constraints-full}
    CH\log n \log(1/\delta) \leq \frac{\log L}{\delta}.
\end{align}
Then there exists a class of $\delta$-strongly separated LMDP with $L$ hidden MDPs, $S=(\log L)^{\cO(\log n)}$ states, $A$ actions, horizon $H$, such that any algorithm requires $\Om{A^{n-1}}$ samples to learn an $\frac{1}{4n}$-optimal policy with probability at least $\frac{3}{4}$. %
\end{theorem}

Based on the results above, we can now provide a direct proof of \cref{thm:lower-demo}. In our proof, it turns out that we can take $c_\degC=\frac{1}{\Tilde{\Theta}(\degC)}$.

\begin{proofof}{\cref{thm:lower-demo}}
Fix $n=3\degC+1$, $\delta_0=\frac{1}{4e^2}$. We proceed to prove \cref{thm:lower-demo} by decomposing
\begin{align*}
    \log(L/\eps) = \log(L) + \log(1/\eps) \leq \frac 12 \max\{ \log L, \log (1/\eps) \},
\end{align*}
and then show that $\Hthre(L, \eps, \delta)$ must be greater than each of the terms in the maximum above, by applying \cref{thm:log-L}, \cref{thm:A-exp-full}, and \cref{thm:log-eps} separately.  %

Let $n_1=\cO(n\log^4 n)$ be the LHS of \cref{eqn:log-L-constraints}, and $N=N_{n,\delta_0}\leq 2^{\cO(n\log^2 n)}$ be given by \cref{thm:log-eps}. We choose $L_\degC:=2^{C_1 n_1\log^2 n_1 }$ for some large absolute constant $C_1$ so that $L_\degC\geq N$, and set $\eps_\degC=\frac{1}{N}$, $c_\degC=\frac{1}{C_1n_1\log^2 n}$. In the following, we work with $L\geq \max(L_\degC,\delta^{-1})$, $\eps\leq \eps_\degC$.

\textbf{Part 1.} In this part, we prove the lower bound involving the term $\log L$. We separately consider the case $\delta\leq \frac{1}{n_1}$ (\cref{thm:log-L}) and $\delta>\frac{1}{n_1}$ (using \cref{thm:A-exp-full}).

\textbf{Case 1: $\delta\leq \frac{1}{n_1}$.} In this case, we take $H_L=\max\paren{\floor{\frac{\log L}{n_1\delta^2}},n_1}$. For $H=H_L$ and any $A\geq 2$, applying \cref{thm:log-L} gives a class of \sepstr~LMDPs with parameters $(L,S_1,A,H)$ where $S_1\leq (\log L)^{\cO(\log n)}$, so that any algorithm requires $\Om{(A\wedge L)^{n-1}}$ samples for learning $\eps_\degC$-optimal policy (because $\eps_\degC\leq \frac{1}{4n}$). %
However, for $A=L$, we have assumed that $\fA$ succeeds with $\max\set{S_1,L,H_L,\eps_\degC^{-1},\delta^{-1}}^{\degC}\leq L^{n-1}$ samples. Therefore, since we have assumed that $\fA$ outputs an $\eps$-optimal policy if $H \geq \Hthre(L, \eps, \delta)$, we must have $H_L<\Hthre(L,\eps,\delta)$.

\textbf{Case 2: $\delta>\frac{1}{n_1}$.} In this case, we take $H_L=\floor{\frac{\log L}{C_1\log^2(n)\delta}}$. By definition, $H_L>n$. Hence, for $H=H_L$ and any $A\geq 2$, applying \cref{thm:A-exp-full} gives a class of \sepstr~LMDPs with parameters $(L,S_2,A,H)$ where $S_2\leq (\log L)^{\cO(\log n)}$, so that any algorithm requires $\Om{A^{n-1}}=\Om{A^{\degC+1}}$ samples for learning $\eps$-optimal policy.
However, for $A\geq \max\set{L,S_2,H,\eps^{-1},\delta^{-1}}$, we have assumed that $\fA$ succeeds with $A^{\degC}$ samples, as long as $H \geq \Hthre(L, \eps, \delta)$. Therefore, we must have $H_L<\Hthre(L,\eps,\delta)$.

Therefore, in both cases, we have $H_L<\Hthre(L,\eps,\delta)$. By definition, it always holds that $H_L\geq \frac{1}{C_1n_1\log^2 n}\cdot \frac{\log L}{\delta^2}$, and the desired result of this part follows.

\textbf{Part 2.} We take $H_\eps=\floor{\frac{\log(1/\eps)}{9\delta^2}}+n$. For any $H \leq H_\eps, A \geq 2$, \cref{thm:log-eps} provides a class of \sepstr~LMDPs with parameters $(L_3,S_3,A,H)$ with $L_3=N$ and $S_3 \leq H^{\cO((1+\delta n)\log^2 n)}$, so that any algorithm requires $\Om{A^{n-1}}=\Om{A^{\degC+1}}$ samples for learning $\eps$-optimal policy. However, for values $A\geq \max\set{N,S_3,H,\eps^{-1},\delta^{-1}}$, we have assumed that $\fA$ succeeds with $A^{\degC}$ samples. Therefore, since we have assumed that $\fA$ outputs an $\eps$-optimal policy if $H \geq \Hthre(L, \eps, \delta)$, we must have $H_\eps<\Hthre(L,\eps,\delta)$.

Combining the two parts above completes the proof of \cref{thm:lower-demo}.
\end{proofof}

In the remaining part of this section, we present the proof of \cref{thm:log-L}, \cref{thm:log-eps} and \cref{thm:A-exp-full}.

\paragraph{Organization} In \cref{appdx:comb-lock}, we present the hard instances of general (non-separated) LMDP \citep{kwon2021rl}. Then we present our tools of transforming LMDP into separated LMDP in \cref{appdx:lower-tool}. The proofs of \cref{thm:A-exp}, \cref{thm:log-L} and \cref{thm:log-eps} then follow.

\paragraph{Additional notations} For any step $h$, we write $\tau_h=(s_1,a_1,\cdots,s_h,a_h)$ and $\tau_{h:h'}=(s_h,a_h,\cdots,s_{h'},a_{h'})$. Denote
\begin{align}\label{eqn:def-P-theta}
    \PP_\theta(\tau_h)=\PP_\theta(s_{1:h}|\doac(a_{1:h-1})),
\end{align}
i.e., the probability of observing $s_{1:h}$ if the agent deterministically executes actions $a_{1:h-1}$ in the LMDP $M_\theta$. Also denote $\pi(\tau_h)\defeq \prod_{h'\le h} \pi_{h'}(a_{h'}|\tau_{h'-1}, s_{h'})$, and then $\PP^{\pi}_\theta(\tau_h)=\PP_\theta(\tau_h)\times \pi(\tau_h)$ %
gives the probability of observing $\tau_h$ for the first $h$ steps when executing $\pi$ in LMDP $M_\theta$. %

\subsection{Lower bound constructions for non-separated LMDPs}\label{appdx:comb-lock}
In this section, we review a lower bound of \cite{kwon2021rl} on the sample complexity of learning latent MDPs \emph{without} separation constraints; we state and prove some intermediate lemmas regarding this lower bound which are useful later on in our proofs. 
\newcommand{\sg}[1]{s_{\oplus,#1}}
\newcommand{\sm}{s_{\ominus}}
\begin{theorem}[\cite{kwon2021rl}]
For $n\geq 1$, there exists a class of LMDP with $L=n$, $S=n+1$, $H=n+1$, such that any algorithm requires $\Om{A^{n-1}}$ samples to learn an $\frac{1}{2n}$-optimal policy.
\end{theorem}

In the following, we present the construction in \citet{kwon2021rl} of a family of LMDPs
\begin{align}
  \label{eq:define-mtheta-lb}
    \cM=\set{ M_{\theta}: \theta\in\cA^{n-1} } \cup\set{ M_{\emptyset} }.
\end{align}
For any $\theta=\ba\in\cA^{n-1}$, we construct a LMDP $M_{\theta}$ as follows.
\begin{itemize}
\item The state space is
\begin{align*}
    \cS_0=\set{\sm, \sg{1},\cdots, \sg{n}}.
\end{align*}
\item The action space is $\cA$ and the horizon is $H\geq n+1$.
\item $L=n$, and for each $m\in[n]$, the MDP $M_{\theta,m}$ has mixing weight $\frac1n$.
\item In the MDP $M_{\theta,m}$, the initial state is $\sg{1}$, and the state $\sm$ is an absorbing state. 

For $m>1$, the transition dynamics of $M_{\theta,m}$ is given as follows.
\begin{itemize}
    \item At state $\sg{h}$ with $h<m-1$, taking any action leads to $\sg{h+1}$.
    \item At state $\sg{m-1}$, taking action $a\neq \ba_{m-1}$ leads to $\sg{m}$, %
    and taking action $\ba_{m-1}$ leads to $\sm$.
    \item At state $\sg{h}$ with $m\leq h<n$, taking action $a\neq \ba_{h}$ leads to $\sm$, and taking action $\ba_{h}$ leads to $\sg{h+1}$.
    \item At state $\sg{n}$, taking any action leads to $\sm$.
\end{itemize}
The transition dynamics of $M_{\theta,1}$ is given as follows.
\begin{itemize}
    \item At state $\sg{h}$ with $h<n$, taking action $a\neq \ba_{h}$ leads to $\sm$, and taking action $\ba_{h}$ leads to $\sg{h+1}$.
    \item The state $\sg{n}$ is an absorbing state.
\end{itemize}
\item The reward function is given by $R_h(s,a)=\indic{s=\sg{n}, h=n+1}$.
\end{itemize}

\paragraph{Construction of the reference LMDP} For $\otheta=\emptyset$, we construct a LMDP with state space $\cS_0$ and MDP instances $M_{\otheta,1}=\cdots=M_{\otheta,n}$ with mixing weights $\rho=\unif([n])$, where the initial state is always $\sg{1}$ and the transition is given by
\begin{align*}
    \TT_{\otheta,m}(\sg{h+1}|\sg{h},a)=\frac{n-h}{n-h+1}, \qquad
    \TT_{\otheta,m}(\sm|\sg{h},a)=\frac{1}{n-h+1}, \qquad \forall h\in[n],
\end{align*}
and $\sm$ is an absorbing state.

\newcommand{\bsq}{\mathbf{s}}
Define $\Theta=\cA^{n-1}\sqcup\set{\otheta}$. An important observation is that for any $\theta\in\Theta$, in the LMDP $M_\theta$, any reachable trajectory $\tau_H$ must have $s_{1:H}$ belonged to one of the following sequences
\begin{align*}
    \bsq_h=&~(\sg{1},\cdots,\sg{h},\underbrace{\sm,\cdots,\sm}_{H-h}), \quad \text{for some $h\in[n]$}, \\
    \text{or }~\bsq_{n,+}=&~(\sg{1},\cdots,\sg{n},\underbrace{\sg{n},\cdots,\sg{n}}_{H-n}).
\end{align*}
In particular, for any action sequence $a_{1:H}$, we have
\begin{align}\label{eqn:probs-ref-model}
    \PP_\otheta(s_{1:H}=\bsq_h|a_{1:H})=\frac{1}{n}, \qquad \forall h\in[n].
\end{align}

We summarize the crucial property of the LMDP class $\set{M_\theta}_{\theta\in\Theta}$ in the following lemma.

\begin{lemma}\label{lem:comb-lock}
For each $\theta=\ba\in\cA^{n-1}$, the following holds.

(a) For any action sequence $a_{1:H}$ such that $a_{1:n-1}\neq \ba$, it holds
\begin{align}\label{eqn:probs-incorrect-acs}
    \PP_\theta(s_{1:H}=\bsq_h|a_{1:H})=\frac{1}{n}, \qquad \forall h\in[n].
\end{align}
On the other hand, for the action sequence $a_{1:H}$ such that $a_{1:n-1}=\ba$,
\begin{align}\label{eqn:probs-correct-acs}
    \PP_\theta(s_{1:H}=\bsq_{n,+}|a_{1:H})=\frac{1}{n}, \qquad
    \PP_\theta(s_{1:H}=\bsq_h|a_{1:H})=\frac{1}{n}, \qquad \forall h\in[n-1].
\end{align}

(b) For any policy $\pi$, define
\begin{align}
  \label{eq:wtheta-def}
    w_\theta(\pi)=\prod_{h=1}^n \pi(a_h=\ba_h|\sg{1},\ba_1,\cdots,\sg{h}).
\end{align}
Then $\sum_{\theta\in\cA^{n-1}} w_\theta(\pi)=1$, and it also holds that
\begin{align*}
    V_{\theta}(\pi)=\frac{1}{n}w_{\theta}(\pi), \qquad
    \DTV{ \PP_\theta^\pi, \PP_\otheta^\pi }=\frac{1}{n}w_{\theta}(\pi).
\end{align*}
In particular, the optimal value in $\theta$ is $V_{\theta}^\star=\frac{1}{n}$, attained by taking $\ba$ in the first $n-1$ steps. 
\end{lemma}

\begin{proof}
We first prove (a). We inductively prove the following fact. 

\textbf{Fact:} For $1\leq h < n$ and any action sequence $a_{1:h}$, there is a unique index $m\in[h]$ such that in the MDP $M_{\theta,m}$, taking action sequence $a_{1:h}$ leads to the trajectory $\sg{1}\to\cdots\to\sg{h}\to\sm$. 

The base case $h=1$ is obvious. Suppose that the statement holds for all $h'<h$. Then in the MDP $M_{\theta,1}, \cdots, M_{\theta,h}$, there are $h-1$ many MDPs such that taking $a_{1:h-1}$ leads to $\sm$ at some step $<h$, and hence there is exactly one index $m'$ such that in $M_{\theta,m'}$, taking $a_{1:h-1}$ leads to the state $\sg{h}$. Therefore, if $a_h\neq \ba_h$, then taking $a_{1:h}$ in $M_{\theta,m'}$ leads to $\sg{1}\to\cdots\to\sg{h}\to\sm$. Otherwise, we have $a_h=\ba_h$, and $a_{1:h}$ in $M_{\theta,h}$ leads to $\sg{1}\to\cdots\to\sg{h}\to\sm$. The uniqueness is also clear, because for $l>h$, taking $a_{1:h}$ always lead to $\sg{h+1}$. This completes the proof of the case $h$.

Now, we consider any given action sequence $a_{1:H}$. For any step $h<n$, there exists a unique index $m(h)$ such that in the MDP $M_{\theta,m(h)}$, taking action sequence $a_{1:n}$ leads to the trajectory $\sg{1}\to\cdots\to\sg{h}\to\sm\to\cdots$. Thus, there is also a unique index $m(n)$ such that in the MDP $M_{\theta,m(n)}$, taking action sequence $a_{1:n-1}$ leads to the trajectory $\sg{1}\to\cdots\to\sg{n}$. Then there are two cases: (1) $a_{1:n-1}\neq \ba$, then $m(n)\neq 1$, and hence taking $a_{1:H}$ leads to the trajectory  $\sg{1}\to\cdots\to\sg{n}\to\sm\to\cdots$ in $M_{\theta,m(n)}$. (2) $a_{1:n-1}=\ba$, which implies $m(n)=1$, and hence taking $a_{1:H}$ in $M_{\theta,m(n)}$ leads to the trajectory $\sg{1}\to\cdots\to\sg{n}\to\sg{n}\to\cdots$. This completes the proof of (a).

We next prove (b) using (a). Notice that $V_\theta(\pi)=\PP_\theta^\pi(s_{n+1}=\sg{n})$. By definition, $s_{h+1}=\sg{n}$ can only happen when the agent is in the MDP $M_{\theta,1}$ and takes actions $a_{1:n}=\ba$, and hence
\begin{align*}
    \PP_\theta^\pi(s_{n+1}=\sg{n})
    =&~\PP_\theta^\pi(s_1=\sg{1},a_1=\ba_1,\cdots,s_n=\sg{n},a_n=\ba_n) \\
    =&~\frac1n\TT_{\theta,1}^\pi(s_1=\sg{1},a_1=\ba_1,\cdots,s_n=\sg{n},a_n=\ba_n) \\
    =&~ \frac{1}{n}\prod_{h=1}^n \pi(a_h=\ba_h|\sg{1},\ba_1,\cdots,\sg{h}) 
    =\frac{1}{n}w_\theta(\pi).
\end{align*}
More generally, we have
\begin{align*}
    2\DTV{ \PP_\theta^\pi, \PP_\otheta^\pi }
    =&~\sum_{\tau_H} \pi(\tau_H)\times \abs{ \PP_\theta(\tau_H)-\PP_\otheta(\tau_H) } \\
    =&~\sum_{\tau_H: s_{1:H}=\bsq_{n,+},a_{1:n-1}=\ba} \pi(\tau_H)\times \abs{ \frac{1}{n}-0 } + \sum_{\tau_H: s_{1:H}=\bsq_{n},a_{1:n-1}=\ba} \pi(\tau_H)\times \abs{ 0-\frac{1}{n} } \\
    =&~\frac{2}{n} \pi(\sg{1},\ba_1,\cdots,\sg{n-1},\ba_{n-1}),
\end{align*}
where the second equality is because $\PP_\theta(\tau_H)\neq \PP_\otheta(\tau_H)$ only when $s_{1:H}\in\set{\bsq_n,\bsq_{n,+}}$ and $a_{1:n-1}=\ba$, and the last line follows from recursively applying $\sum_{a_h} \pi(a_h|\tau_{h-1},s_h)=1$. This completes the proof of (b).
\end{proof}

\subsection{Tools}\label{appdx:lower-tool}

\begin{definition}\label{def:MDP-tensor}
Suppose that $M=(\cS,\cA,\TT,\mu,H)$ is a MDP instance, $\cO$ is a finite set, and  $\mu\in\Delta(\cO)$ is a distribution. Then we define $M\otimes \mu$ to be the MDP instance given by $(\cS\times \cO,\cA,\TT\otimes\mu,\rho\otimes \mu,H)$, where we define
\begin{align*}
    [\TT\otimes\mu]((s',o')|(s,o),a)=\TT(s'|s,a)\cdot \mu(o').
\end{align*}
\end{definition}

Given a finite set $\cO$, \cref{def:dist-family} introduces a property of a collection of distributions $\mu_1, \ldots, \mu_{L'} \in \Delta(\cO)$ which, roughly speaking, states that the distributions $\mu_i$ are \emph{separated} in total variation distance but that certain mixtures of $H$-wise tensorizations of the distributions $\mu_i$ are \emph{close} in total variation distance. Given that such collections of distributions exist, we will ``augment'' the hard instance of (non-separated) LMDPs from \cref{appdx:comb-lock} with the $\mu_i$ (per \cref{def:MDP-tensor}) to create hard instances of separated LMDPs. 
\begin{definition}
  \label{def:dist-family}
A $(L,H,\delta,\gamma,L')$-family over a space $\cO$ is a collection of distributions $\set{ \mu_i }_{i\in[L']}\subset \Delta(\cO)$ and $\xi_{1},\cdots,\xi_{L}\in\Delta([L'])$ such that the following holds:

(1) $\supp(\xi_k)\cap\supp(\xi_l)=\emptyset$ for all $k,l \in [L]$ with $k\neq l$.

(2) The distribution $\bQ_k:=\EE_{i\sim \xi_k}\brac{ 
\mu_i^{\otimes H} }\in\Delta(\cO^H)$ satisfies $\DTV{\bQ_k, \bQ_1}\leq \gamma$ for all $k\in[L]$.

(3) $\DTV{\mu_i,\mu_j}\geq \delta$ for all $i\neq j$, $i,j\in\cup_k\supp(\xi_k)$.
\end{definition}

\cref{prop:2-family,lem:family-tensor} state that $(L, H, \delta, \gamma, L')$-families exist, for appropriate settings of the parameters. 
\begin{proposition}\label{prop:2-family}
Suppose that $H\geq 1$, $\delta\in(0,\frac{1}{4e^2}]$. Then the following holds:

(a) Let $d=\ceil{4e^2\delta H}$. Then there exists a $(2,H,\delta,0,N)$-family over $[2d]$ with $N\leq \min\paren{\frac{1}{2e\delta},2H}^d$.

(b) Suppose $\lambda\in[1,\frac{1}{4e^2\delta}]$ is a real number and $d\geq \lambda\cdot 4e^7\delta^2 H$. Then there exists is a $(2,H,\delta,\gamma,N)$-family over $[2d]$ with $\gamma\leq 4e^{-\lambda d}$ and $N\leq (2e(\lambda+1))^d$.
\end{proposition}

\begin{lemma}\label{lem:family-tensor}
    Suppose that $\cQ$ is a $(2,H,\delta,\gamma,L)$-family over a space $\cO$. Then there exists a $(2^r,H,\delta,r\gamma,L^r)$ family over space $\cO^{r}$.
\end{lemma}

Proofs of the two results above are deferred to Appendices~\ref{appdx:proof-2-family} and~\ref{appdx:proof-family-tensor}.

\newcommand{\Rext}{\tilde{R}}
\newcommand{\Rexth}{\tilde{R}_{h}}
\begin{definition}[Augmenting an MDP with a family]
Suppose that $M=(\cS,\cA,(M_m)_{m=1}^L,H,\rho,R)$ is a LMDP instance and $\cQ=(\set{\mu_i}_{i \in [L']},\set{\xi_m}_{m \in [L]})$ is a $(L,H,\delta,\gamma,L')$-family over $\cO$. Then $M\otimes\cQ=(\cS\times\cO,\cA,(M_i')_{i=1}^{L'},H,\rho',\Rext)$ is defined to be the following $\delta$-strongly separated LMDP instance:
\begin{itemize}
\item For each $i\in\cup_{m\in[L]}\supp(\xi_m) \subset [L']$, there is a unique index $m(i)\in[L]$ such that $i\in\supp(\xi_{m(i)})$; we define $M_i':=M_{m(i)}\otimes \mu_i$, with mixing weight $\rho'(i): = \rho_{m(i)}\cdot \xi_{m(i)}(i)$. 
\item The reward function $\Rext$ is given by $\Rexth((s,o),a)=R_h(s,a)$.
\end{itemize}
\end{definition}

\newcommand{\Piv}[1]{\Pi_{#1}}
\newcommand{\PPcq}[2]{\PP_{#1,\cQ}^{#2}}
\newcommand{\EEcq}[2]{\EE_{#1,\cQ}^{#2}}
\newcommand{\hPPcq}[2]{\hPP_{#1,\cQ}^{#2}}
\newcommand{\hEEcq}[2]{\widehat{\EE}_{#1,\cQ}^{#2}}
\newcommand{\Vcq}[2]{V_{#1,\cQ}(#2)}
\begin{proposition}\label{prop:property-tensor}
Suppose that $M_{\theta}=(\cS,\cA,(M_{\theta,m})_{m=1}^L,H,\rho,R)$ is a LMDP instance, $\cQ$ is a $(L,H,\delta,\gamma,L')$-family over $\cO$, so that $M_{\theta}\otimes\cQ$ is a LMDP with state space $\tcS=\cS\times\cO$. Let $\Piv{\cS}$ be the set of all $H$-step policies operating over $\cS$, and $\Piv{\tcS}$ be the set of all $H$-step policies operating over $\tcS$. 

For any policy $\pi\in\Piv{\cS}$, we let $\PPcq{\theta}{\pi}$ denote the distribution of trajectory under $\pi$ in the LMDP $M_{\theta}\otimes\cQ$, and we let $V_{\theta,\cQ}(\pi)$ denote the value function of $\pi$. Then the following statements hold:
\begin{itemize}
\item[(a)] We can regard $\Piv{\cS}$ as a subset of $\Piv{\tcS}$ naturally, because any policy $\pi\in\Piv{\cS}$ can operate over state space $\tcS=\cS\times\cO$ by ignoring the second component of the state $\ts\in\tcS$. Then, for any policy $\pi\in\Piv{\cS}$, $V_{\theta}(\pi)=V_{\theta,\cQ}(\pi)$. In particular, $V_\theta^\star\leq V_{\theta,\cQ}^\star$.
\item[(b)] For any policy $\pi\in\Piv{\tcS}$, we define $\pi_{\cQ}=\EE_{o_{1:H}\sim\bQ_1}\brac{ \pi(\cdot|o_{1:H}) }\in\Piv{\cS}$, i.e. $\pi_{\cQ}$ is the policy that executes $\pi$ over state space $\cS$ by randomly drawing a sequence $o_{1:H}\sim\bQ_1$ at the beginning of each episode. Then we have $\abs{\Vcq{\theta}{\pi}-V_{\theta}(\pi_{\cQ})}\leq \gamma$. 
\item[(c)] For LMDPs with parameters $\theta,\otheta$ and any policy $\pi\in\Piv{\tcS}$, it holds
\begin{align*}
    \DTV{ \PPcq{\theta}{\pi}, \PPcq{\otheta}{\pi} }\leq 2\gamma+\DTV{ \PP_\theta^{\pi_{\cQ}}, \PP_\otheta^{\pi_{\cQ}}  }.
\end{align*}
\end{itemize}
\end{proposition}

\begin{proof}
For any $\ts=(s,o)\in\tcS=\cS\times\cO$, we denote $\ts[1]=s$. Fact (a) follows directly from the definition: for any policy $\pi\in\Piv{\cS}$,
\begin{align*}
    \Vcq{\theta}{\pi}=\EEcq{\theta}{\pi}\brac{\sum_{h=1}^H \Rexth(\ts_h,a_h)}
    =\EEcq{\theta}{\pi}\brac{\sum_{h=1}^H R_h(\ts_h[1],a_h)}
    =\EE_\theta^\pi\brac{ R_h(s_h,a_h) },
\end{align*}
where the last equality is because the marginal distribution $\PPcq{\theta}{\pi}$ over $(\cS\times\cA)^H$ agrees with $\PP_\theta^\pi$ by our construction. This completes the proof of (a).

We next prove (b) and (c). In the following, we fix any policy $\pi\in\Piv{\tcS}$.

By definition, for any $\tau_H=(\ts_1,a_1,\cdots,\ts_H,a_H)\in(\tcS\times\cA)^H$, we have $\ts_h=(s_h,o_h)\in\cS\times\cO$, and
\begin{align*}
    \PPcq{\theta}{\pi}(\tau_H)
    =&~\sum_{i\in[L']}\rho'(i)\times \PP_{M_{\theta,i}'}^{\pi}(\tau_H) \\
    =&~\sum_{m\in[L]}\rho(m)\sum_{i}\xi_m(i) \PP_{M_{\theta,m}\otimes \mu_i}^{\pi}(\tau_H)\\
    =&~
    \sum_{m\in[L]} \rho(m) \sum_{i}\xi_m(i)\times \pi(\tau_H)\times \PP_{\theta,m}(s_{1:H}|a_{1:H})\times \mu_i(o_1)\cdots\mu_i(o_H) \\
    =&~\sum_{m\in[L]} \rho(m)\times \pi(\tau_H)\times \PP_{\theta,m}(s_{1:H}|a_{1:H})\times \bQ_m(o_{1:H}).
\end{align*}
Consider the distribution $\hPPcq{\theta}{\pi}\in\Delta((\tcS\times\cA)^{H})$ given as follows:
\begin{align*}
    \hPPcq{\theta}{\pi}(\tau_H)
    =&~\pi(\tau_H)\times\bQ_1(o_{1:H})\times\PP_\theta(s_{1:H}|a_{1:H})\\
    =&~ \pi(\tau_H)\times\bQ_1(o_{1:H})\times\sum_{m\in[L]}\rho(m)\PP_{\theta,m}(s_{1:H}|a_{1:H}).
\end{align*}
Then, by definition,
\begin{align*}
    \PPcq{\theta}{\pi}(\tau_H)-\hPPcq{\theta}{\pi}(\tau_H)
    =\pi(\tau_H)\times\sum_{m\in[L]}\rho(m)\PP_m(s_{1:H}|a_{1:H})\cdot\paren{ \bQ_m(o_{1:H})-\bQ_1(o_{1:H}) },
\end{align*}
and hence
\begin{align*}
    &~\DTV{ \PPcq{\theta}{\pi}, \hPPcq{\theta}{\pi} } \\
    =&~\frac12\sum_{\tau_H}\abs{\PPcq{\theta}{\pi}(\tau_H)-\hPPcq{\theta}{\pi}(\tau_H)}\\
    \leq&~\frac12\sum_{\tau_H}\pi(\tau_H)\times\sum_{m\in[L]}\rho(m)\PP_m(s_{1:H}|a_{1:H})\cdot\abs{ \bQ_m(o_{1:H})-\bQ_1(o_{1:H}) } \\
    =&~\frac12\sum_{m\in[L]} \rho(m) \sum_{o_{1:H}} \abs{ \bQ_m(o_{1:H})-\bQ_1(o_{1:H}) } \sum_{s_{1:H},a_{1:H}} \pi((s,o)_{1:H},a_{1:H})\times \PP_m(s_{1:H}|a_{1:H}) \\
    =&~ \frac12\sum_{m\in[L]} \rho(m) \sum_{o_{1:H}} \abs{ \bQ_m(o_{1:H})-\bQ_1(o_{1:H}) }
    \leq \gamma,
\end{align*}
where the last line follows from the fact that for any fixed $o_{1:H}$, $\pi((s,o)_{1:H},a_{1:H})\times \PP_m(s_{1:H}|a_{1:H})$ gives a probability distribution over $(s_{1:H},a_{1:H})$.

Let $\hEEcq{\theta}{\pi}$ be the expectation taken over $\hPPcq{\theta}{\pi}$. Then it holds that
\begin{align*}
    &~\hEEcq{\theta}{\pi}\brac{\sum_{h=1}^H \Rexth(\ts_h,a_h)} \\
    =&~\sum_{\tau_H} \pi(\tau_H)\times\bQ_1(o_{1:H})\times\PP_\theta(s_{1:H}|a_{1:H})\times \paren{\sum_{h=1}^H R_h(s_h,a_h)} \\
    =&~ \sum_{s_{1:H},a_{1:H}} \paren{ \sum_{o_{1:H}} \bQ_1(o_{1:H})\cdot \pi(a_{1:H}|s_{1:H},o_{1:H}) } \times\PP_\theta(s_{1:H}|a_{1:H}) \times  \paren{\sum_{h=1}^H R_h(s_h,a_h)} \\
    =&~ \sum_{s_{1:H},a_{1:H}} \pi_{\cQ}(a_{1:H}|s_{1:H}) \times\PP_\theta(s_{1:H}|a_{1:H}) \times  \paren{\sum_{h=1}^H R_h(s_h,a_h)} 
    = V_\theta(\pi_{\cQ}),
\end{align*}
where the last line follows from our definition of $\pi_{\cQ}$, which is a policy given by
\begin{align*}
    \pi_{\cQ}(\cdot)=\EE_{o_{1:H}\sim\bQ_1}\brac{ \pi(\cdot|o_{1:H}) }.
\end{align*}
Therefore, we can bound
\begin{align*}
    \abs{\Vcq{\theta}{\pi}-V_{\theta}(\pi_{\cQ})}
    =\abs{ \EEcq{\theta}{\pi}\brac{\sum_{h=1}^H \Rexth(\ts_h,a_h)}-\hEEcq{\theta}{\pi}\brac{\sum_{h=1}^H \Rexth(\ts_h,a_h)} }
    \leq \DTV{ \PPcq{\theta}{\pi}, \hPPcq{\theta}{\pi} }
    \leq \gamma,
\end{align*}
and hence complete the proof of (b).

Similarly, using the fact that $\DTV{ \PPcq{\theta}{\pi}, \hPPcq{\theta}{\pi} } \leq \gamma$ and $\DTV{ \PPcq{\otheta}{\pi}, \hPPcq{\otheta}{\pi} } \leq \gamma$, we have
\begin{align*}
    \DTV{ \PPcq{\theta}{\pi}, \PPcq{\otheta}{\pi} }\leq 2\gamma+\DTV{ \hPPcq{\theta}{\pi}, \hPPcq{\otheta}{\pi} }.
\end{align*}
Further, by definition,
\begin{align*}
    \DTV{ \hPPcq{\theta}{\pi}, \hPPcq{\otheta}{\pi} }
    =&~ \frac12\sum_{\tau_H} \pi(\tau_H)\times\bQ_1(o_{1:H})\times\abs{\PP_\theta(s_{1:H}|a_{1:H})-\PP_\otheta(s_{1:H}|a_{1:H})} \\
    =&~\frac12\sum_{s_{1:H},a_{1:H}} \paren{ \sum_{o_{1:H}} \bQ_1(o_{1:H})\cdot \pi(a_{1:H}|s_{1:H},o_{1:H}) }\times\abs{\PP_\theta(s_{1:H}|a_{1:H})-\PP_\otheta(s_{1:H}|a_{1:H})} \\
    =&~\frac12\sum_{s_{1:H},a_{1:H}} \pi_{\cQ}(a_{1:H}|s_{1:H})\times\abs{\PP_\theta(s_{1:H}|a_{1:H})-\PP_\otheta(s_{1:H}|a_{1:H})} \\
    =&~ \DTV{ \PP_\theta^{\pi_{\cQ}}, \PP_\otheta^{\pi_{\cQ}}  }.
\end{align*}
Combining the above two equations completes the proof of (c).
\end{proof}

\newcommand{\tcM}{\widetilde{\cM}}
\newcommand{\tw}{\Tilde{w}}
\renewcommand{\tM}{\widetilde{M}}

Fix an action set $\cA$ and $n \in \mathbb{N}$. Recall the MDPs $M_\theta$, indexed by $\theta \in \cA^{n-1}\cup\set{\emptyset}$, introduced in \cref{eq:define-mtheta-lb}. \cref{prop:tensor-comb-lock} below uses \cref{lem:comb-lock} to show that when these MDPs are augmented with a $(n, H, \delta, \gamma, L)$-family per \cref{def:MDP-tensor}, then the resulting family of LMDPs also requires many samples to learn. 
\begin{proposition}\label{prop:tensor-comb-lock}
Suppose that $n\geq 2$, $A\geq 2$, $H\geq n+1$, $\gamma\in[0,\frac{1}{4n})$, and $\cQ$ is a $(n,H,\delta,\gamma,L)$-family over $\cO$. Consider
\begin{align*}
    \tcM=\set{ M_{\theta}\otimes \cQ: \theta\in\cA^{n-1} } \cup\set{ M_{\emptyset}\otimes\cQ },
\end{align*}
which is a class of \sepstr~LMDPs with parameters $(L,S,A,H)$, where $S=(n+1)\abs{\cO}$.
Suppose $\fA$ is an algorithm such that for any $M\in\tcM$, $\fA$ interacts with $M$ for $T$ episodes and outputs an $\frac{1}{4n}$-optimal policy $\hpi$ for $M$ with probability at least $\frac34$. Then it holds that
\begin{align*}
    T\geq \frac{1}{8}\min\set{ \frac{1}{2\gamma}, A^{n-1}-2 }.
\end{align*}
\end{proposition}

\begin{proof}
In the following, we denote $\otheta=\emptyset$, consistently with the notations in \cref{appdx:comb-lock}.

Notice that by \cref{prop:property-tensor} (a), for any $\theta\in\cA^{n-1}$, we have $V_{\theta,\cQ}^\star\geq \frac{1}{n}$. Furthermore, for any $\pi\in\Piv{\tcS}$,
\begin{align*}
    \Vcq{\theta}{\pi}\leq V_\theta(\pi_{\cQ})+\gamma
    =\frac{1}{n}w_{\theta}(\pi_{\cQ})+\gamma.
\end{align*}
In the following, for each $\theta\in\cA^{n-1}$, we denote $\tM_\theta\defeq M_\theta\otimes \cQ$ and $\tw_\theta(\pi)=w_{\theta}(\pi_{\cQ})$ for any policy $\pi\in\Piv{\tcS}$ (recall the definition of $w_\theta(\cdot)$ in \cref{eq:wtheta-def}). Therefore, using item (b) of \cref{lem:comb-lock}, if $\pi$ is $\frac{1}{4n}$-optimal in $\tM_\theta$, then we have $\tw_\theta(\pi)\geq \frac{3}{4}-n\gamma>\frac{1}{2}$. Also notice that by \cref{prop:property-tensor} (c) and \cref{lem:comb-lock} (b), 
\begin{align}\label{eqn:tensor-tv-to-part}
    \DTV{ \PPcq{\theta}{\pi},\PPcq{\otheta}{\pi}  }\leq 2\gamma+\DTV{ \PP_\theta^{\pi_{\cQ}}, \PP_\otheta^{\pi_{\cQ}}  }
    =2\gamma+\tw_\theta(\pi).
\end{align}

Consider the following set of near-optimal policies in $\tM_\theta$:
\begin{align}
    \Pi_{\theta}^\star\defeq \set{\pi\in\Piv{\tcS}: V_{\theta,\cQ}^\star-\Vcq{\theta}{\pi} \leq \frac{1}{4n}}
    \subseteq\set{\pi\in\Piv{\tcS}: \tw_\theta(\pi)>\frac12}.
\end{align}
We know $\PPcq{\theta}{\fA}(\hpi\in\Pi_{\theta}^\star)\geq \frac{3}{4}$, where we use $\PPcq{\theta}{\fA}$ to denote the probability distribution induced by executing $\fA$ in the LMDP $\tM_\theta$. Using the fact (from \cref{lem:comb-lock}) that $\sum_{\theta\in\cA^{n-1}}\tw_\theta(\pi)=1$, we also know that $\Pi_{\theta}^\star\cap \Pi_{\theta'}^\star=\emptyset$ for any $\theta\neq \theta'\in\cA^{n-1}$. Therefore,
\begin{align*}
    \sum_{\theta\in\cA^{n-1}} \PPcq{\otheta}{\fA}(\hpi\in\Pi_{\theta}^\star) \leq 1.
\end{align*}
Hence, there is a set $\Theta_0\subset \cA^{n-1}$ such that $\abs{\Theta_0}\geq A^{n-1}-2$, and for each $\theta\in\Theta_0$, $\PPcq{\otheta}{\fA}(\hpi\in\Pi_{\theta}^\star)\leq \frac{1}{2}$, which implies that
\begin{align*}
    \DTV{ \PPcq{\theta}{\fA}, \PPcq{\otheta}{\fA} } \geq \frac{1}{4}, \qquad\forall \theta\in\Theta_0.
\end{align*}

Now we proceed to upper bound the quantity $\DTV{ \PPcq{\theta}{\fA}, \PPcq{\otheta}{\fA} }$. Notice that the algorithm $\fA$ can be described by interaction rules $\set{ \pi^{(t)} }_{t\in[T]}$, where $\pi^{(t)}$ is a function that maps the history $(\tau^{(1)},\cdots,\tau^{(t-1)})$ to a policy in $\Piall$ to be executed in the $t$-th episode. Then, by \cref{lemma:TV-cond}, it holds that
\begin{align*}
    \DTV{ \PPcq{\theta}{\fA}, \PPcq{\otheta}{\fA} }
    \leq \sum_{t=1}^T \EE_\otheta^\fA\brac{ \DTV{ \PPcq{\theta}{\pi^{(t)}}, \PPcq{\otheta}{\pi^{(t)}} } }
    = T\cdot \EE_{\pi\sim q_{\fA}}\brac{ \DTV{ \PPcq{\theta}{\pi}, \PPcq{\otheta}{\pi} } },
\end{align*}
where $q_{\fA}\in\Delta(\Piall)$ is the distribution of $\pi=\pi^{(t)}$ with $t\in\unif([T])$ and $(\pi^{(1)},\cdots,\pi^{(T)})\sim \PP_\otheta^{\fA}$. Therefore, using \cref{eqn:tensor-tv-to-part}, we know
\begin{align*}
    \DTV{ \PPcq{\theta}{\fA}, \PPcq{\otheta}{\fA} }
    \leq 2T\gamma+T\cdot \EE_{\pi\sim q_{\fA}}\tw_\theta(\pi),
\end{align*}
where the last equality follows from \cref{lem:comb-lock} (b). Taking summation over $\theta\in\Theta_0$, we obtain
\begin{align*}
    \abs{\Theta_0}\cdot 2T\gamma+T
    \geq \sum_{\theta\in\Theta_0} \paren{ 2T\gamma+T\cdot \EE_{\pi\sim q_{\fA}}\tw_\theta(\pi) }
    \geq \frac{1}{4} \abs{\Theta_0}.
\end{align*}
The desired result follows immediately.
\end{proof}

\subsection{Proof of Theorem~\ref{thm:A-exp} and Theorem~\ref{thm:A-exp-full}}\label{appdx:proof-A-exp}
\paragraph{Proof of \cref{thm:A-exp-full}} Fix a given $n\leq H-1$, we set $r=\ceil{\log_2 n}$. By \cref{prop:2-family} (a) and \cref{lem:family-tensor}, there exists a $(n,H,\delta,0,L_0)$-family over $[2d]^r$, where $d=\ceil{4e^2\delta H}$ and $L_0\leq \paren{\frac{1}{2e\delta}}^{dr}$. Notice that \cref{eqn:A-exp-constraints} and $\log L\geqsim \log n \log(1/\delta)$ together ensure that $L_0\leq L$. Hence, applying \cref{prop:tensor-comb-lock} completes the proof.
\qed

\paragraph{Proof of \cref{thm:A-exp}} Notice that for sufficiently large constant $C$, the presumptions of \cref{thm:A-exp} that $\log L\geq C\log^2(1/\delta)$ and \cref{eqn:A-exp-constraints} together ensure we can apply \cref{thm:A-exp-full} with $n=H-1$, and hence the proof is completed.
\qed

\subsection{Proof of Theorem~\ref{thm:log-L}}\label{appdx:proof-log-L}

Set $\lambda=2n\log^2n$. Also set %
\begin{align}\label{eqn:log-L-choose-d}
    d=\max\set{ \ceil{2\lambda^{-1}n\log L}, \ceil{\lambda\cdot 4e^7 H\delta^2} }.
\end{align}
Notice that we have $1\leq \lambda\leq \frac{1}{4e^2\delta}$ as long as we choose the absolute constant $C\geq 8e^2$ in \cref{eqn:log-L-constraints}. 
Then, applying \cref{prop:2-family} (b), there exists a $(2,H,\delta,\gamma,N)$-family over $[2d]$ with
\begin{align*}
    N\leq \paren{e(\lambda+1)}^d, \qquad
    \gamma \leq 4e^{-d\lambda}.
\end{align*}
Denote $r=\ceil{\log_2 n}$. By our assumption \cref{eqn:log-L-constraints}, we have $\log L\geq (c^{-1}\log n)^2$, and hence choosing $c$ sufficiently small and $C$ sufficiently large ensures that we have $N^r \leq L$. Further, by our choice of $d$ in \cref{eqn:log-L-choose-d}, we have $r\gamma\leq L^{-n}$. 

Hence, by \cref{lem:family-tensor}, there exists a $(n,H,\delta,L^{-n},L)$-family over $[2d]^{r}$, and we denote it as $\cQ$. Applying \cref{prop:tensor-comb-lock} to $\cQ$, we obtain a family $\tcM$ of $\delta$-strongly separated LMDPs, with state space $\tcS=\cS\times [2d]^r$, and any algorithm requires $\Om{A^n\wedge L^n}$ samples to learn $\tcM$. Noticing that $|\tcS|\leq (n+1)(2d)^r=(\log L)^{\cO(\log n)}$ completes the proof. 
\qed

\subsection{Proof of Theorem~\ref{thm:log-eps}}\label{appdx:proof-log-eps}

\newcommand{\boundS}{\partial \cS^+}
\newcommand{\oH}{\bar{H}}
\newcommand{\stermin}{\mathsf{terminal}}
\newcommand{\pmin}{\bar{p}}
\newcommand{\odelta}{\Bar{\delta}}
\newcommand{\bPP}{\Bar{\PP}}
\renewcommand{\tR}{\Tilde{R}}

Let $d_0=\ceil{4e^2\delta (n+1)}$,  $r=\ceil{\log_2 n}$, and $\oH=H-n-1$. By \cref{prop:2-family} and \cref{lem:family-tensor}, there exists a $(n,n+1,\delta,0,N)$-family over $[2d_0]^r$ with $N\leq \min\paren{\frac{1}{2e\delta},2n}^{d_0r}$. 
In particular, we choose $\Neps=(4nN)^2$, and then it holds that $\Neps=2^{\cO((1+\delta n)\log^2 n)}$. %

Applying \cref{prop:tensor-comb-lock} to this family, we obtain $\tcM$ a class of $\delta$-strongly separated LMDP with state space $\tcS=\cS_0\times[2d_0]^r$, action space $\cA$, horizon $n+1$. Recall that by our construction in \cref{prop:tensor-comb-lock} (and \cref{appdx:comb-lock}), for each $\theta\in\cA^{n-1}\cup\set{\otheta}$,$\tM_\theta$ is given by $(\tcS,\cA,(\tM_{\theta,m})_{m=1}^{N},n+1,\rho_{\theta},\tR)$, and the mixing weight $\rho_{\theta}\in\Delta([N])$ of the MDPs $\tM_{\theta,1}, \cdots, \tM_{\theta,N}$ does not depend on $\theta$, i.e. $\rho_\theta=\rho$ for a fixed $\rho\in\Delta([N])$. Furthermore, for each $m\in[N]$, the initial distribution $\nu_{\theta,m}$ of $\tM_{\theta,m}$ is also independent of $\theta$, i.e. $\nu_{\theta,m}=\nu_m$ for a fixed $\nu_m\in\Delta(\tcS)$. We also know that $\tR=(\tR_h:\tcS\times\cA\to[0,1])_{h=1}^{n+1}$ is the reward function.

For each $\theta$, we construct an augmented \sepstr~LMDP $\tM_\theta^+$ with horizon $H$, as follows. 

Fix $d=2\ceil{C_1\log N}$ for a large absolute constant $C_1$ so that there exists $\mu_1,\cdots,\mu_N\in\set{-1,1}^d$ such that $\iprod{\mu_i}{\II}=0 \forall i\in[N]$ and $\lone{\mu_i-\mu_j}\geq d/2$ (see e.g. \cref{lem:compute-net}). Denote $\odelta=4\delta$ and set $\eta=\frac12$.
\begin{itemize}
\item The state space is $\tcS^+=\tcS\sqcup \cS^+\sqcup\set{\stermin_1,\cdots,\stermin_N}$, where
\begin{align*}
    \cS^+=\set{ (k_1,\cdots,k_d)\in\NN^d: k_1+\cdots+k_d\leq \oH-1 }.
\end{align*}
We will construct the transition so that at the state outside $\tcS$, the transition does not depend on $\theta$. 
We also write $\boundS=\set{ (k_1,\cdots,k_d)\in\NN^d: k_1+\cdots+k_d=\oH-1 }$.
\item The initial state is always $(0,\cdots,0)\in\cS^+$.
\item For $s\in\cS^+\backslash \boundS$, we set
\begin{align*}
    \TT_{m}(s+\be_i|s,a)=\frac{1+\odelta\mu_m[i]}{d}.
\end{align*}
\item For $s\in\boundS$, we define
\begin{align*}
    p_m(s)=\prod_{i=1}^d (1+\odelta\mu_m[i])^{s[i]},
\end{align*}
and we set $\pmin(s)=\min_{l\in[N]}p_l(s)$, 
\begin{align*}
    \TT_{m}(s'|s,a)=\eta\frac{\pmin(s)}{p_m(s)}\cdot \nu_m(s'), \qquad s'\in\tcS,
\end{align*}
and $\TT_{m}(\stermin_m|s,a)=1-\eta\frac{\pmin(s)}{p_m(s)}$.
\item For state $s\in\set{\stermin_1,\cdots,\stermin_N}$, we set $\TT_m(\stermin_m|s,a)=1$.
\item The reward function is given by $\tR^+_h=0$ for all $h\in[\oH]$, and $\tR^+_{\oH+h}=\tR_h$ for $h\in[n+1]$.
\end{itemize}

By our construction, it is clear that $\tM_\theta^+$ is $\delta$-strongly separated, and $|\tcS^+|\leq n+N+2+H^d$. %

Furthermore, we can also notice that for any trajectory $\tau_H=(s_{1:H},a_{1:H})$ such that $s_{\oH+1}\not\in\tcS$, the probability $\PP_{\theta,+}(\tau_H)=\PP_+(\tau_H)$ does not depend on $\theta$. Furthermore, for any trajectory $\tau_{\oH}$, the probability $\PP_{\theta,+}(\tau_H)=\PP_+(\tau_H)$ is also independent of $\theta$.

Now, we consider the event $E=\set{s_{\oH+1}\in\tcS}$. Notice that the probability $\PP_{\theta,+}(E)=p$ also does not depend on $\theta$.

\begin{lemma}\label{lem:log-eps-probs}
For any trajectory $\tau_{\oH}=(s_{1:\oH},a_{1:\oH})$, we have
\begin{align*}
    \PP_{\theta,+}(\tau_{\oH+1:H}=\cdot|E,\tau_{\oH})=\PPcq{\theta}{}(\tau_{1:n+1}=\cdot),
\end{align*}
which does not depend on $\tau$. 
\end{lemma}

\begin{proof}
For any reachable trajectory $\tau_{\oH}=(s_{1:\oH},a_{1:\oH})$, we have $s_{h+1}=s_h+\be_{i_h}$ for all $h<\oH$. Hence, for $m\in[N]$ and $s\in\tcS$,
\begin{align*}
    \PP_{\tM_{\theta,m}^+}(\tau_{\oH}, s_{\oH+1}=s)
    =&~ \prod_{h=1}^{\oH} \TT_m(s_{h+1}|s_h,a_h) \\
    =&~ \TT_m(s|s_{\oH},a_{\oH}) \times \prod_{h=1}^{\oH-1} \frac{1+\odelta\mu_m[i_h]}{d} \\
    =&~ \TT_m(s|s_{\oH},a_{\oH}) \times \frac{1}{d^{\oH-1}} \prod_{i=1}^d (1+\odelta\mu_m[i])^{s_{\oH}[i]} \\
    =&~ \nu_m(s)\times \eta \frac{\pmin(s_{\oH})}{p_m(s_{\oH})} \times \frac{p_m(s_{\oH})}{d^{\oH-1}} \\
    =&~ \eta \nu_m(s)\times \frac{\pmin(s_{\oH})}{d^{\oH-1}}, 
\end{align*}
which is independent of $\theta$. Hence, for any $\theta\in\Theta$, we have
\begin{align*}
    \tPP_{\theta,+}(\ms=m,s_{\oH+1}=s|E,\tau_{\oH})
    =
    \frac{\rho(m)\PP_{\tM_{\theta,m}^+}(\tau_{\oH}, s_{\oH+1}=s)}{\sum_{l\in[N]}\sum_{s\in\tcS} \rho(m)\PP_{\tM_{\theta,l}^+}(\tau_{\oH}, s_{\oH+1}=s) } = \rho(m)\nu_m(s).
\end{align*}
In other words, conditional on the event $E$ and any reachable trajectory $\tau_{\oH}$, the posterior distributions of $(\ms,s_{\oH+1})$ in $\tM_\theta^+$ is the same as the distribution of $(\ms,s_1)$ in $\tM_\theta$. Hence, for any trajectory $\tau\in(\tcS\times\cA)^{H-\oH}$ that starts with $s\in\tcS$, we have
\begin{align*}
    &~\PP_{\theta,+}(\tau_{\oH+1:H}=\tau|E,\tau_{\oH}) \\
    =&~ 
    \sum_{m\in[N]} \tPP_{\theta,+}(\tau_{\oH+1:H}=\tau|\ms=m,s_{\oH+1}=s)\cdot \tPP_{\theta,+}(\ms=m,s_{\oH+1}=s|E,\tau_{\oH}) \\
    =&~
    \sum_{m\in[N]} \rho(m)\nu_m(s) \PP_{\tM_{\theta,m}}(\tau_{\oH+1:H}=\tau|\ms=m,s_{\oH+1}=s) \\
    =&~
    \PPcq{\theta}{}(\tau_{1:n+1}=\tau),
\end{align*}
where in the second equality we also use the fact that in the MDP $\tM_{\theta,m}^+$ and starting at state $s\in\tcS$, the agent will stay in $\tcS$, and the transition dynamics of $\tM_{\theta,m}^+$ over $\tcS$ agrees with $\tM_{\theta,m}$. This completes the proof of \cref{lem:log-eps-probs}.
\end{proof}

Using the observations above and \cref{lem:log-eps-probs}, we know that for any policy $\pi\in\Piv{\tcS^+}$, we have
\begin{align*}
    V_{\theta,+}(\pi)=p\cdot \EE_{\tau_{\oH-1}|E}\brac{ \Vcq{\theta}{\pi(\cdot|\tau_{\oH-1})} },
\end{align*}
where $\PP_{\theta,+}(E)=p$, the expectation is taken over distribution of $\tau_{\oH-1}$ conditional on the event $E$, and $\pi(\cdot|\tau_{\oH-1})$ is regarded as a policy for the LMDP $\tM_\theta$ by conditional on the trajectory $\tau_{\oH-1}$ and restricting to $\tcS$. 

Therefore, for each $\pi\in\Piv{\tcS^+}$, there is a corresponding policy $\pi_+=\EE_{\tau_{\oH-1}|E}\brac{ \pi(\cdot|\tau_{\oH-1})}\in\Piv{\tcS}$, such that $V_{\theta,+}(\pi)=p\cdot \Vcq{\theta}{\pi_+}=p\tw_\theta(\pi)$. Similarly, we can also show that (using \cref{eqn:tensor-tv-to-part})
\begin{align*}
    \DTV{ \PP_{\theta,+}^\pi, \PP_{\otheta,+}^\pi }=
    p\DTV{ \PPcq{\theta}{\pi_+},\PPcq{\otheta}{\pi_+}  }\leq p\tw_\theta(\pi_+).
\end{align*}

The following lemma provides a lower bound of $p$ (the proof of \cref{lem:log-eps-prob-lower} is deferred to the end of this section).
\begin{lemma}\label{lem:log-eps-prob-lower}
It holds that
\begin{align*}
    \PP_{\theta,+}(E)=p\geq \frac{\eta}{N} (1-\odelta^2)^{\oH-1}.
\end{align*}
In particular, $p>2n\eps$.
\end{lemma}

With the preparations above, we now provide the proof of \cref{thm:log-eps}, whose argument is analogous to the proof of \cref{prop:tensor-comb-lock}. 

\paragraph{Proof of \cref{thm:log-eps}} 
Suppose that $\fA$ is an algorithm such that for any $M\in\tcM$, $\fA$ interacts with $M$ for $T$ episodes and outputs an $\frac{1}{4n}$-optimal policy $\hpi$ for $M$ with probability at least $\frac34$.

Notice that $V_{\theta,\cQ}^\star=\frac{p}{n}$, and $\epsilon<\frac{p}{2n}$. Thus, if $\hpi$ is $\frac{p}{4n}$-optimal in $\tM_\theta^+$, then $\tw_\theta(\pi)>\frac{1}{2}$. 
Now, consider the following set of near-optimal policies in $\tM_\theta^+$:
\begin{align}
    \Pi_{\theta,+}^\star\defeq \set{\pi\in\Piv{\tcS^+}: \text{$\pi$ is $\eps$-optimal in $\tM_\theta^+$} }.
\end{align}
Then $\Pi_{\theta,+}^\star$ are mutually disjoint for $\theta\in\cA^{n-1}$. We then have
\begin{align*}
    \PP_{\theta,+}^\fA(\hpi\in\Pi_{\theta,+}^\star)\geq \frac{3}{4}, \qquad 
    \sum_{\theta\in\cA^{n-1}} \PP_{\otheta,+}^\fA(\hpi\in\Pi_{\theta}^\star) \leq 1.
\end{align*}
Repeating the argument as in the proof of \cref{prop:tensor-comb-lock} gives $T\geq \frac{1}{4p}(A^{n-1}-2)$, and the desired result follows. \qed

\begin{proofof}{\cref{lem:log-eps-prob-lower}}
We next lower bound the probability $p$. By definition,
\begin{align*}
    \PP_{\theta,+}(s_{\oH+1}\in\tcS)
    =&~ \sum_{\tau_{\oH}\text{ reachable}, s_{\oH+1}\in\tcS} \PP_{\theta,+}(\tau_{\oH}, s_{\oH+1}) \\
    =&~ \sum_{\tau_{\oH}\text{ reachable}, s_{\oH+1}\in\tcS}\sum_{m\in[N]} \rho(m)\PP_{\tM_{\theta,m}^+}(\tau_{\oH}, s_{\oH+1}=s) \\
    =&~ \sum_{\tau_{\oH}\text{ reachable}}\eta \cdot\frac{\pmin(s_{\oH})}{d^{\oH-1}} \\
    =&~ \sum_{i_1,\cdots,i_{\oH-1}\in[d]} \frac{\eta}{d^{\oH-1}}\cdot \pmin\paren{\be_{i_1}+\cdots+\be_{i_{\oH-1}}} \\
    \geq&~ \frac{\eta}{d^{\oH-1}} \paren{ \sum_{i_1,\cdots,i_{\oH-1}\in[d]} \frac{1}{\pmin\paren{\be_{i_1}+\cdots+\be_{i_{\oH-1}}}} }^{-1},
\end{align*}
where in the last line we apply Cauchy inequality. Notice that for any $s\in\boundS$,
\begin{align*}
    \frac{1}{\pmin(s)}=\max_{l\in[N]}\frac{1}{p_l(s)}\leq \sum_{l\in[N]} \frac{1}{p_l(s)},
\end{align*}
and we also have
\begin{align*}
    \sum_{i_1,\cdots,i_{\oH-1}\in[d]} \frac{1}{p_m\paren{\be_{i_1}+\cdots+\be_{i_{\oH-1}}}} 
    =&~
    \sum_{i_1,\cdots,i_{\oH-1}\in[d]} \frac{1}{\prod_{h=1}^{\oH-1} (1+\odelta\mu_m[i_h])} 
    =
    \paren{ \sum_i \frac{1}{1+\odelta\mu_m[i]} }^{\oH-1} \\
    =&~ \paren{ \frac{d}{2}\times \frac{1}{1+\odelta} +\frac{d}{2}\times \frac{1}{1-\odelta}}^{\oH-1} = \frac{d^{\oH-1}}{(1-\odelta^2)^{\oH-1}},
\end{align*}
where the second line follows from the fact that $\mu_m\in\set{-1,1}^d$ and $\iprod{\mu_m}{\II}=0$. Combining the inequalities above gives $p\geq \frac{\eta}{N} (1-\odelta^2)^{\oH-1}$. 

In particular, to prove $p>2n\eps$, we only need to prove $(\oH-1)\log\frac{1}{1-\odelta^2}\leq \log(1/(4Nn\eps))$. Notice that $\log\frac{1}{1-\odelta^2}\leq \frac{\odelta^2}{1-\odelta^2}$, $\odelta=4\delta$, and we also have $\frac{1}{4nN\eps}\geq \frac{1}{\sqrt{\eps}}$ using $\eps\leq \frac{1}{\Neps}=\frac{1}{(4nN)^2}$. Combining these completes the proof.
\end{proofof}

\subsection{Proof of Proposition~\ref{prop:2-family}}\label{appdx:proof-2-family}

Towards proving \cref{prop:2-family}, we first prove the following proposition, which provides a simple approach of bounding TV distance between mixtures of distributions of a special form.

\begin{proposition}\label{prop:unif-moments}
Let $n, d \in \mathbb{N}$ be given. For $\bx\in[-1,1]^d$, we consider the distribution
\begin{align}\label{eqn:def-Px}
    \QQ_\bx=\brac{ \frac{1+\bx[1]}{2d}; \frac{1-\bx[1]}{2d}; \cdots; \frac{1+\bx[d]}{2d}; \frac{1-\bx[d]}{2d} } \in\Delta([2d]).
\end{align}
Then, for distributions $\mu,\nu$ over $[-1,1]^d$, it holds that
\begin{align*}
    \DTV{ \EE_{\bx\sim \mu}\brac{ \QQ_{\bx}^{\otimes n} }, \EE_{\by\sim \nu}\brac{ \QQ_{\by}^{\otimes n} } }^2 \leq \frac14 \sum_{\ell=0}^n \binom{n}{\ell} \cdot \frac{1}{d^{\ell}}\ltwo{ \bDelta_\ell }^2,
\end{align*}
where we denote 
\begin{align*}
    \bDelta_\ell\defeq \EE_{\bx\sim \mu}\brac{ \bx^{\otimes \ell} }-\EE_{\by\sim \nu}\brac{ \by^{\otimes \ell} } \in\RR^{d^{\ell}}.
\end{align*}
\end{proposition}

\begin{proof}
  We utilize the idea of the %
  {orthogonal polynomials} (see e.g. \citet{hanmixtures}) %
  to simplify our calculation.
For simplicity, we denote $\cO=[2d]$. 
By definition, for any $\bo=(o_1,\cdots,o_n)\in\cO^n$, we have
\begin{align*}
    \frac{\QQ_{\bx}^{\otimes n}(\bo)}{\QQ_{\bz}^{\otimes n}(\bo)}=\prod_{j=1}^n \frac{\QQ_{\bx}(o_j)}{\QQ_{\bz}(o_j)}=\sum_{\bk\in\NN^d} c_{n,\bk}(\bo)\bx^{\bk} ,
\end{align*}
where for $\bk=(k_1,\cdots,k_d)\in\NN^d$ we denote $\abs{\bk}=k_1+\cdots+k_d$,  $\bx^{\bk}=\bx[1]^{k_1}\cdots\bx[d]^{k_d}$, and %
$c_{n,\bk}:\cO^n\to\R$ are coefficients satisfying $c_{n,\bk}(\bo)=0$ for all $\abs{\bk}>n$. Notice that for $\bx,\by\in\RR^d$,
\begin{align*}
    \sum_{\bo\in\cO^n} \frac{\QQ_{\bx}^{\otimes n}(\bo)\QQ_{\by}^{\otimes n}(\bo)}{\QQ_{\bz}^{\otimes n}(\bo)}
    =&~ 
    \sum_{o_1,\cdots,o_n\in\cO} \prod_{j=1}^n \frac{\QQ_{\bx}(o_j)\QQ_{\by}(o_j)}{\QQ_{\bz}(o_j)} 
    =
    \paren{ \sum_{o\in\cO} \frac{\QQ_{\bx}(o)\QQ_{\by}(o)}{\QQ_{\bz}(o)} }^n
    =
    \paren{1+\frac{\iprod{x}{y}}{d}}^n.
\end{align*}
On the other hand, it also holds (where the expectation $\EE_{\bz}$ is taken over $\bo\sim\QQ_{\bz}$)
\begin{align*}
    \sum_{\bo\in\cO^n} \frac{\QQ_{\bx}^{\otimes n}(\bo)\QQ_{\by}^{\otimes n}(\bo)}{\QQ_{\bz}^{\otimes n}(\bo)}
    =&~
    \EE_{\bz}\brac{ \frac{\QQ_{\bx}^{\otimes n}(\bo)}{\QQ_{\bz}^{\otimes n}(\bo)}\cdot \frac{\QQ_{\by}^{\otimes n}(\bo)}{\QQ_{\bz}^{\otimes n}(\bo)} } \\
    =&~
    \EE_{\bz}\brac{ \sum_{\bk\in\NN^d} c_{n,\bk}(\bo)\bx^{\bk} \sum_{\bj\in\NN^d} c_{n,\bj}(\bo)\by^{\bj} } \\
    =&~
    \sum_{\bk,\bj\in\NN^d} \EE_{\bz}\brac{ c_{n,\bk}(\bo)c_{n,\bj}(\bo)}\cdot \bx^{\bk}\by^{\bj}.
\end{align*}
Therefore, by comparing the coefficients between the two sides of
\begin{align*}
    \paren{1+\frac{\iprod{x}{y}}{d}}^n
    =
    \sum_{\bk,\bj\in\NN^d} \EE_{\bz}\brac{ c_{n,\bk}(\bo)c_{n,\bj}(\bo)}\cdot \bx^{\bk}\by^{\bj},
\end{align*}
we have
\begin{align*}
    \EE_{\bz}\brac{ c_{n,\bk}(\bo)c_{n,\bj}(\bo)} = \begin{cases}
        0, & \bk\neq \bj,\\
        \binom{n}{\abs{\bk}} \frac{N_{\bk}}{d^{\abs{\bk}}}, & \bk=\bj,
    \end{cases}
\end{align*}
where for $\bk=(k_1,\cdots,k_d)$ such that $\abs{\bk}=\ell$, $N_{\bk}=\binom{\ell}{k_1, \cdots, k_d}$.
Now, we can express
\begin{align*}
    2\DTV{ \EE_{\bx\sim \mu}\brac{ \QQ_{\bx}^{\otimes n} }, \EE_{\by\sim \nu}\brac{ \QQ_{\by}^{\otimes n} } }
    =&~
    \EE_{\bz}\abs{ \EE_{\bx\sim \mu}\brac{ \frac{\QQ_{\bx}^{\otimes n}(\bo)}{\QQ_{\bz}^{\otimes n}(\bo)} }-\EE_{\by\sim \mu}\brac{ \frac{\QQ_{\by}^{\otimes n}(\bo)}{\QQ_{\bz}^{\otimes n}(\bo)} } } \\
    =&~
    \EE_{\bz}\abs{ \EE_{\bx\sim \mu}\brac{ \sum_{\bk\in\NN^d} c_{n,\bk}(\bo)\bx^{\bk} }-\EE_{\by\sim \nu}\brac{ \sum_{\bk\in\NN^d} c_{n,\bk}(\bo)\by^{\bk} } }   \\
    =&~
    \EE_{\bz}\abs{ \sum_{\bk\in\NN^d} c_{n,\bk}(\bo) \Delta_{\bk} },
\end{align*}
where in the last line we abbreviate $\Delta_{\bk}=\EE_{\bx\sim \mu}\brac{ \bx^{\bk} }-\EE_{\by\sim \nu}\brac{ \by^{\bk} }$ for $\bk\in\NN^d$.
By Jensen inequality, %
\begin{align*}
    4\DTV{ \EE_{\bx\sim \mu}\brac{ \QQ_{\bx}^{\otimes n} }, \EE_{\by\sim \nu}\brac{ \QQ_{\by}^{\otimes n} } }^2
    \leq&~
    \EE_{\bz}\abs{ \sum_{\bk\in\NN^d} c_{n,\bk}(\bo) \Delta_{\bk} }^2 \\
    =&~ \EE_{\bz}\brac{ \sum_{\bk\in\NN^d} c_{n,\bk}(\bo)\Delta_{\bk} \sum_{\bj\in\NN^d} c_{n,\bj}(\bo)\Delta_{\bj} } \\
    =&~
    \sum_{\bk,\bj\in\NN^d} \EE_{\bz}\brac{ c_{n,\bk}(\bo)c_{n,\bj}(\bo)}\cdot \Delta_{\bk}\Delta_{\bj} \\
    =&~
    \sum_{\bk\in\NN^d} \binom{n}{\abs{\bk}} \frac{N_{\bk}}{d^{\abs{\bk}}}\Delta_{\bk}^2\\
    =&~
    \sum_{\ell=0}^n \binom{n}{\ell} \frac{1}{d^{\ell}}\sum_{\bk\in\NN^d: \abs{\bk}=\ell } N_{\bk} \Delta_{\bk}^2 \\
    =&~
    \sum_{\ell=0}^n \binom{n}{\ell}\frac{1}{d^{\ell}}\ltwo{ \bDelta_\ell }^2,
\end{align*}
where the last equality follows directly from definition:
\begin{align*}
    \sum_{\bk\in\NN^d: \abs{\bk}=\ell } N_{\bk} \Delta_{\bk}^2
    =&~
    \sum_{\bk\in\NN^d: \abs{\bk}=\ell } N_{\bk} \abs{ \EE_{\bx\sim \mu}\brac{ \bx^{\bk} }-\EE_{\by\sim \nu}\brac{ \by^{\bk} } }^2 \\
    =&~
    \sum_{i_1,\cdots,i_\ell\in[d]^\ell} \abs{ \EE_{\bx\sim \mu}\brac{ \bx[i_1]\cdots\bx[i_\ell] }-\EE_{\by\sim \nu}\brac{ \by[i_1]\cdots\by[i_\ell] } }^2\\
    =&~
    \ltwo{ \EE_{\bx\sim \mu}\brac{ \bx^{\otimes \ell} }-\EE_{\by\sim \nu}\brac{ \by^{\otimes \ell} } }^2.
\end{align*}
\end{proof}

\begin{corollary}\label{cor:prob-matching}
Let $d, N, K, H \in \mathbb{N}$ and $\delta\in(0,1]$ be given so that $N\geq \binom{K+d-1}{d}+1$. Suppose $\bx_1,\cdots,\bx_N\in[-\delta,\delta]^d$. Then there exist two distributions $\xi_0,\xi_1\in\Delta([N])$, such that  $\supp(\xi_0)\cap\supp(\xi_1)=\emptyset$ and
\begin{align*}
    \DTVt{ \EE_{i\sim \xi_0}\brac{ \QQ_{\bx_i}^{\otimes H} }, \EE_{i\sim \xi_1}\brac{ \QQ_{\bx_i}^{\otimes H} } }
    \leq \sum_{k=K}^H \paren{\frac{eH\delta^2}{K}}^{k}.
\end{align*}
\end{corollary}

\begin{proof}
Consider the following system of equations:
\begin{align*}
    \sum_{i=1}^N v_i \bx_i[1]^{k_1}\cdots\bx_i[d]^{k_d}=0, \qquad \forall k_j\geq 0, k_1+\cdots+k_d\leq K-1.
\end{align*}
There are exactly $\binom{K+d-1}{d}$ equations, and hence such a system must have a non-zero solution $v^\star\in\RR^{N}$. Notice that $\sum_{i=1}^N v_i^\star=0$, and we then take $\xi_0=[v^\star]_+/V$, $\xi_1=[-v^\star]_{+}/V \in\Delta([N])$, where $V=\lone{[v^\star]_+}=\lone{[-v^\star]_+}$ is the normalizing factor. 
Clearly, $\supp(\xi_0)\cap\supp(\xi_1)=\emptyset$, and we also have
\begin{align*}
    \EE_{i\sim \xi_0} \bx_i^{\otimes \ell}=\EE_{i\sim \xi_1} \bx_i^{\otimes \ell}, \qquad \forall \ell=0,\cdots K-1.
\end{align*}
Consider $\bDelta_\ell\defeq \EE_{i\sim \xi_0} \bx_i^{\otimes \ell}-\EE_{i\sim \xi_1} \bx_i^{\otimes \ell}$; then we have $\bDelta_\ell=0$ for $\ell<K$, and we also have
\begin{align*}
    \ltwo{\bDelta_\ell}\leq 2\max_{i}\ltwo{\bx_i^{\otimes \ell}}
    \leq 2\ltwo{\bx_i}^{\ell}
    \leq 2(\sqrt{d}\delta)^{\ell}, \qquad \forall \ell\geq 0.
\end{align*}
This implies that $\frac{1}{d^{\ell}}\ltwo{\bDelta_\ell}^2\leq 4\delta^{2\ell}$ %
always holds.
Therefore, applying \cref{prop:unif-moments} with $n=H$ and using the fact that $\binom{H}{k}\leq \paren{\frac{eH}{k}}^k$, %
we obtain %
\begin{align*}
    \DTVt{ \EE_{i\sim \xi_0}\brac{ \QQ_{\bx_i}^{\otimes H} }, \EE_{i\sim \xi_1}\brac{ \QQ_{\bx_i}^{\otimes H} } }
    \leq \sum_{k=K}^H \paren{\frac{eH}{k}}^k \cdot (\delta)^{2k} 
    \leq \sum_{k=K}^H \paren{\frac{eH\delta^2}{K}}^{k}.
\end{align*}
\end{proof}

\newcommand{\delinf}{\delta_{\infty}}
\newcommand{\vol}{\mathrm{vol}}

\paragraph{Proof of \cref{prop:2-family}}
Choose $\delinf>0$, $d\geq 1$, and an integer $K\leq \paren{\frac{\delinf}{2e^2\delta}-1}d+1$ (to be specified later in the proof).
For the $\ell_{\infty}$-ball $\BB\defeq [-\delinf,\delinf]^d$, we consider its packing number under the $\ell_1$-norm, denoted $M(\cdot; \BB, \lone{\cdot})$. Using \citet[Lemma 5.5 \& 5.7]{wainwright2019high}, we have
\begin{align*}
    M(\delta_1; \BB, \lone{\cdot})\geq \paren{\frac{1}{\delta_1}}^d\frac{\vol(\BB)}{\vol(\BB')}, \qquad \forall \delta_1>0,
\end{align*}
where $\BB'=\set{x\in\R^d: \lone{x}\leq 1}$ is the $\ell_1$ unit ball. Notice that $\vol(\BB)=(2\delinf)^d$, $\vol(\BB')=\frac{2^d}{d!}$. Thus, using the fact $d!>(d/e)^d$, we have
\begin{align*}
    M(\delta_1; \BB, \lone{\cdot})\geq d!\paren{\frac{\delinf}{\delta_1}}^d
    > \paren{\frac{d\delinf}{e\delta_1}}^d
\end{align*}
In particular, $M\defeq M(2d\delta; \BB, \lone{\cdot})> \paren{\frac{\delinf}{2e\delta}}^d$.
Notice that our choice of $K$ ensures that for $N=\binom{K+d-1}{d}+1$, it holds that $N\leq M$. Therefore, we can pick $N$ vectors $\bx_1,\cdots,\bx_N\in\BB$ such that $\lone{\bx_i-\bx_j}\geq 2d\delta$. 

Consider the distributions $\mu_i=\QQ_{\bx_i}\in\Delta([2d])$ for each $i\in[N]$. Clearly, we have $\DTV{\mu_i,\mu_j}\geq \delta$ for $i\neq j$.
Also, by \cref{cor:prob-matching}, %
there exists $\xi_0,\xi_1\in\Delta([N])$ such that $\supp(\xi_0)\cap\supp(\xi_1)=\emptyset$,
\begin{align*}
    \DTVt{ \EE_{i\sim \xi_0}\brac{ \mu_i^{\otimes H} }, \EE_{i\sim \xi_0}\brac{ \mu_i^{\otimes H} } } 
    \leq \sum_{k=K}^H \paren{\frac{eH\delinf^2}{K}}^{k}.
\end{align*}

Consider $\cQ=\set{(\mu_1,\cdots,\mu_N),(\xi_0,\xi_1)}$.

\textbf{Proof of \cref{prop:2-family} (a).} %
In this case, we pick $\delinf=1$, $K=H+1$, $d=\ceil{4e^2\delta H}$. Then $\cQ$ is a $(2,H,\delta,0,N)$-family over $[2d]$, with $N\leq \min\paren{\frac{1}{2e\delta},2H}^d$.

\textbf{Proof of \cref{prop:2-family} (b).} %
In this case, we take $K=\ceil{\lambda d}$, $\delinf=2e^2\delta(\lambda+1)$, so $\frac{eH\delinf^2}{K}\leq e^{-2}$ and hence $\cQ$ is a $(2,H,\delta,\gamma,N)$-family over $[2d]$ with $\gamma\leq 2e^{-\lambda d}$ and $N\leq (2e(\lambda+1))^d$. 
\qed

\subsection{Proof of Lemma~\ref{lem:family-tensor}}\label{appdx:proof-family-tensor}

\newcommand{\txi}{\Tilde{\xi}}

Suppose that $\cQ=\set{(\mu_1,\cdots,\mu_N),(\xi_0,\xi_1)}$ is a $(2,H,\delta,\gamma,N)$-family over $\cO$. Then, for each integer $m\in\set{0,1,\cdots,2^r-1}$, we consider its binary representation $m=(m_r\cdots m_1)_2$, and define
\begin{align*}
    \txi_m=\xi_{m_r}\otimes \cdots\otimes \xi_{m_1}\in [N]^{r}.
\end{align*}
Further, for each $\bk=(k_1,\cdots,k_r)\in[N]^r$, we define
\begin{align*}
    \tmu_{\bk}=\mu_{k_1}\otimes \cdots\otimes\mu_{k_r} \in \cO^{r}.
\end{align*}
Under the definitions above, we know
\begin{align*}
    \EE_{\bk\sim \txi_m}\brac{ \tmu_{\bk}^{\otimes H} } = \EE_{k_1\sim \xi_{m_1}}\brac{ \mu_{k_1}^{\otimes H} } \otimes \cdots\otimes \EE_{k_r\sim \xi_{m_r}}\brac{ \mu_{k_r}^{\otimes H} },
\end{align*}
and hence for $0\leq m,l\leq 2^r-1$, it holds that
\begin{align*}
    \DTV{ \EE_{\bk\sim \txi_m}\brac{ \tmu_{\bk}^{\otimes H} }, \EE_{\bk\sim \txi_l}\brac{ \tmu_{\bk}^{\otimes H} } }
    \leq \sum_{i=1}^r \DTV{ \EE_{k\sim \xi_{m_i}}\brac{ \mu_{k}^{\otimes H} }, \EE_{k\sim \xi_{l_i}}\brac{ \mu_{k}^{\otimes H} } }
    \leq r\gamma.
\end{align*}
We also know that $\supp(\txi_m)\cap \supp(\txi_l)=\emptyset$ as long as $m\neq l$. For $\bk,\bj\in\cup_{m} \supp(\txi_m)$ such that $\bk\neq\bj$, it also holds that
\begin{align*}
    \DTV{ \tmu_{\bk}, \tmu_{\bj} }
    \geq \max_{1\leq i\leq r} \DTV{\mu_{k_i}, \mu_{j_i}} \geq \delta.
\end{align*}
Therefore, $\bQ'=\set{ (\tmu_{\bk})_{\bk\in[N]^r}, (\txi_{0},\cdots,\txi_{2^r-1}) }$ is indeed a $(2^r,H,\delta,r\gamma,N^r)$-family over $\cO^r$.
\qed

\subsection{Proof of Theorem~\ref{thm:decode-lower}}\label{appdx:comb-lock-decode}

In this section, we modify the constructions in \cref{appdx:comb-lock} to obtain a class of hard instances of $N$-step decodable LMDPs
\begin{align}
  \label{eq:define-mtheta-lb-decode}
    \cM^+=\set{ M_{\theta}^+: \theta\in\cA^{n-1} } \cup\set{ M_{\emptyset}^+ },
\end{align}
and then sketch the proof of \cref{thm:decode-lower} (as most parts of the proof follow immediately from \cref{appdx:comb-lock} and \cref{prop:tensor-comb-lock}).

\newcommand{\sms}[1]{s_{\ominus,#1}}
For any given integer $N, n, A$, we set $k=N-n$ so that $H=n+2k$, and we take $\cA=[A]$.
We specify the state space, action space and reward function (which are shared across all LMDP instances) as follows.
\begin{itemize}
\item The state space is
\begin{align*}
    \cS=\set{\sg{i}: -k+1\leq i\leq n+k}\bigsqcup \set{ \sms{i}: 2\leq i\leq n+k }\bigsqcup \set{\stermin_1,\cdots,\stermin_n}.
\end{align*}
\item The action space is $\cA$.
\item The reward function is given by $R_h(s,a)=\indic{s=\sg{n}, h=n+k+1}$.
\end{itemize}
We remark that, our below construction has (essentially) the same LMDP dynamics at the state $s\in\cS_+:=\set{\sg{1},\cdots,\sg{n}}$, as the construction in \cref{appdx:comb-lock}. The auxiliary states $\sms{2},\cdots,\sms{n+k}, \stermin_1,\cdots,\stermin_n$ are introduced so that we can ensure $N$-step decodability, while the auxiliary states $\sg{-k+1},\cdots,\sg{0}$ are introduced to so that we can take the horizon $H$ to equal $N+k$.

\paragraph{Construction of the LMDP $M_\theta^+$}
For any $\theta=\ba\in\cA^{n-1}$, we construct a LMDP $M_{\theta}^+$ as follows.
\begin{itemize}
\item $L=n$, the MDP instances of $M_\theta^+$ is given by $M_{\theta,1}^+,\cdots,M_{\theta,n}^+$ with mixing weight $\rho=\Unif([n])$.
\item For each $m\in[n]$, in the MDP $M_{\theta,m}^+$, the initial state is $\sg{-k+1}$, and the transition dynamics at state $s\not\in\cS_{+}=\set{\sg{1},\cdots,\sg{n}}$ is specified as follows and does not depend on $\theta$:
\begin{itemize}
    \item At state $\sg{h}$ with $h\leq 0$, taking any action leads to $\sg{h+1}$.
    \item At state $\sms{h}$ with $h<n+k$, taking any action leads to $\sms{h+1}$.
    \item At state $s\in\set{\sms{n+k},\stermin_1,\cdots,\stermin_n}$, taking any action leads to $\stermin_m$.
\end{itemize}

For $m>1$, the transition dynamics of $M_{\theta,m}^+$ at state $s\in\cS_+$ is given as follows (similar to \cref{appdx:comb-lock}).
\begin{itemize}
    \item At state $\sg{h}$ with $h<m$, taking any action leads to $\sg{h+1}$.
    \item At state $\sg{m-1}$, taking action $a\neq \ba_{m-1}$ leads to $\sg{m}$, 
    and taking action $\ba_{m-1}$ leads to $\sms{m}$.
    \item At state $\sg{h}$ with $m\leq h<n$, taking action $a\neq \ba_{h}$ leads to $\sm$, and taking action $\ba_{h}$ leads to $\sg{h+1}$.
    \item At state $\sg{n}$, taking any action leads to $\sms{n+1}$.
\end{itemize}
The transition dynamics of $M_{\theta,1}^+$ at state $s\in\cS_+$ is given as follows.
\begin{itemize}
    \item At state $\sg{h}$ with $h<n$, taking action $a\neq \ba_{h}$ leads to $\sm$, and taking action $\ba_{h}$ leads to $\sg{h+1}$.
    \item The state $\sg{n}$ is an absorbing state.
\end{itemize}
\end{itemize}

\paragraph{Construction of the reference LMDP} For $\otheta=\emptyset$, we construct the LMDP $M_{\otheta}$ with state space $\cS$, MDP instances $M_{\otheta,1},\cdots,M_{\otheta,n}$, mixing weights $\rho=\unif([n])$, where for each $m\in[n]$, the transition dynamics of $M_{\otheta,m}$ is specified as follows: (1) the initial state is always $\sg{-k+1}$, (2) the transition dynamics at state $s\not\in\cS_+$ agrees with the transition dynamics of $M_{\theta,m}$ described as above, (3) at state $\sg{h}$ with $h< m$, taking any action leads to $\sg{h+1}$, and (4)  at state $\sg{h}$ with $h\geq m$, taking any action leads to $\sms{h+1}$.

\paragraph{Sketch of proof} The following are several key observations for the LMDP $M_\theta$ ($\theta\in\cA^{n-1}\sqcup\set{\otheta}$).

(1) At state $s\in\cS_+$, the transition dynamics of $M_{\theta,m}^+$ agrees with the transition dynamics of $M_{\theta,m}$ (defined in \cref{appdx:comb-lock}), in the sense that we identify the state $\sm$ there as the set of $\set{\sms{2},\cdots,\sms{n+k}}$. %

(2) With horizon $H=n+2k$, we always have $s_H\in\set{\sg{n},\sms{n+k}}$, and all the states in $\set{\stermin_1,\cdots,\stermin_n}$ are not reachable. In other words, the auxiliary states $\stermin_1, \cdots, \stermin_n$ (introduced for ensuring $N$-step decodability) do not reveal information of the latent index because they are never reached.

(3) $M_\theta$ is $N$-step decodable, because:

(3a) $M_\theta$ is $N$-step decodable when we start at $s\in\set{ \sms{2},\cdots,\sms{n+k}, \stermin_1,\cdots,\stermin_n}$. This follows immediately from definition, because in $M_\theta$, any reachable trajectory $\otau_N$ starting at such state $s$ must end with $s_N=\stermin_m$, where $m$ is the index of the MDP instance $M_{\theta,m}$. Similar argument also shows that $M_\theta$ is $N$-step decodable when we start at $s\in\set{ \sg{2}, \cdots, \sg{n}}$.

(3b) $M_\theta$ is $n$-step decodable when we start at $\sg{1}$. This follows immediately from our proof of \cref{lem:comb-lock} (a), which shows that for any reachable trajectory $\otau_n$, there is a unique latent index $m$ such that $\otau_n$ is reachable under $M_{\theta,n}$. Therefore, we also know that $M_\theta$ is $N$-step decodable when we start at $s\in\set{\sg{-k+1},\cdots,\sg{0}}$.

Given the above observations, we also know that our argument in the proof of \cref{prop:tensor-comb-lock} indeed applies to $\cM^+$, which concludes that the class $\cM^+$ of $N$-step decodable LMDPs requires $\Om{A^{n-1}}$ samples to learn.
\qed

\section{Proofs for Section~\ref{sec:stat-upper}}\label{appdx:stat-upper}

\newcommand{\clogK}{\iota_K}
\newcommand{\clogKv}{\log(LdH\cdot K/(A\beta))}

\paragraph{Miscellaneous notations} 
We identify $\Piall=\Delta(\Piall)$ as both the set of all policies and all distributions over policies interchangeably. 

Also, recall that for any step $h$, we write $\tau_h=(s_1,a_1,\cdots,s_h,a_h)$, and $\tau_{h:h'}=(s_h,a_h,\cdots,s_{h'},a_{h'})$ compactly. Also recall that
\begin{align*}
    \PP_\theta(\tau_h)=\PP_\theta(s_{1:h}|\doac(a_{1:h-1})),
\end{align*}
i.e., $\PP_\theta(\tau_h)$ is the probability of observing $s_{1:h}$ if the agent deterministically executes actions $a_{1:h-1}$ in the LMDP $M_\theta$. Also denote $\pi(\tau_h)\defeq \prod_{h'\le h} \pi_{h'}(a_{h'}|\tau_{h'-1}, s_{h'})$, and then $\PP^{\pi}_\theta(\tau_h)=\PP_\theta(\tau_h)\times \pi(\tau_h)$ %
gives the probability of observing $\tau_h$ for the first $h$ steps when executing $\pi$ in LMDP $M_\theta$. %

For any policy $\pi,\pi'\in\Pi$ and step $h\in[H]$, we define $\pi\circ_h \pi'$ to be the policy that executes $\pi$ for the first $h-1$ steps, and then starts executing $\piexp$ at step $h$ (i.e. discarding the history $\tau_{h-1}$). %

To avoid confusion, we define $\PP_\theta(\tau_{h:H}|\tau_{h-1},\pi)$ to be the probability of observing $\tau_{h:H}$ conditional on the history $\tau_{h-1}$ if we start executing $\pi$ at the step $h$ (i.e. $\pi$ does not use the history data $\tau_{h-1}$). By contrast, consistently with the standard notation of conditional probability, $\PP_\theta^\pi(\tau_{h:H}|\tau_{h-1})$ is the conditional probability of the model $\PP_\theta^\pi$, i.e. the probability of observing $\tau_{h:H}$ conditional on the history $\tau_{h-1}$ under policy $\pi$. Therefore, we have
\begin{align}\label{eqn:cond-pi}
    \PP_\theta^\pi(\tau_{h:H}|\tau_{h-1})=\PP_\theta(\tau_{h:H}|\tau_{h-1},\pi(\cdot|\tau_{h-1})).
\end{align}

\subsection{Details of Algorithm OMLE}\label{appdx:OMLE-basic}

Given a \emph{separating policy} $\piexp$, we can construct a corresponding map $\modf{\cdot}:\Piall\to\Piall$, that transforms any policy $\pi$ to an explorative version of it. The definition of $\modf{\cdot}$ below is similar to the choice of the explorative policies for learning PSRs in \citet{zhan2022pac,chen2022partially,liu2022optimistic}.
\begin{definition}\label{def:policy-mod}
Suppose that $\piexp\in\Piall$ is a given policy and $1\leq W\leq H$. For any step $1\leq h\leq H$, we define $\varphi_h:\Piall\to\Piall$ to be a policy modification given by
$$
\varphi_h (\pi) = \pi \circ_{h}\unif(\cA)\circ_{h +1}\piexp, \qquad \pi\in\Piall,
$$
i.e. $\varphi_h (\pi)$ means that we follow $\pi$ for the first $h-1$ steps, take $\unif(\cA)$ at step $h$, and start executing $\piexp$ afterwards. 

Further, we define $\modp{\cdot}, \modf{\cdot}$ as follows:
\begin{align*}
    \modp{\pi} = \pi \circ_{W} \piexp, \qquad
    \modf{\pi} = \frac{1}{2}\modp{\pi}+\frac{1}{2H}\sum_{h=0}^{H-1} \varphi_h(\pi).
\end{align*}
\end{definition}

The following guarantee pertaining to the confidence set maintained in OMLE is taken from \citet[Proposition E.2]{chen2022partially}. There is a slight difference in the policy modification applied to $\pi^t$, which does not affect the argument in \citet[Appendix E.1]{chen2022partially}. 

\begin{proposition}[Confidence set guarantee]\label{thm:MLE}
Suppose that we choose $\beta\geq \betaN$ in \cref{alg:OMLE}. Then with probability at least $1-p$, the following holds:
\begin{enumerate}[wide, label=(\alph*)]
\item For all $k\in[K]$, $\ths\in\Theta^k$;
\item For all $k\in[K]$ and any $\theta\in\Theta^k$, it holds that
\begin{align}\label{eqn:OMLE-est-err}
    \sum_{t=1}^{k-1} \DH{ \PP^{\modf{\pi^t}}_{\theta}, \PP^{\modf{\pi^t}}_{\ths} } \leq 2\beta.
\end{align}
\end{enumerate}
\end{proposition}

\newcommand{\MLEevent}{$E_0$\xspace}

Let \MLEevent~be the event that both (a) and (b) of \cref{thm:MLE} above hold true. In the following, we will analyze the performance of \cref{alg:OMLE} conditional on the suceess event \MLEevent.

The following proposition relates the sub-optimality of the output policy $\hpi$ of \cref{alg:OMLE} to the error of estimation.
\begin{proposition}
\label{prop:optimism}
Suppose that \cref{def:single-policy-sep} holds, and $W\geq \om^{-1}(\log(L/\epssep))$. Conditional on the success event \MLEevent, we have
\begin{align*}
    V_\star-V_{\ths}(\hat\pi)\leq \frac{1}{K}\sum_{k=1}^K \DTV{ \PP^{\pi^k}_{\theta^k}, \PP^{\pi^k}_{\ths} }.
\end{align*}
\end{proposition}
\begin{proof}
Under the given condition on $W$, it holds $\perr_{\ths,W}(\pis)\leq \epssep$ (\cref{prop:latent-MLE}).
By \cref{thm:MLE} (a), we also have $\ths\in\Theta^k$ for each $k\in[K]$. Therefore, by the choice of $(\theta^k,\pi^k)$ in \cref{alg:OMLE}, it holds that $V_\star=V_{\ths}(\pis)\leq V_{\theta^k}(\pi^k)$. Hence,
\begin{align*}
    V_\star-V_{\ths}(\pi^k)
    \leq V_{\theta^k}(\pi^k)-V_{\ths}(\pi^k)
    \leq \DTV{ \PP^{\pi^k}_{\theta^k}, \PP^{\pi^k}_{\ths} },
\end{align*}
where the last inequality follows from the definition of TV distance and the fact that $\sum_{h=1}^H R_h(s_h,a_h)\in[0,1]$ for any trajectory. Taking average over $k\in[K]$ completes the proof.
\end{proof}

\subsection{Proof overview}\label{appdx:upper-overview}

Given \cref{thm:MLE} and \cref{prop:optimism}, upper bounding the sub-optimality of the output $\hpi$ reduces to the following task.
\begin{align*}
    \text{Task: upper bound }\sum_{k=1}^K \DTV{ \PP^{\pi^k}_{\theta^k}, \PP^{\pi^k}_{\ths} }, 
    ~~\text{given that }\forall k\in[K], 
    &~\sum_{t=1}^{k-1} \DH{ \PP^{\modf{\pi^t}}_{\theta^k}, \PP^{\modf{\pi^t}}_{\ths} } \leq 2\beta.
\end{align*}

A typical strategy, used in \citet{liu2022partially,chen2022unified,chen2022partially,liu2023optimistic}, of relating these two terms is three-fold: (1) find a decomposition of the TV distance, i.e. an upper bound of $\DTV{ \PP^{\pi}_{\theta}, \PP^{\pi}_{\ths} }$; (2) show that the decomposition can be upper bounded by the squared Hellinger distance $\DH{ \PP^{\pi}_{\theta}, \PP^{\pi}_{\ths} }$; (3) apply an eluder argument on the decomposition to complete the proof. 

For example, we describe this strategy for the special case of MDPs. %
\begin{example}
Suppose that $\Theta$ is instead a class of MDPs and $\modf{\pi}=\pi$, then we can decompose
\begin{align}\label{eqn:example-decomp-MDP}
    \DTV{ \PP^{\pi}_{\theta}, \PP^{\pi}_{\ths} }
    \leq&~
    \underbrace{ \sum_{h=1}^{H-1} \EE_{\ths}^{\pi} \DTV{ \TT_\theta(\cdot|s_h,a_h),\TT_{\ths}(\cdot|s_h,a_h) } }_{=:G_{\ths}(\pi,\theta)}
    \leq 2H\DTV{ \PP^{\pi}_{\theta}, \PP^{\pi}_{\ths} }.
\end{align}
In tabular case, the decomposition $G_{\ths}(\cdot,\cdot)$ can be written as an inner product over $\R^{\cS\times\cA}$, i.e. $G_{\ths}(\pi,\theta)=\iprod{X(\theta)}{W(\pi)}$ for appropriate embeddings $X(\theta), W(\pi) \in \sR^{\cS \times \cA}$. Then, using the eluder argument for linear functionals (i.e. the ``elliptical potential lemma'', \citet{lattimore2020bandit}), we can prove that under \cref{eqn:OMLE-est-err}, it holds that $\sum_k \DTV{ \PP^{\pi^k}_{\theta^k}, \PP^{\pi^k}_{\ths} }\leq \tO(\sqrt{SA\cdot KH^2\beta})$.  

More generally, beyond the tabular case, we can also apply a coverability argument (see e.g. \citet{xie2022role} and also \Cref{prop:coverage-eluder}) as follows. Suppose that $\rank(\TT_{\ths})\leq d$. We can then invoke \cref{prop:rank-to-cov} to show that $G_{\ths}$ admits the following representation:
\begin{align*}
    G_{\ths}(\pi,\theta)=\EE_{x\sim p(\pi)} f_\theta(x),
\end{align*}
where $p:\Pi\to\Delta(\cS\times\cA)$ is such that there exists $\mu\in\Delta(\cS\times\cA)$, $\linf{p(\pi)/\mu}\leq d\cdot A$ for all $\pi$. Hence, \cref{prop:coverage-eluder} implies that $\sum_k \DTV{ \PP^{\pi^k}_{\theta^k}, \PP^{\pi^k}_{\ths} }\leq \tO(\sqrt{dA\cdot KH^2\beta})$. \exend
\end{example}

\paragraph{Analyzing the separated LMDPs}
In our analysis, we first decompose the TV distance between LMDPs into two parts: 
\begin{align}\label{eqn:decomp-demo}
\begin{aligned}
    \DTV{ \PP^{\pi}_\theta, \PP^{\pi}_{\ths} }
    \leq&~ \DTV{ \PP^{\pi}_\theta(\otau_W=\cdot), \PP^{\pi}_{\ths}(\otau_W=\cdot) } \\
    &~+
    \EE_{\ths}^\pi\brac{ \DTV{ \PP_\theta^{\pi}\paren{ \otau_{W:H}=\cdot  | \otau_W } , \PP_{\ths}^{\pi}\paren{ \otau_{W:H}=\cdot  | \otau_W } } }
\end{aligned}
\end{align}
where the part (a) is the TV distance between the distribution of trajectory up to step $W$, and part (b) is the TV distance between the conditional distribution of the last $H-W+1$ steps trajectory. We analyze part (a) and part (b) separately.

\paragraph{Part (a)} Under the assumption of $\om$-separation under $\piexp$ and $H-W\geq \om^{-1}(\log(2L))$, we can show that a variant of the revealing condition \citep{liu2022partially,chen2022partially,liu2023optimistic} holds (\cref{lem:DB-inv}). Therefore, restricting to dynamics of the first $W$ steps, we can regard $\Theta$ as a class of revealing POMDPs, and then apply the eluder argument developed in \citet{chen2022partially}. More specifically, our analysis of part (a) relies on the following result, which is almost an immediately corollary of the analysis in \citet[Appendix D \& E]{chen2022partially}. 
\begin{theorem}\label{thm:psr}
Suppose that for all $\theta\in\Theta$, $\theta$ is $\om$-separated under $\piexp$, and $H-W\geq \om^{-1}(\log (2L))$. Then conditional on the success event \MLEevent,
\begin{align*}
    \sum_{k=1}^K \DTV{ \PP^{\modp{\pi^k}}_{\theta^k}, \PP^{\modp{\pi^k}}_{\ths} }
    \leqsim \sqrt{LdAH^2\clogK \cdot K\beta },
\end{align*}
where $\clogK=\clogKv$ is a logarithmic factor.
\end{theorem}
We provide a more detailed discussion of \cref{thm:psr} and a simplified proof in \cref{appdx:psr}. Notice that, although the statement of \cref{thm:psr} bounds the total variation distance between the entire ($H$-step) trajectories $\PP^{\modp{\pi^k}}_{\theta^k}$ and $\PP^{\modp{\pi^k}}_{\ths}$, the policies $\modp{\pi^k}$ act according to the \emph{fixed} policy $\piexp$ on steps $h \geq W$. Thus, \cref{thm:psr} is not establishing that the model $\ths$ is being learned in any meaningful way after step $W$ (indeed, it cannot since we may not have $H-h \geq \om^{-1}(\log(2L))$ for $h > W$). To learn the true model $\ths$ at steps $h \geq W$, we need to analyze part (b) of \cref{eqn:decomp-demo}. 

\paragraph{Part (b)} The main idea for analyzing the steps $h\geq W$ is that, given $\perr_\theta(\pi)$ is small, we can regard
\begin{align}\label{eqn:proof-idea-approx}
    \PP_\theta^{\pi}\paren{ \otau_{W:H}=\cdot  | \otau_W } \approx \MM^\theta_{\mthtau,H-W+1}(\pi(\cdot|\tau_{W-1}),s_W).
\end{align}
In other words, conditional on the first $W$ steps, the dynamics of the trajectory $\otau_{W:H}$ is close to the dynamics of the MDP $M_{\theta,\mthtau}$. Therefore, we can decompose part (b) in a fashion similar to the decomposition \cref{eqn:example-decomp-MDP} for MDP (\cref{prop:err-decomp}), and then apply the eluder argument of \cref{prop:coverage-eluder} (see \cref{cor:OMLE-almost-done}).

\subsection{Structural properties of separated LMDP}

In this section, we formalize the idea described in the part (b) of our proof overview.

For each $h\in[H]$ and trajectory $\otau_h$, we define the belief state of the trajectory $\otau_h$ under model $\theta$ as
\begin{align}\label{eqn:def-belief}
    \belief_\theta(\otau_h)=\brac{ \tPP_\theta(m|\otau_h) }_{m\in[L]} \in\Delta([L]).
\end{align}
Recall the definition of $\MM_{m,h}(\cdot) \in \Delta((\cA \times \cS)^{h-1})$ in \cref{eq:mm-def}. Then, conditional on the trajectory $\otau_W$, the distribution of $\otau_{W:H}=(a_W,\cdots,a_{H-1},s_{H})$ under policy $\pi$ can be written as
\begin{align}\label{eqn:cond-prob-to-belief}
\begin{aligned}
    \PP_\theta^{\pi}\paren{ \otau_{W:H}=\cdot  | \otau_W } 
    =&~
    \EE_{m\sim \belief_\theta(\otau_W)}\brac{ \TT_{\theta,m}^{\pi}\paren{ \otau_{W:H}=\cdot  | \otau_W } } \\
    =&~
    \EE_{m\sim \belief_\theta(\otau_W)}\brac{ \MM^{\theta}_{m,H-W+1}(\pitau{W-1},s_W) }
\end{aligned}
\end{align}
where $\pitau{W-1}=\pi(\cdot|\tau_{W-1})$ is the policy obtained from $\pi$ by conditional on $\otau_W$. %
In particular,
\begin{align}
    \DTV{ \PP_\theta^{\pi}\paren{ \otau_{W:H}=\cdot  | \otau_W } , \MM^\theta_{\mthtau,H-W+1}(\pitau{W-1},s_W) }
    \leq
    \sum_{m\neq \mthtau} \belief_\theta(\otau_W)[m].\label{eq:tvd-theta-tau-ub}
\end{align}
We denote
\begin{align}\label{eqn:def-perr-tau}
    \perrtau{\theta}
    \defeq
    \sum_{m\neq \mthtau} \belief_\theta(\otau_W)[m].
\end{align}
Notice that by the definition of $\belief_\theta(\otau_W)$, 
\begin{align}
    \perrtau{\theta}
    =
    \sum_{m\neq \mthtau} \belief_\theta(\otau_W)[m]
    =
    1-\max_m\belief_\theta(\otau_W)[m]
    =
    \tPP_\theta\paren{ m\neq \mthtau | \otau_W},\label{eq:etheta-tauw}
\end{align}
and hence $\perr_{\theta,W}(\pi)=\EE^\pi_\theta[\perrtau{\theta}]$.

In the following, we denote $\oW\defeq H-W+1$, and we will use the inequality
\begin{align}\label{eqn:cond-to-mdp-tv}
    \DTV{ \PP_\theta^{\pi}\paren{ \otau_{W:H}=\cdot  | \otau_W } , \MM^\theta_{\mthtau,\oW}(\pitau{W-1},s_W) }
    \leq
    \perrtau{\theta},
\end{align}
(which follows from \cref{eq:etheta-tauw,eq:tvd-theta-tau-ub}) and the fact that $\perr_{\theta,W}(\pi)=\EE^\pi_\theta[\perrtau{\theta}]$ repeatedly. This formalizes the idea of \cref{eqn:proof-idea-approx}.
Also notice that $\modp{\pi}=\pi\circ_W \piexp$, and hence we also have
\begin{align}\label{eqn:cond-to-mdp-tv-exp}
    \DTV{ \PP_\theta^{\modp{\pi}}\paren{ \otau_{W:H}=\cdot  | \otau_W } , \MM^\theta_{\mthtau,\oW}(\piexp,s_W) }
    \leq
    \perrtau{\theta}.
\end{align}

The following proposition shows that, as long as the model $\theta$ is close to $\otheta$, there is a correspondence between the maps $m_\theta$ and $m_\otheta$. %
\begin{proposition}\label{prop:map-err}
Suppose that $\theta$ and $\otheta$ are $\om$-separated under $\piexp$ and $\oW=H-W+1\geq \om^{-1}(1)$. Then there exists a map $\sigma=\sigma_{\theta;\otheta}:[L]\times\cS\to[L]$ such that for any $(W-1)$-step policy $\pi$, %
\begin{align}
    \PP_{\otheta}^{\pi}\paren{ \mthtau\neq \sigma(\mothtau,s_W) } 
    \leq&~ 288\DH{ \PP^{\modp{\pi}}_\theta, \PP^{\modp{\pi}}_{\otheta} } + 144\perr_{\theta,W}(\pi)+144\perr_{\otheta,W}(\pi), \label{eqn:map-err-hell}
\end{align}
where $\modp{\pi}=\pi\circ_W \piexp$ is defined in \cref{def:policy-mod}. %
\end{proposition}

\begin{proof}
In the following proof, we abbreviate $\eps=\DH{ \PP^{\modp{\pi}}_\theta, \PP^{\modp{\pi}}_{\otheta} }$.
By \cref{lemma:Hellinger-cond}, %
\begin{align}\label{eqn:map-proof-1}
    \EE_{\otheta}^\pi\brac{ \DTVt{ \PP_\theta^{\modp{\pi}}\paren{ \otau_{W:H}=\cdot  | \otau_W } , \PP_{\otheta}^{\modp{\pi}}\paren{ \otau_{W:H}=\cdot  | \otau_W } } } \leq 4\eps.
\end{align}
Using \cref{eqn:cond-to-mdp-tv-exp} and the triangle inequality of TV distance, we have
\begin{align*}
    &~ \DTV{ \MMth{ \piexp } , \MMoth{ \piexp } } \\
    \leq&~ \DTV{ \PP_\theta^{\modp{\pi}}\paren{ \otau_{W:H}=\cdot  | \otau_W } , \PP_{\otheta}^{\modp{\pi}}\paren{ \otau_{W:H}=\cdot  | \otau_W } }
    +\perrtau{ \theta } + \perrtau{ \otheta },
\end{align*}
and hence
\begin{align}\label{eqn:map-proof-2}
\begin{aligned}
    &~ \EE_{\otheta}^\pi\brac{ \DTVt{ \MMth{ \piexp } , \MMoth{ \piexp } }  } \\
    \leq&~
    3\EE_{\otheta}^\pi\brac{ \DTVt{ \PP_\theta^{\modp{\pi}}\paren{ \otau_{W:H}=\cdot  | \otau_W } , \PP_{\otheta}^{\modp{\pi}}\paren{ \otau_{W:H}=\cdot  | \otau_W } } }
    +3\EE_{\otheta}^\pi\brac{ \perrtau{ \theta }} 
    +3\EE_{\otheta}^\pi\brac{ \perrtau{ \otheta } }.
\end{aligned}
\end{align}
By definition, we know $\EE_{\otheta}^{\pi}\brac{ \perrtau{ \otheta } } = \perr_{\otheta,W}(\pi)$, and by \cref{lemma:multiplicative-hellinger}, we also have
\begin{align}\label{eqn:map-proof-3}
\begin{aligned}
    \EE_{\otheta}^{\pi}\brac{ \perrtau{ \theta } } 
    \leq&~
    3\EE_{\theta}^{\pi}\brac{ \perrtau{ \theta } } + 2\DH{ \PP^{\pi}_\theta(\otau_W=\cdot), \PP^{\pi}_{\otheta}(\otau_W=\cdot) } \\
    =&~
    3\perr_{\theta,W}(\pi)+2\DH{ \PP^{\pi}_\theta(\otau_W=\cdot), \PP^{\pi}_{\otheta}(\otau_W=\cdot) }.
\end{aligned}
\end{align}
Plugging the inequalities \cref{eqn:map-proof-1} and \cref{eqn:map-proof-3} into \cref{eqn:map-proof-2}, we have
\begin{align*}
    \EE_{\otheta}^\pi\brac{ \DTVt{ \MMth{ \piexp } , \MMoth{ \piexp } } } \leq 18\eps+9\perr_{\theta,W}(\pi)+9\perr_{\otheta,W}(\pi)
    =:\eps'.
\end{align*}
In other words, it holds that %
\begin{align}\label{eqn:map-proof-4}
    \sum_{l,\ol,s} \PP_{\otheta}^{\pi}\paren{ s_W=s, \mthtau=l, \mothtau=\ol } \cdot \DTVt{ \MMop{\theta}{l}{s}{\oW}{\piexp}, \MMop{\otheta}{\ol}{s}{\oW}{\piexp} }\leq \eps'.
\end{align}

Notice that $\oW\geq \om^{-1}(1)$. Thus, using \cref{eqn:TV-DB}, for any $m,l\in\supp(\rho_\theta)$ such that $m\neq l$, we have
\begin{align*}
    \DTV{ \MMop{\theta}{l}{s}{\oW}{\piexp}, \MMop{\theta}{m}{s}{\oW}{\piexp} } \geq \frac{1}{2}.
\end{align*}
Hence, we choose $\sigma=\sigma_{\theta;\otheta}$ as
\begin{align}\label{eqn:def-sigma}
    \sigma_{\theta;\otheta}(\ol,s) \in \argmin_{l\in\supp(\rho_\theta)} \DTV{ \MMop{\theta}{l}{s}{\oW}{\piexp}, \MMop{\otheta}{\ol}{s}{\oW}{\piexp} }.
\end{align}
Then for any $l\in\supp(\rho_\theta)$ such that $l\neq \sigma(\ol,s)$, it holds that
\begin{align*}
    &~ 2\DTV{ \MMop{\theta}{l}{s}{\oW}{\piexp}, \MMop{\otheta}{\ol}{s}{\oW}{\piexp} } \\
    \geq&~ 
    \DTV{ \MMop{\theta}{l}{s}{\oW}{\piexp}, \MMop{\otheta}{\ol}{s}{\oW}{\piexp} }
    +\DTV{ \MMop{\theta}{\sigma(\ol,s)}{s}{\oW}{\piexp}, \MMop{\otheta}{\ol}{s}{\oW}{\piexp} } \\
    \geq&~
    \DTV{ \MMop{\theta}{l}{s}{\oW}{\piexp}, \MMop{\theta}{\sigma(\ol,s)}{s}{\oW}{\piexp} }
    \geq \frac{1}{2},
\end{align*}
and hence $\DTV{ \MMop{\theta}{l}{s}{\oW}{\piexp}, \MMop{\otheta}{\ol}{s}{\oW}{\piexp} } \geq \frac14$. Therefore,
\begin{align*}
    \eps'
    \geq &~
     \sum_{l,\ol,s} \PP_{\otheta}^{\pi}\paren{ s_W=s, \mthtau=l, \mothtau=\ol } \cdot \DTVt{ \MMop{\theta}{l}{s}{\oW}{\piexp}, \MMop{\otheta}{\ol}{s}{\oW}{\piexp} } \\
    \geq&~
    \sum_{\ol,s} \sum_{l\neq \sigma(\ol,s)} \PP_{\otheta}^{\pi}\paren{ s_W=s, \mthtau=l, \mothtau=\ol } \cdot \frac{1}{16} \\
    =&~ 
    \frac{1}{16}\cdot \PP_{\otheta}^{\pi}\paren{ \mthtau\neq \sigma(\mothtau,s_W) }.
\end{align*}
The proof is hence completed.
\end{proof}

\begin{proposition}[Performance decomposition]\label{prop:err-decomp}
Given LMDP model $\theta$ and reference LMDP $\otheta$, for any trajectory $\otau_h$ with step $W\leq h<H$, we define
\begin{align}\label{eqn:def-err}
    \cE^{\theta;\otheta}(\otau_h) = \max_{a\in\cA}\DTV{ \TT_{\sigma(\mothtau,s_W)}^\theta(\cdot|s_h,a), \TT_{\mothtau}^{\otheta}(\cdot|s_h,a) },
\end{align}
where $\sigma=\sigma_{\theta;\otheta}:[L]\times\cS\to[L]$ is the function defined in \cref{eqn:def-sigma}. Then it holds that
\begin{align}\label{eqn:tv-to-err-decomp}
    \!\!\DTV{ \PP^{\pi}_\theta, \PP^{\pi}_{\otheta} }
    \leq 300\DTV{ \PP^{\modp{\pi}}_\theta, \PP^{\modp{\pi}}_{\otheta} } + 150\perr_{\theta,W}(\pi)+ 150\perr_{\otheta,W}(\pi) + \sum_{h=W}^{H-1} \EE_{\otheta}^{\pi} \cE^{\theta;\otheta}(\otau_h).
\end{align}
Conversely, for any step $W\leq h<H$, %
\begin{align}\label{eqn:decomp-to-Hell}
\begin{aligned}
    \EE_{\otheta}^{\pi} \cE^{\theta;\otheta}(\otau_h)^2 
    \leq&~
    18A\DH{ \PP^{\varphi_h(\pi)}_\theta, \PP^{\varphi_h(\pi)}_{\otheta} }+300\DH{ \PP^{\modp{\pi}}_\theta, \PP^{\modp{\pi}}_{\otheta} } \\
    &~+ 200\perr_{\theta,W}(\pi)+200\perr_{\otheta,W}(\pi).
\end{aligned}
\end{align}
\end{proposition}

\begin{proof}
We first prove \cref{eqn:tv-to-err-decomp}. Notice that, by \cref{lemma:TV-cond}, 
\begin{align}\label{eqn:decomp-proof-0}
\begin{aligned}
    \DTV{ \PP^{\pi}_\theta, \PP^{\pi}_{\otheta} }
    \leq&~ \DTV{ \PP^{\pi}_\theta(\otau_W=\cdot), \PP^{\pi}_{\otheta}(\otau_W=\cdot) } \\
    &~+
    \EE_{\otheta}^\pi\brac{ \DTV{ \PP_\theta^{\pi}\paren{ \otau_{W:H}=\cdot  | \otau_W } , \PP_{\otheta}^{\pi}\paren{ \otau_{W:H}=\cdot  | \otau_W } } }.
\end{aligned}
\end{align}

Using \cref{eqn:cond-to-mdp-tv} and the triangle inequality of TV distance, we have %
\begin{align*}
    &~\DTV{ \PP_\theta^{\pi}\paren{ \otau_{W:H}=\cdot  | \otau_W } , \PP_{\otheta}^{\pi}\paren{ \otau_{W:H}=\cdot  | \otau_W } } \\
    \leq &~ \DTV{ \MMth{ \pitau{W-1} } , \MMoth{ \pitau{W-1} } } 
    + \perrtau{ \theta } + \perrtau{ \otheta },
\end{align*}
and taking expectation over $\otau_W\sim \PP_\otheta^\pi$, we obtain
\begin{align}\label{eqn:decomp-proof-1}
\begin{aligned}
    &~\EE_{\otheta}^\pi\brac{ \DTV{ \PP_\theta^{\pi}\paren{ \otau_{W:H}=\cdot  | \otau_W } , \PP_{\otheta}^{\pi}\paren{ \otau_{W:H}=\cdot  | \otau_W } } } \\
    \leq&~
    \EE_{\otheta}^\pi\brac{ \DTV{ \MMth{ \pitau{W-1} } , \MMoth{ \pitau{W-1} }  } }
    +\EE_{\otheta}^\pi\brac{ \perrtau{ \theta }} 
    +\EE_{\otheta}^\pi\brac{ \perrtau{ \otheta } }.
\end{aligned}
\end{align}
For the last two term in the RHS of \cref{eqn:decomp-proof-1}, we have $\EE_{\otheta}^{\pi}\brac{ \perrtau{ \otheta } } = \perr_{\otheta,W}(\pi)$ and
\begin{align}\label{eqn:decomp-proof-2}
    \EE_{\otheta}^{\pi}\brac{ \perrtau{ \theta } } 
    \leq
    \EE_{\theta}^{\pi}\brac{ \perrtau{ \theta } } + \DTV{ \PP^{\pi}_\theta(\otau_W=\cdot), \PP^{\pi}_{\otheta}(\otau_W=\cdot) }.
\end{align}
To bound the first term in the RHS of \cref{eqn:decomp-proof-1}, we consider the event $E_{\theta;\otheta}\defeq \set{ \mthtau= \sigma(\mothtau,s_W) }$. Under event $E_{\theta;\otheta}$, by \cref{lemma:TV-cond} we have
\begin{align*}
    &~ \DTV{ \MMth{ \pitau{W-1} } , \MMoth{ \pitau{W-1} }  } \\
    \leq&~
    \sum_{h=W}^{H-1} \EE \brcond{ \DTV{ \TT_{\mthtau}^\theta(\cdot|s_h,a_h), \TT_{\mothtau}^\otheta(\cdot|s_h,a_h) } }{\tau_h\sim \PP^{\pi}_{\otheta}(\cdot|\otau_W)} \\
    \leq&~
    \sum_{h=W}^{H-1} \EE \brcond{ \max_a \DTV{ \TT_{\mthtau}^\theta(\cdot|s_h,a), \TT_{\mothtau}^\otheta(\cdot|s_h,a) } }{\otau_h\sim \PP^{\pi}_{\otheta}(\cdot|\otau_W)} \\
    \stackrel{E_{\theta;\otheta}}{=}&~
    \sum_{h=W}^{H-1} \EE \brcond{ \cE^{\theta;\otheta}(\otau_{h}) }{\tau_h\sim \PP^{\pi}_{\otheta}(\cdot|\otau_W)} 
    =
    \sum_{h=W}^{H-1} \EE_\otheta^\pi \brcond{ \cE^{\theta;\otheta}(\otau_{h}) }{\otau_W} .
\end{align*}
Taking expectation over $\otau_W\sim\PP_\otheta^\pi$, it holds
\begin{align}\label{eqn:decomp-proof-3}
\EE_{\otheta}^\pi\brac{ \DTV{ \MMth{ \pitau{W-1} } , \MMoth{ \pitau{W-1} }  } }
\leq \PP(E_{\theta;\otheta}^c)+\sum_{h=W}^{H-1} \EE_{\otheta}^{\pi} \cE^{\theta;\otheta}(\otau_h).
\end{align}
Combining \cref{eqn:decomp-proof-1} with \cref{eqn:decomp-proof-2},  \cref{eqn:decomp-proof-3} and \cref{eqn:map-err-hell} (\cref{prop:map-err}), the proof of \cref{eqn:tv-to-err-decomp} is completed.

We proceed similarly to prove \cref{eqn:decomp-to-Hell}. Notice that for any trajectory $\tau_h$,
\begin{align*}
    \PP_{\theta}(s_{h+1}=\cdot|\tau_h)=\EE_{m\sim \belief_\theta(\otau_h)}\brac{ \TT_{\theta,m}(\cdot|s_h,a_h) }.
\end{align*}
Therefore,
\begin{align*}
    \DTV{ \PP_{\theta}(s_{h+1}=\cdot|\tau_h), \TT_{\mthtau}^\theta(\cdot|s_h,a_h) }
    \leq \sum_{m\neq m_\theta(\otau_h) } \belief_\theta(\otau_h)[m]
    = \perrtauh{\theta},
\end{align*}
and hence
\begin{align*}
    \DTV{ \TT_{\mthtau}^\theta(\cdot|s_h,a_h), \TT_{\mothtau}^\otheta(\cdot|s_h,a_h) }
    \leq&~ \DTV{ \PP_{\theta}(s_{h+1}=\cdot|\tau_h), \PP_{\otheta}(s_{h+1}=\cdot|\tau_h) } \\
    &~+\perrtauh{\theta} + \perrtauh{\otheta}.
\end{align*}
In particular, given $h \geq W$, for any trajectory $\otau_h$ whose prefix $\otau_W$ satisfies $\otau_W\in E_{\theta;\otheta}$, we have
\begin{align*}
    \cE^{\theta;\otheta}(\otau_h) 
    \leq \max_{a}\DTV{ \PP_{\theta}(s_{h+1}=\cdot|\otau_h,a), \PP_{\otheta}(s_{h+1}=\cdot|\otau_h,a) } 
    + \perrtauh{\theta} + \perrtauh{\otheta}.
\end{align*}
Thus,
\begin{align*}
    \indic{E_{\theta;\otheta}}\cE^{\theta;\otheta}(\tau_h)^2
    \leq 3\max_{a}\DTVt{ \PP_{\theta}(s_{h+1}=\cdot|\otau_h,a), \PP_{\otheta}(s_{h+1}=\cdot|\otau_h,a) } 
    + 3\perrtauh{\theta} + 3\perrtauh{\otheta}. 
\end{align*}
Taking expectation over $\tau_h\sim \PP_\otheta^\pi$, we have
\begin{align*}
    \EE_{\otheta}^{\pi} \cE^{\theta;\otheta}(\tau_h)^2 
    \leq&~ \PP_\otheta^\pi\paren{E_{\theta;\otheta}^c} 
    + 3\EE_{\otheta}^{\pi}\brac{ \max_{a}\DTVt{ \PP_{\theta}(s_{h+1}=\cdot|\otau_h,a), \PP_{\otheta}(s_{h+1}=\cdot|\otau_h,a) } } \\
    &~+ 3\EE_{\otheta}^{\pi}\brac{ \perrtauh{\theta} }
    + 3\EE_{\otheta}^{\pi}\brac{ \perrtauh{\otheta} } .
\end{align*}
Notice that
\begin{align*}
    &~\EE_{\otheta}^{\pi}\brac{ \max_{a}\DTVt{ \PP_{\theta}(s_{h+1}=\cdot|\otau_h,a), \PP_{\otheta}(s_{h+1}=\cdot|\otau_h,a) } } \\
    \leq&~
    \EE_{\otheta}^{\pi}\brac{ \sum_{a}\DTVt{ \PP_{\theta}(s_{h+1}=\cdot|\otau_h,a), \PP_{\otheta}(s_{h+1}=\cdot|\otau_h,a) } } \\
    \leq&~ 
    2\EE_{\otheta}^{\pi}\brac{ \sum_{a}\DH{ \PP_{\theta}(s_{h+1}=\cdot|\otau_h,a), \PP_{\otheta}(s_{h+1}=\cdot|\otau_h,a) } } \\
    =&~
    2\EE_{\otheta}^{\pi}\brac{ A\cdot\DH{ \PP_{\theta}(s_{h+1}=\cdot|\otau_h,a_h\sim\unif(\cA)), \PP_{\otheta}(s_{h+1}=\cdot|\otau_h,a_h\sim\unif(\cA)) } } \\
    \leq&~ 4A\DH{ \PP^{\pi\circ_h\unif(\cA)}_\theta(\otau_{h+1}=\cdot), \PP^{\pi\circ_h\unif(\cA)}_{\otheta}(\otau_{h+1}=\cdot) } \\
    \leq&~ 4A\DH{ \PP^{\varphi_h(\pi)}_\theta, \PP^{\varphi_h(\pi)}_{\otheta} },
\end{align*}
where the third inequality follows from \cref{lemma:Hellinger-cond}.
By definition, we know $\EE_{\otheta}^{\pi}\brac{ \perrtauh{\otheta} }= \perr_{\otheta,h}(\pi) \leq \perr_{\theta,W}(\pi)$ (\cref{lem:perr-monotone}), and using \cref{lemma:multiplicative-hellinger}, we also have
\begin{align*}
    \EE_{\otheta}^{\pi}\brac{ \perrtauh{\theta} } 
    \leq&~
    3\EE_{\theta}^{\pi}\brac{ \perrtauh{\theta} } + 2\DH{ \PP^{\pi}_\theta(\otau_h=\cdot), \PP^{\pi}_{\otheta}(\otau_h=\cdot) } \\
    \leq&~ 
    3\perr_{\theta,W}(\pi)+2\DH{ \PP^{\varphi_h(\pi)}_\theta, \PP^{\varphi_h(\pi)}_{\otheta} }.
\end{align*}
Combining the inequalities above with \cref{eqn:map-err-hell} completes the proof.
\end{proof}

\begin{lemma}\label{lem:perr-monotone}
For $h\geq W$, it holds that $\perr_{\theta,h}(\pi)\leq \perr_{\theta,W}(\pi)$.
\end{lemma}
\begin{proof}
By definition,
\begin{align*}
    \perr_{\theta,h}(\pi)
    =&~ \EE_\theta^\pi\brac{1-\max_m \tPP_\theta(\ms=m|\otau_h) } \\
    \leq &~ \EE_\theta^\pi\brac{1-\tPP_\theta(\ms=\mthtau|\otau_h) } \\
    =&~ 1-\tPP(\ms=\mthtau) \\
    =&~ \perr_{\theta,W}(\otau_W).
\end{align*}
\end{proof}

\subsection{Proof of Theorem~\ref{thm:all-policy-sep-demo}}\label{appdx:proof-all-policy-sep}

We first present and prove a more general result as follows; \cref{thm:all-policy-sep-demo} is then a direct corollary.

\begin{corollary}\label{cor:OMLE-almost-done}
Under the success event \MLEevent of \cref{thm:MLE}, it holds that
\begin{align*}
    V_\star-V_{\ths}(\hat\pi)
    \leqsim \sqrt{ Ld^2\clogK\paren{\frac{AH^2\beta}{K}+\frac{\oW^2(U_{+}+KU_\star)}{K^2}} }+\epssep,
\end{align*}
where we denote $\clogK=\clogKv$, and
\begin{align*}
    U_\star=\sum_{k=1}^K \perr_{\ths,W}(\pi^k), \qquad
    U_+=\sum_{1\leq t<k\leq K} \perr_{\theta^k,W}(\pi^t).
\end{align*}
\end{corollary}

\begin{proof}
Recall that by \cref{prop:optimism}, we have that under \MLEevent
\begin{align*}
    V_\star-V_{\ths}(\hat\pi)\leq \frac{1}{K}\sum_{k=1}^K \DTV{ \PP^{\pi^k}_{\theta^k}, \PP^{\pi^k}_{\ths} }.
\end{align*}
Taking summation of \cref{eqn:tv-to-err-decomp} over $(\theta^1,\pi^1),\cdots,(\theta^K,\pi^K)$, we have 
\begin{align*}
    \sum_{k=1}^K \DTV{ \PP^{\pi^k}_{\theta^k}, \PP^{\pi^k}_{\ths} }
    \leqsim&~ 
    \sum_{k=1}^K \DTV{ \PP^{\modp{\pi^k}}_{\theta^k}, \PP^{\modp{\pi^k}}_{\ths} } 
    + \sum_{k=1}^K\paren{ \perr_{\theta^k,W}(\pi^k)+\perr_{\ths,W}(\pi^k) } \\
    &~+  \sum_{k=1}^K\sum_{h=W}^{H-1} \EE_{\ths}^{\pi^k} \cE^{\theta^k;\ths}(\otau_h).
\end{align*}
By \cref{thm:psr}, we can bound the first term in the RHS above as
\begin{align*}
    \sum_{k=1}^K \DTV{ \PP^{\modp{\pi^k}}_{\theta^k}, \PP^{\modp{\pi^k}}_{\ths} } \leqsim \sqrt{LdAH^2\clogK K\beta}.
\end{align*}
Combining with the fact that $\perr_{\theta^k,W}(\pi^k)\leq \epssep$, we obtain
\begin{align}\label{eqn:eluder-step-1}
    \sum_{k=1}^K \DTV{ \PP^{\pi^k}_{\theta^k}, \PP^{\pi^k}_{\ths} }
    \leqsim 
    \sqrt{LdAH^2\clogK K\beta}+K\epssep+U_\star + \sum_{h=W}^{H-1} \sum_{k=1}^K \EE_{\ths}^{\pi^k} \cE^{\theta^k;\ths}(\otau_h).
\end{align}
Using \cref{eqn:decomp-to-Hell} and the definition of $\modf{\cdot}$, we also know that for all $t,k\in[K]$,
\begin{align}
    \sum_{h=W}^{H-1}\EE_{\ths}^{\pi^t} \cE^{\theta^k;\ths}(\otau_h)^2 \leqsim
    AH \DH{ \PP^{\modf{\pi^t}}_{\theta^k}, \PP^{\modf{\pi^t}}_{\ths} }
    + \oW\perr_{\theta^k,W}(\pi^t)+\oW\perr_{\ths,W}(\pi^t).
\end{align}
Therefore, using \cref{eqn:OMLE-est-err} and the fact that \MLEevent holds, we have 
\begin{align}\label{eqn:eluder-step-2}
    \sum_{t<k} \sum_{h=W}^{H-1} \EE_{\ths}^{\pi^t} \cE^{\theta^k;\ths}(\otau_h)^2 
    \leqsim AH\beta+\oW U_k,
\end{align}
where we denote $U_k\defeq \sum_{t<k}\paren{ \perr_{\theta^k,W}(\pi^t)+\perr_{\ths,W}(\pi^t)}$.
Therefore, it remains to bridge between the inequalities in \cref{eqn:eluder-step-1,eqn:eluder-step-2} above using \cref{prop:coverage-eluder}. 

Fix a $W\leq h\leq H-1$. Notice that $\cE^{\theta^k;\ths}(\otau_h)$ only depends on $\otau_h$ through the tuple
\begin{align*}
    x_h=(\mthstau, s_W, s_h)\in\cX\defeq [L]\times\cS\times\cS,
\end{align*}
and hence we can consider the distribution $p_{t,h}=\PP^{\pi^t}_{\ths}(x_h=\cdot)\in\Delta(\cX)$. It remains to shows that there exists a distribution $\mu_h\in\Delta(\cX)$ such that $p_{t,h}(x)/\mu_h(x)\leq \Ccov\forall x\in\cX$ for some parameter $\Ccov$. 

Under \cref{def:low-rank}, by \cref{prop:rank-to-cov}, there exist distributions $\tmu_m\in\Delta(\cS)$ for each $m\in[L]$ such that
\begin{align*}
    \TT_{\ths,m}(s'|s,a)\leq d\cdot\tmu_m(s'), \qquad \forall m\in[L], (s,a,s')\in\cS\times\cA\times\cS.
\end{align*}
Therefore, in the case $h>W$, for any $x=(m,s,s')\in\cX$, we have
\begin{align*}
    p_{t,h}(x)=\PP^{\pi^t}_{\ths}(x_h=x)
    \leq&~ \PP^{\pi^t}_{\ths}(s_W=s, s_h=s') \\
    =&~ \EE_{(\ms,\tau_{h-1},s_h)}\brac{ \indic{s_W=s, s_h=s'}  } \\
    =&~ \EE_{(\ms,\tau_{h-1})}\brac{ \indic{s_W=s} \cond{ \indic{s_h=s'} }{{s_h\sim \tPP_{\ths}(\cdot|\tau_{h-1},\ms)} }  } \\
    =&~ \EE_{(\ms,\tau_{h-1})}\brac{ \indic{s_W=s} \TT^{\ths}_{\ms}(s'|s_{h-1},a_{h-1})  } \\
    \leq&~ \EE_{(\ms,\tau_{h-1})}\brac{ \indic{s_W=s} \cdot d\cdot\tmu_{\ms}(s')  } \\
    =&~ \EE_{(\ms,\tau_{W-1})}\brac{ \cond{\indic{s_W=s}}{s_W\sim \tPP_{\ths}(\cdot|\tau_{W-1},\ms)} \cdot d\cdot\tmu_{\ms}(s') } \\
    =&~ \EE_{(\ms,\tau_{W-1})}\brac{ \TT^{\ths}_{\ms}(s|s_{W-1},a_{W-1}) \cdot d\cdot\tmu_{\ms}(s') } \\
    \leq&~ \EE_{\ms}\brac{ d\cdot\tmu_{\ms}(s)\cdot d\cdot\tmu_{\ms}(s') } \\
    =&~ d^2\sum_{\ms\in[L]}\rho_{\ths}(\ms)\tmu_{\ms}(s)\tmu_{\ms}(s'),
\end{align*}
where the expectation is taken over $(\ms,\tau_H)\sim \tPP_{\ths}^{\pi^t}$. Thus, we can choose $\mu_h\in\Delta(\cX)$ as %
\begin{align*}
    \mu_h(m,s,s')=\frac{1}{L}\sum_{\ms\in[L]}\rho_{\ths}(\ms)\tmu_{\ms}(s)\tmu_{\ms}(s'), \qquad \forall (m,s,s')\in\cX.
\end{align*}
Then, for $h>W$, $t\in[T]$ and any $x\in\cX$, we know $p_{t,h}(x)\leq Ld^2 \cdot \mu_h(x)$. For the case $h=W$, an argument essentially the same as above also yields that there exists a $\mu_W\in\Delta(\cX)$ such that $p_{t,W}(x)\leq Ld \cdot \mu_W(x)$ for all $t\in[T]$, $x\in\cX$.

We can now apply \cref{prop:coverage-eluder} with $M = A\beta$ to obtain that for all $W\leq h\leq H-1$,
\begin{align}
\begin{aligned}
    \sum_{k=1}^K \EE_{\ths}^{\pi^k} \cE^{\theta^k;\ths}(\otau_h)
    \leqsim&~ \sqrt{Ld^2\log\paren{1+\frac{Ld^2K}{A\beta}} \brac{ KA\beta+\sum_{k=1}^K \sum_{t<k} \EE_{\ths}^{\pi^t} \cE^{\theta^k;\otheta}(\otau_h)^2 }}.
\end{aligned}
\end{align}
Taking summation over $W\leq h\leq H-1$ and using \cref{eqn:eluder-step-2}, we have
\begin{align}\label{eqn:eluder-step-3}
\begin{aligned}
    \sum_{h=W}^{H-1} \sum_{k=1}^K \EE_{\ths}^{\pi^k} \cE^{\theta^k;\ths}(\otau_h)
    \leqsim&~
    \sqrt{Ld^2\clogK \brac{ KAH^2\beta+\oW^2\sum_{k=1}^K U_k }}.
\end{aligned}
\end{align}
Combining \cref{eqn:eluder-step-3} above with \cref{eqn:eluder-step-1}, we can conclude that
\begin{align*}
    \sum_{k=1}^K \DTV{ \PP^{\pi^k}_{\theta^k}, \PP^{\pi^k}_{\ths} }
    \leqsim&~ 
    \sqrt{LdAH^2\clogK K\beta}+K\epssep+U_\star + H\sqrt{Ld^2\clogK \brac{ KAH^2\beta+\oW^2\sum_{k=1}^K U_k }} \\
    \leqsim &~
    \sqrt{Ld^2\clogK \paren{KAH^2\beta+\oW^2(KU_\star+U_+)}}+K\epssep+U_\star \\
    \leqsim &~ \sqrt{Ld^2\clogK \paren{KAH^2\beta+\oW^2(KU_\star+U_+)}}+K\epssep,
\end{align*}
where the last inequality follows from $U_\star\leq K$ and hence $U_\star\leq\sqrt{KU_\star}$. Applying \cref{prop:optimism} completes the proof.
\end{proof}

\paragraph{Proof of \cref{thm:all-policy-sep-demo}}
Under \cref{def:all-policy-sep}, it holds that $\perr_{\theta,W}(\pi)\leq \epssep$ for all $\theta\in\Theta$ and $\pi\in\Pi$ (\cref{prop:latent-MLE}). Therefore, $U_\star\leq K\epssep$, $U_+\leq K^2\epssep$, and \cref{cor:OMLE-almost-done} implies that as long as
\begin{align*}
    K\geqsim \frac{Ld^2AH^2\clogK}{\eps^2}\cdot \beta, \qquad 
    \epssep\leqsim \frac{\eps^2}{Ld^2\oW^2\clogK},
\end{align*}
we have $V_\star-V_{\ths}(\hpi)\leq \eps$, which is fulfilled by the choice of parameters in \cref{thm:all-policy-sep-demo}.
\qed

\subsection{Proof of Theorem~\ref{thm:single-policy-sep-demo}}\label{appdx:proof-single-policy-sep}

According to \cref{cor:OMLE-almost-done}, we only need to upper bound the term $U_\star$ and $U_+$ under \cref{assmp:pi-exp}. The following proposition links these two quantities with the condition $\perr_{\theta^k,W}(\pi^k)\leq \epssep\forall k\in[K]$.

\begin{proposition}\label{prop:err-reg-bound}
Suppose that \cref{assmp:pi-exp} holds. Then for any policy $\pi$, LMDP model $\theta$ and reference LMDP model $\otheta$, it holds that 
\begin{align*}
    \perr_{\theta,W}(\pi)\leq \frac{1}{\alpha}\brac{ 3\DTV{ \PP^{\modp{\pi}}_\theta, \PP^{\modp{\pi}}_{\otheta} }+\perr_{\otheta,W}(\pi) } 
\end{align*}
\end{proposition}

\begin{proof}
Using \cref{eqn:cond-to-mdp-tv-exp} and the triangle inequality, we have
\begin{align*}
    &~\DTV{ \PP_\theta^{\modp{\pi}}\paren{ \otau_{W:H}=\cdot  | \otau_W } , \MMoth{ \piexp } } \\
    \leq&~ 
    \DTV{ \PP_\theta^{\modp{\pi}}\paren{ \otau_{W:H}=\cdot  | \otau_W } , \PP_\otheta^{\modp{\pi}}\paren{ \otau_{W:H}=\cdot  | \otau_W } }
    +\perrtau{\otheta}.
\end{align*}
On the other hand, 
\begin{align*}
    \PP_\theta^{\modp{\pi}}\paren{ \otau_{W:H}=\cdot  | \otau_W }
    =\EE_{m\sim\belief_\theta(\otau_W)} \brac{ \MMop{\theta}{m}{s_W}{ \oW }{\piexp} },
\end{align*}
and hence by \cref{assmp:reg}, it holds that
\begin{align*}
    \DTV{ \PP_\theta^{\modp{\pi}}\paren{ \otau_{W:H}=\cdot  | \otau_W } , \MMoth{ \piexp } }
    \geq \alpha\paren{1-\max_m \belief_\theta(\otau_W)[m]} 
    =\alpha \perrtau{\theta}.
\end{align*}
Taking expectation over $\otau_W\sim\PP_\otheta^\pi$, we obtain
\begin{align*}
    \alpha \EE_{\otheta}^{\pi}\brac{ \perrtau{ \theta } }
    \leq&~
    \EE_\otheta^\pi\brac{ \DTV{ \PP_\theta^{\modp{\pi}}\paren{ \otau_{W:H}=\cdot  | \otau_W } , \MMoth{ \piexp } }} \\
    \leq&~ \EE_\otheta^\pi\brac{ \DTV{ \PP_\theta^{\modp{\pi}}\paren{ \otau_{W:H}=\cdot  | \otau_W } , \PP_\otheta^{\modp{\pi}}\paren{ \otau_{W:H}=\cdot  | \otau_W } } } + \EE_\otheta^\pi\brac{ \perrtau{\otheta} } \\
    \leq&~ 2\DTV{ \PP^{\modp{\pi}}_\theta, \PP^{\modp{\pi}}_{\otheta} }+\perr_{\otheta,W}(\pi),
\end{align*}
where the last inequality follows from \cref{lemma:TV-cond} and the fact that $\EE_\otheta^\pi\brac{ \perrtau{\otheta} }=\perr_{\otheta,W}(\pi)$.
Notice that we also have
\begin{align*}
    \EE_{\otheta}^{\pi}\brac{ \perrtau{ \theta } } 
    \geq&~
    \EE_{\theta}^{\pi}\brac{ \perrtau{ \theta } } - \DTV{ \PP^{\pi}_\theta(\otau_W=\cdot), \PP^{\pi}_{\otheta}(\otau_W=\cdot) } \\
    =&~
    \perr_{\theta,W}(\pi)-\DTV{ \PP^{\pi}_\theta(\otau_W=\cdot), \PP^{\pi}_{\otheta}(\otau_W=\cdot) }.
\end{align*}
Combining the inequalities above completes the proof.
\end{proof}

\paragraph{Proof of \cref{thm:single-policy-sep-demo}}
According to our choice of $(\theta^k,\pi^k)$, we know that $\perr_{\theta^k,W}(\pi^k)\leq \epssep$ always holds for $k\in[K]$. Hence, by \cref{prop:err-reg-bound},
\begin{align*}
    \perr_{\ths,W}(\pi^k) \leq
    \frac{1}{\alpha}\brac{ 3\DTV{ \PP^{\modp{\pi^k}}_{\theta^k}, \PP^{\modp{\pi^k}}_{\ths} }+\epssep }.
\end{align*}
Summing over $k\in[K]$, we obtain that
\begin{align*}
    U_\star=\sum_{k=1}^K \perr_{\ths,W}(\pi^k) 
    \leq&~ \frac{1}{\alpha}\brac{ 3\sum_{k=1}^K \DTV{ \PP^{\modp{\pi^k}}_{\theta^k}, \PP^{\modp{\pi^k}}_{\theta^k} }+K\epssep } \\
    \leqsim&~ \frac1\alpha \sqrt{LdAH^2\clogK K\beta}+\frac{K\epssep}{\alpha},
\end{align*}
where the last inequality follows from \cref{thm:psr}. 

Similarly, by \cref{prop:err-reg-bound}, we can bound
\begin{align*}
    \perr_{\theta^k,W}(\pi^t)
    \leq&~ \frac{1}{\alpha}\brac{ 3\DTV{ \PP^{\modp{\pi^t}}_{\theta^k}, \PP^{\modp{\pi^t}}_{\theta^t} }+\perr_{\theta^t}(\pi^t) } \\
    \leq&~ \frac{1}{\alpha}\brac{ 3\DTV{ \PP^{\modp{\pi^t}}_{\theta^k}, \PP^{\modp{\pi^t}}_{\ths}}+3\DTV{ \PP^{\modp{\pi^t}}_{\theta^t}, \PP^{\modp{\pi^t}}_{\ths} }+\perr_{\theta^t}(\pi^t) }.
\end{align*}
Therefore, taking summation over $1\leq t<k\leq K$, we have
\begin{align*}
    U_+=\sum_{1\leq t<k\leq K} \perr_{\theta^k,W}(\pi^t)
    \leqsim&~ \frac{1}{\alpha}\brac{ \sum_{1\leq t<k\leq K} \DTV{ \PP^{\modp{\pi^t}}_{\theta^k}, \PP^{\modp{\pi^t}}_{\ths}}+K\sum_{t=1}^K\DTV{ \PP^{\modp{\pi^t}}_{\theta^t}, \PP^{\modp{\pi^t}}_{\ths} }+K^2\epssep }.
\end{align*}
By Cauchy inequality, it holds
\begin{align*}
    \sum_{1\leq t<k\leq K} \DTV{ \PP^{\modp{\pi^t}}_{\theta^k}, \PP^{\modp{\pi^t}}_{\ths}}
    \leq \sqrt{K^2\cdot \sum_{1\leq t<k\leq K} \DTVt{ \PP^{\modp{\pi^t}}_{\theta^k}, \PP^{\modp{\pi^t}}_{\ths}}}
    \leqsim K\sqrt{K\beta},
\end{align*}
where we use the fact that $\dTV\leq\sqrt{2}\dH$ and \cref{thm:MLE}. Combining \cref{thm:psr} with the above two inequalities, we can conclude that
\begin{align*}
    U_+=\sum_{1\leq t<k\leq K} \perr_{\theta^k,W}(\pi^t)
    \leqsim \frac1{\alpha} K\sqrt{LdAH^2\iota K\beta} + \frac{K^2\epssep}{\alpha}.
\end{align*}
Hence, \cref{cor:OMLE-almost-done} implies that
\begin{align*}
    V_\star-V_{\ths}(\hat\pi)
    \leqsim \sqrt{ Ld^2\clogK\paren{\frac{AH^2\beta}{\alpha K}+\frac{\epssep}{\alpha}+\frac{1}{\alpha}\sqrt{\frac{LdAH^2\clogK \beta}{K}}} }.
\end{align*}
Therefore, to ensure that $V_\star-V_{\ths}(\hat\pi)\leq \eps$, we only need to ensure
\begin{align*}
    K\geqsim \frac{L^3d^5AH^6\clogK^3}{\alpha^2\eps^4} \cdot \beta, \qquad
    \epssep\leqsim \frac{\alpha\eps^2}{Ld^2\oW^2\clogK}.
\end{align*}
In particular, the choice of parameters in \cref{thm:single-policy-sep-demo} suffices.
\qed

\subsection{Proof of Theorem~\ref{thm:psr}}
\label{appdx:psr}

The proof of \cref{thm:psr} is (almost) a direct analog of the analysis in \citet[Appendix D \& G]{chen2022partially}. However, we may not directly invoke the guarantees there for general PSR to obtain \cref{thm:psr} because PSR is formalized in terms of a set of \emph{core action sequences}, so that the system dynamics is uniquely determined by the dynamics under these action sequences. %
However, for our setting, we are instead given an explorative policy $\piexp$, which is not necessary a mixture of action sequences.

Therefore, in the following,  we present a minimal self-contained proof of \cref{thm:psr}, which is in essence a slight modification of the original proof in \citet{chen2022partially}. We refer the reader to \citet{chen2022partially} for more detailed analysis and proofs.

\newcommand{\Em}{\mathbb{K}}
\newcommand{\tTheta}{\widetilde{\Theta}}

\newcommand{\Lamrev}{\Lambda_{\exp}}

\newcommand{\termin}{\mathsf{terminal}}

\newcommand{\snorm}{\mathsf{normal}}
\newcommand{\srev}{\mathsf{rev}}

\renewcommand{\tT}{\tilde{\T}}
\renewcommand{\tO}{\tilde{\O}}
\newcommand{\tTT}{\Tilde{\TT}}

In the following, we first introduce the notations for POMDPs, which generalize LMDPs. %
\paragraph{POMDPs} A Partially Observable Markov Decision Process (POMDP) is a sequential decision process whose transition dynamics are governed by \emph{latent states}. A POMDP is specified by a tuple $\{\cZ,\cO,\cA,\T,\O,H,\mu_1 \}$, where $\cZ$ is the latent state space,  $\O(\cdot|\cdot):\cZ\to\Delta(\cO)$ is the emission dynamics, $\T(\cdot|\cdot,\cdot):\cZ\times\cA\to\Delta(\cZ)$ is the transition dynamics over the latent states, and $\mu_1\in\Delta(\cZ)$ specifies the distribution of initial state $z_1$. At each step $h$, given the latent state $z_h$ (which the agent cannot observe), the system emits observation $o_h\sim \O(\cdot|z_h)$, receives action $a_h\in\cA$ from the agent, and then transits to the next latent state $z_{h+1}\sim \T(\cdot|z_h, a_h)$ in a Markov fashion. 
The episode terminates immediately after $a_H$ is taken.

In a POMDP with observation space $\cO$ and action space $\cA$, a policy $\pi = \{\pi_h: (\cO\times\cA)^{h-1}\times\cO\to\Delta(\cA) \}_{h=1}^H$ is a collection of $H$ functions. At step $h\in[H]$, an agent running policy $\pi$ observes the observation $o_h$ and takes action $a_{h}\sim \pi_h(\cdot|\tau_{h-1}, o_h)\in\Delta(\cA)$ based on the history $(\tau_{h-1},o_h)=(o_1,a_1,\dots,o_{h-1},a_{h-1},o_h)$. The environment then generates the next observation $o_{h+1}$ based on $\tau_h=(o_1,a_1,\cdots,o_h,a_h)$ (according to the dynamics of the underlying POMDP). %

Suppose that $\tTheta$ is a set of POMDP models with common action space $\cA$ and observation space $\cO$, such that each $\theta\in\tTheta$ specifies the tuple $(\T_\theta,\O_\theta,\mu_\theta)$ and hence the POMDP dynamics. \footnote{Strictly speaking, $\theta$ also specifies $\cZ_\theta$, its own latent state space. For notational simplicity, we always omit the subscript $\theta$ of the state space $\cZ$ in the following analysis.}

Suppose that a step parameter $1\leq W<H$ is given, along with a policy $\piexp$. Then, for each policy $\pi$, we define
\begin{align}\label{eqn:def-policy-modq}
    \modq{\pi}\defeq \frac{1}{W}\sum_{h=0}^{W-1}\pi \circ_{h}\unif(\cA)\circ_{h +1}\piexp
\end{align}
analogously to \cref{def:policy-mod}.
We also consider the emission matrix induced by $\piexp$:
\begin{align}
    \Em_\theta=\brac{ \PP^{\piexp}_\theta((o_1,a_1,\cdots,o_{\oW})=\otau|s_1=s) }_{(\otau,s)} \in\RR^{\cT\times\cZ}, 
\end{align}
where $\oW=H-W+1$, $\cT=(\cO\times\cA)^{\oW-1}\times\cO$. Suppose that for each $\theta\in\Theta$, there exists $\Em_\theta^+\in\RR^{\cZ\times\cT}$ such that $\Em_\theta^+\Em_\theta=\id_\cZ$, and we write $\Lamrev\defeq \max_{\theta\in\Theta} \lone{\Em_\theta^+}$.

\paragraph{Operator representation of POMDP dynamics}
Define
\begin{align}\label{eqn:def-ops}
    \BB_\theta(o,a)=\Em_\theta\T_{\theta,a}\diag(\O_\theta(o|\cdot))\Em_\theta^+, \qquad
    \bq_{\theta,0}=\Em_\theta\mu_\theta.
\end{align}
where we denote $\T_{\theta,a}\defeq \T_\theta(\cdot|\cdot,a)\in\R^{\cZ\times \cZ}$ for each $a\in\cA$, and $\diag(\O_\theta(o|\cdot))\R^{\cZ\times \cZ}$ is the diagonal matrix with the $(z,z)$-entry being $\O(o|z)$ for each $z\in\cZ$.

\newcommand{\trajhtow}{o_{h+1},a_{h+1},\cdots,o_{h+\oW}}
An important property of the definition \cref{eqn:def-ops} is that, for any trajectory $\otau_{h+\oW}=(\tau_h,\trajhtow)$, it holds that
\begin{align*}
    \be_{(\trajhtow)}^\top \BB_\theta(o_h,a_h)\cdots\BB_\theta(o_1,a_1)\bq_{\theta,0} 
    =&~ \PP_\theta(\trajhtow|\tau_h,\piexp)\times \PP_\theta(o_{1:h}|\doac(a_{1:h})),
\end{align*}
where we recall that $\PP_\theta(\trajhtow|\tau_h,\piexp)$ is the probability of observing $\trajhtow$ when executing policy $\piexp$ starting at step $h+1$ in POMDP $\theta$, conditional on the history $\tau_h$ (see also \cref{eqn:cond-pi}).
Therefore, for any policy $\pi$, it holds that %
\begin{align}\label{eqn:prob-to-ops}
    \PP_\theta^{\pi\circ_{h+1}\piexp}(\otau_{h+\oW})=\be_{(\trajhtow)}^\top \BB_\theta(o_h,a_h)\cdots\BB_\theta(o_1,a_1)\bq_{\theta,0} \times \pi(\tau_h).
\end{align}
In particular, we can now express TV distance between model as difference between operators:
\begin{align}\label{eqn:TV-to-ops}
\begin{aligned}
    &~\DTV{ \PP_\theta^{\pi\circ_{h+1}\piexp}, \PP_\otheta^{\pi\circ_{h+1}\piexp} }\\
    =&~
    \frac12\sum_{\tau_h} \pi(\tau_h)\times \lone{ \BB_\theta(o_h,a_h)\cdots\BB_\theta(o_1,a_1)\bq_{\theta,0}-\BB_\otheta(o_h,a_h)\cdots\BB_\otheta(o_1,a_1)\bq_{\otheta,0} }.
\end{aligned}
\end{align}
Also, we denote $\bq_\theta(\tau_h)=\brac{ \PP_\theta((o_{h+1},a_{h+1},\cdots,o_{h+\oW})=\cdot|\tau_h, \piexp) }\in\Delta(\cT)$, then we also have
\begin{align}\label{eqn:bq-to-prob}
    \BB_\theta(o_h,a_h)\cdots\BB_\theta(o_1,a_1)\bq_{\theta,0}
    =\bq_\theta(\tau_h) \times \PP_\theta(\tau_h),
\end{align}
where we recall the notation $\PP_\theta(\tau_h) = \PP_\theta(o_{1:h} | \doac(a_{1:h}))$.

Another important fact is that, for any 1-step policy $\pi:\cO\to\Delta(\cA)$ and $\bq\in\R^{\cT}$, 
\begin{align}
    \sum_{o,a} \pi(a|o)\times \lone{ \BB_\theta(o,a)\bq } \leq&~ \lone{ \Em_\theta^+ \bq }, \label{eqn:B-op-stability-1-step-0}\\
    \sum_{o,a} \pi(a|o)\times \lone{ \Em_\theta^+\BB_\theta(o,a)\bq } \leq&~ \lone{ \Em_\theta^+ \bq } \label{eqn:B-op-stability-1-step}.
\end{align}
This is because $\lone{\Em_\theta}\leq 1$, $\lone{\TT_{\theta,a}}\leq 1$, and $\sum_{o,a} \pi(a|o) \O_\theta(o|z)=1$ for any $z\in\cZ$. Hence, we can apply \cref{eqn:B-op-stability-1-step} recursively to show that, for any $h$-step policy $\pi$,
\begin{align}\label{eqn:B-op-stability}
    \sum_{\tau_{h}} \pi(\tau_{h})\times \lone{ \BB_\theta(o_h,a_h)\cdots\BB_\theta(o_1,a_1)\bq } \leq \lone{ \Em_\theta^+ \bq }.
\end{align}

\newcommand{\cEb}{\bar{\cE}}
\newcommand{\wexp}{\circ_{W}\piexp}
\newcommand{\pihu}{\pi\circ_h\unif(\cA)\circ_{h+1}\piexp}
\newcommand{\pihpu}{\pi\circ_{h-1}\unif(\cA)\circ_{h}\piexp}
\newcommand{\pih}{\pi\circ_h\piexp}
\begin{proposition}\label{prop:rev-decomp}
For each pair of models $\theta, \otheta \in \Theta$, we define $\cEb^{\theta;\otheta}:\R^{\cT}\to\R$ as follows:
\begin{align}\label{eqn:def-rev-err}
    \cEb^{\theta;\otheta}(\bq)\defeq \frac12\max_{\pi':\cO\to\Delta(\cA)} \sum_{o,a} \pi'(a|o)\times \lone{ \Em_\theta^+\paren{ \BB_\theta(o,a)-\BB_\otheta(o,a) }\bq }
\end{align}
For each step $h$, define\footnote{
The error functional might seem strange at first glance, but it can be regarded as a counterpart of the decomposition \cref{eqn:example-decomp-MDP} for MDP. Indeed, when $\tTheta$ is a class of MDP models (i.e. $\cZ=\cO=\cS$ and $\Em=\O=\id_{\cS}$), then %
\begin{align*}
    \cE^{\theta;\otheta}(\tau_{h-1})=\EE_{s_h|\tau_{h-1},\otheta} \max_a \DTV{ \TT_{\theta}(\cdot|s_h,a), \TT_{\otheta}(\cdot|s_h,a) }.
\end{align*}
}%
\begin{align*}
    \cE^{\theta;\otheta}(\tau_h)\defeq \cEb^{\theta;\otheta}\paren{\bq_\otheta(\tau_h)}, \qquad
    \cE^{\theta;\otheta}_0\defeq \frac12\lone{\mathbb{K}_\theta^+(\bq_{\theta,0}-\bq_{\otheta,0})}.
\end{align*}
Then it holds that
\begin{align}\label{eqn:rev-tv-to-err}
    \DTV{ \PP_\theta^{\pi\wexp}, \PP_\otheta^{\pi\wexp} }
    \leq \cE^{\theta;\otheta}_0+\sum_{h=1}^{W-1} \EE_\otheta^\pi \cE^{\theta;\otheta}(\tau_{h-1}).
\end{align}
Conversely, it holds
\begin{align}\label{eqn:rev-err-to-hell}
    (\cE^{\theta;\otheta}_0)^2+\sum_{h=1}^{W-1}\EE_\otheta^\pi \cE^{\theta;\otheta}(\tau_{h-1})^2 \leq 8AW\Lamrev^2 \DH{ \PP_\theta^{\modq\pi}, \PP_\otheta^{\modq\pi} }.
\end{align}
\end{proposition}

\begin{proof}
Before presenting the proof, we first introduce some notations. We abbreviate $\BB_\theta(o_1,a_1,\cdots,o_l,a_l)=\BB_\theta(o_l,a_l)\cdots\BB_\theta(o_1,a_1)$. For a trajectory $\tau_H=(o_1,a_1,\cdots,o_H,a_H)$, we write $\tau_{h':h}=(o_{h'},a_{h'},\cdots,o_h,a_h)$ and $\otau_{h':h}=(o_{h'},a_{h'},\cdots,o_h)$.

Using \cref{eqn:TV-to-ops}, we have
\begin{align*}
    &~ 2\DTV{ \PP_\theta^{\pi\wexp}, \PP_\otheta^{\pi\wexp} }\\
    \stackrel{\text{\cref{eqn:TV-to-ops}}}{=}&~
    \sum_{\tau_{W-1}} \pi(\tau_{W-1})\times \lone{ \BB_\theta(o_{W-1},a_{W-1})\cdots\BB_\theta(o_1,a_1)\bq_{\theta,0}-\BB_\otheta(o_{W-1},a_{W-1})\cdots\BB_\otheta(o_1,a_1)\bq_{\otheta,0} } \\
    \leq&~
    \sum_{\tau_{W-1}} \pi(\tau_{W-1})\lone{ \BB_\theta(\tau_{1:W-1})\paren{ \bq_{\theta,0}-\bq_{\otheta,0}}} \\
    &~
    +\sum_{\tau_{W-1}} \pi(\tau_{W-1})\times \sum_{h=1}^{W-1}\lone{ \BB_\theta(\tau_{h+1:W-1})\paren{ \BB_\theta(o_h,a_h)-\BB_\otheta(o_h,a_h) }\BB_\otheta(\tau_{1:h-1})\bq_{\otheta,0} }  \\
    \stackrel{\text{\cref{eqn:B-op-stability}}}{\leq}&~
    \frac12\lone{\Em_\theta^+\paren{\bq_{\theta,0}-\bq_{\otheta,0}} }
    +
    \frac12\sum_{h=1}^{W-1} \sum_{\tau_h} \pi(\tau_h)\times \lone{ \Em_\theta^+\paren{ \BB_\theta(o_h,a_h)-\BB_\otheta(o_h,a_h) }\BB_\otheta(\tau_{1:h-1})\bq_{\otheta,0} }\\
    \stackrel{\text{\cref{eqn:bq-to-prob}}}{=}&~\frac12\lone{\Em_\theta^+\paren{\bq_{\theta,0}-\bq_{\otheta,0}} }
    +
    \frac12\sum_{h=1}^{W-1} \sum_{\tau_h} \pi(\tau_h)\times \lone{ \Em_\theta^+\paren{ \BB_\theta(o_h,a_h)-\BB_\otheta(o_h,a_h) }\bq_\otheta(\tau_{h-1}) } \times \PP_\theta(\tau_{h-1})\\
    =&~ 
    \cE^{\theta;\otheta}_0+
    \frac12\sum_{h=1}^{W-1} \sum_{\tau_{h-1}}\sum_{o_h,a_h} \PP_\theta^\pi(\tau_{h-1})\times \pi(a_h|\tau_{h-1},o_h)\times \lone{ \Em_\theta^+\paren{ \BB_\theta(o_h,a_h)-\BB_\otheta(o_h,a_h) }\bq_\otheta(\tau_{h-1}) }\\
    \leq&~ 
    \cE^{\theta;\otheta}_0+
    \sum_{h=1}^{W-1} \sum_{\tau_{h-1}}\PP_\theta^\pi(\tau_{h-1})\times \cE^{\theta;\otheta}\paren{\bq_\otheta(\tau_{h-1})},
\end{align*}
where 
the last two lines follow from the definition \cref{eqn:def-rev-err}.
This completes the proof of \cref{eqn:rev-tv-to-err}.

Next, we proceed to prove \cref{eqn:rev-err-to-hell}. By definition,
\begin{align*}
    2\cE^{\theta;\otheta}(\tau_h)
    =&~ \max_{\pi'}\sum_{o,a} \pi'(a|o)\times \lone{ \Em_\theta^+\paren{ \BB_\theta(o,a)-\BB_\otheta(o,a) }\bq_\otheta(\tau_{h-1}) } \\
    \leq&~ \max_{\pi'}\sum_{o,a} \pi'(a|o)\times \lone{ \Em_\theta^+\paren{ \BB_\theta(o,a)\bq_\theta(\tau_{h-1})-\BB_\otheta(o,a)\bq_\otheta(\tau_{h-1}) } } \\
    &~+ \max_{\pi'} \sum_{o,a} \pi'(a|o)\times \lone{ \Em_\theta^+\BB_\theta(o,a) \paren{\bq_\theta(\tau_{h-1})-\bq_\otheta(\tau_{h-1})} }.
\end{align*}
For the first term, notice that for any $o\in\cO$, $a\in\cA$,
\begin{align*}
    \BB_\theta(o,a)\bq_\theta(\tau_{h-1})
    =\brac{ \PP_\theta(o_h=o,\otau_{h+1:h+\oW}=\cdot|\tau_{h-1},a_h=a,a_{h+1:h+\oW}\sim\piexp) }\in\RR^{\cT}.
\end{align*}
Therefore, for any step $1\leq h\leq W-1$ and any 1-step policy $\pi':\cO\to\Delta(\cA)$, we have
\begin{align*}
    &~ \sum_{o,a} \pi'(a|o)\times \lone{ \Em_\theta^+\paren{ \BB_\theta(o,a)\bq_\theta(\tau_{h-1})-\BB_\otheta(o,a)\bq_\otheta(\tau_{h-1}) } } \\
    \leq &~ \Lamrev\sum_{o,a} \pi'(a|o)\times \lone{ \BB_\theta(o,a)\bq_\theta(\tau_{h-1})-\BB_\otheta(o,a)\bq_\otheta(\tau_{h-1}) } \\
    =&~ 2\Lamrev \DTV{ \PP_\theta(\otau_{h:h+\oW}=\cdot|\tau_{h-1},\pi'\circ\piexp), \PP_\otheta(\otau_{h:h+\oW}=\cdot|\tau_{h-1},\pi'\circ\piexp) },
\end{align*}
where the inequality uses the fact that $\| \Em_\theta^+\|_1 \leq \Lamrev$ for all $\theta \in \Theta$. 
Furthermore, 
\begin{align*}
    &~\frac{1}{2}\DTVt{ \PP_\theta(\otau_{h:h+\oW}=\cdot|\tau_{h-1},\pi'\circ\piexp), \PP_\otheta(\otau_{h:h+\oW}=\cdot|\tau_{h-1},\pi'\circ\piexp) } \\
    \leq&~ \DH{ \PP_\theta(\otau_{h:h+\oW}=\cdot|\tau_{h-1},\pi'\circ\piexp), \PP_\otheta(\otau_{h:h+\oW}=\cdot|\tau_{h-1},\pi'\circ\piexp) } \\
    \leq&~ \sum_{a\in\cA} \DH{ \PP_\theta(\otau_{h:h+\oW}=\cdot|\tau_{h-1},a\circ\piexp), \PP_\otheta(\otau_{h:h+\oW}=\cdot|\tau_{h-1},a\circ\piexp) } \\
    =&~
    A\DH{ \PP_\theta(\otau_{h:h+\oW}=\cdot|\tau_{h-1},\unif(\cA)\circ\piexp), \PP_\otheta(\otau_{h:h+\oW}=\cdot|\tau_{h-1}, \unif(\cA)\circ\piexp) },
\end{align*}
where the second inequality uses the fact that squared Hellinger distance is an $f$-divergence. 
For the second term, by the definition of $\BB_\theta$, we have
\begin{align*}
    &~ \sum_{o,a} \pi'(a|o)\times \lone{ \Em_\theta^+\BB_\theta(o,a) \paren{\bq_\theta(\tau_{h-1})-\bq_\otheta(\tau_{h-1})} } 
    \stackrel{\text{\cref{eqn:B-op-stability-1-step}}}{\leq}
    \lone{ \Em_\theta^+\paren{ \bq_\theta(\tau_{h-1})-\bq_\otheta(\tau_{h-1}) } } \\
    \leq&~
    \Lamrev \lone{ \bq_\theta(\tau_{h-1})-\bq_\otheta(\tau_{h-1}) } \\
    =&~ \Lamrev \cdot2\DTV{ \PP_\theta(\otau_{h:h+\oW-1}=\cdot|\tau_{h-1},\piexp), \PP_\otheta(\otau_{h:h+\oW-1}=\cdot|\tau_{h-1},\piexp) }
\end{align*}
Combining the inequalities above and applying \cref{lemma:Hellinger-cond}, we obtain
\begin{align}\label{eqn:rev-err-to-hell-step}
\begin{aligned}
    \EE_\otheta^\pi \cE^{\theta;\otheta}(\tau_{h-1})^2 
    \leq&~ 4A\Lamrev^2\DH{ \PP_\theta^{\pihu}, \PP_\otheta^{\pihu} } \\
    &~+4\Lamrev^2\DH{ \PP_\theta^{\pih}, \PP_\otheta^{\pih} }.
\end{aligned}
\end{align}
Notice that for step $h\geq 2$, we have
\begin{align*}
    \DH{ \PP_\theta^{\pih}, \PP_\otheta^{\pih} }
    \leq A\DH{ \PP_\theta^{\pihpu}, \PP_\otheta^{\pihpu} },
\end{align*}
and we also have
\begin{align}
\begin{aligned}
    \cE^{\theta;\otheta}_0
    =\frac12\lone{\Em_\theta^+\paren{\bq_{\theta,0}-\bq_{\otheta,0}}}
    \leq&~
    \Lamrev\DTV{ \PP_\theta^{\piexp}(\otau_{1:\oW}=\cdot), \PP_\otheta^{\piexp}(\otau_{1:\oW}=\cdot) } \\
    \leq&~ \sqrt{2}\Lamrev\dH\paren{ \PP_\theta^{\piexp}, \PP_\otheta^{\piexp} }.
\end{aligned}
\end{align}
Combining the inequalities above completes the proof of \cref{eqn:rev-err-to-hell}.
\end{proof}

\begin{proposition}\label{prop:rev-eluder}
Suppose that $D=\rank(\T_{\ths})$, $\beta\geq 1$, and $(\theta^1,\pi^1),\cdots,(\theta^K,\pi^K)$ is a sequence of (POMDP, policy) pairs such that for all $k\in[K]$,
\begin{align*}
    \sum_{t<k} \DH{ \PP^{\modq{\pi^t}}_{\theta^k}, \PP^{\modq{\pi^t}}_{\ths} }
    \leq M.
\end{align*}
Then it holds that
\begin{align*}
    \sum_{k=1}^K \DTV{ \PP_{\theta^k}^{\pi^k\wexp}, \PP_{\ths}^{\pi^k\wexp} }
    \leqsim \sqrt{\Lamrev^2ADW^2\tilde{\iota} \cdot KM},
\end{align*}
where $\tilde{\iota}=\log\paren{1+\frac{2\Lamrev^2KD}{AM}}$.
\end{proposition}

\begin{proof}
Using \cref{prop:rev-decomp}, we have
\begin{align}\label{eqn:rev-step-1}
    \sum_{k=1}^K \DTV{ \PP_{\theta^k}^{\pi^k\wexp}, \PP_{\ths}^{\pi^k\wexp} }
    \leq \sum_{k=1}^K 1\wedge \cE^{\theta^k;\ths}_0+\sum_{h=1}^{W-1} \sum_{k=1}^K 1\wedge \EE_{\ths}^{\pi^k} \cE^{\theta^k;\ths}(\tau_{h-1}),
\end{align}
and for any pair of $(t,k)$,
\begin{align*}
    (\cE^{\theta^k;\ths}_0)^2+\sum_{h=1}^{W-1} \EE_{\ths}^{\pi^t} \cE^{\theta^k;\ths}(\tau_{h-1})^2
    \leq 8AW\Lamrev^2\DH{ \PP^{\modq{\pi^t}}_{\theta^k}, \PP^{\modq{\pi^t}}_{\ths} }.
\end{align*}
In particular, for any $k\in[K]$,
\begin{align}\label{eqn:rev-step-2}
    \sum_{t<k}(\cE^{\theta^k;\ths}_0)^2+\sum_{h=1}^{W-1} \sum_{t<k} \EE_{\ths}^{\pi^t} \cE^{\theta^k;\ths}(\tau_{h-1})^2
    \leq 8AW\Lamrev^2 M.
\end{align}
It remains to apply \cref{prop:semi-linear-eluder} to bridge between \cref{eqn:rev-step-1} and \cref{eqn:rev-step-2}.

For each $k\in[K]$, define $f_k=\cEb^{\theta^k;\ths}:\RR^{\cT}\to\R$. By definition, $f_k$ takes the form
\begin{align*}
    f_k(x)=\max_{\pi}\sum_{o,a,s} \abs{\iprod{x}{y_{k,(o,a),\pi}}}
\end{align*}
where $y_{k,(o,a),\pi}^\top=\pi(a|o)\times \be_{s}^\top\Em_{\theta^k}^+\paren{ \BB_{\theta^k}(o,a)-\BB_{\ths}(o,a) }$. It is also easy to verify that $f_k(x)\leq 2\Lamrev^2\lone{x}$ using $\lone{\Em_\theta^+}\leq\Lamrev$ and $\lone{\Em_{\theta^\star}^+}\leq\Lamrev$. Furthermore, for each step $1\leq h\leq W-1$, the set
\begin{align*}
    \cX_h\defeq\set{ \bq_{\ths}(\tau_{h-1}): \tau_{h-1}\in(\cO\times\cA)^{h-1} }
\end{align*}
spans a subspace of dimension at most $D$. 

Therefore, applying \cref{prop:semi-linear-eluder} yields that for each $1\leq h\leq W-1$
\begin{align}
    \sum_{k=1}^K 1\wedge \EE_{\ths}^{\pi^k} \cE^{\theta^k;\ths}(\tau_{h-1})
    \leqsim \sqrt{D\tilde{\iota}\brac{K\cdot AM+\sum_{k=1}^K\sum_{t<k}\EE_{\ths}^{\pi^t} \cE^{\theta^k;\ths}(\tau_{h-1})^2}},
\end{align}
where $\tilde{\iota}=\log(1+2\Lamrev^2DK/AM)$. Similarly, treating $\cE^{\theta^k;\ths}_0$ as a function over the singleton set, we also have
\begin{align*}
    \sum_{k=1}^K 1\wedge \cE^{\theta^k;\ths}_0 
    \leqsim \sqrt{ \tilde{\iota}\brac{KAM + \sum_{k=1}^K\sum_{t<k}(\cE^{\theta^k;\ths}_0 )^2} }
\end{align*}
Combining the two inequalities above with \cref{eqn:rev-step-1} and \cref{eqn:rev-step-2}, we obtain
\begin{align*}
    \sum_{k=1}^K \DTV{ \PP_{\theta^k}^{\pi^k\wexp}, \PP_{\ths}^{\pi^k\wexp} }
    \leq&~
    \sum_{k=1}^K 1\wedge \cE^{\theta^k;\ths}_0+\sum_{h=1}^{W-1} \sum_{k=1}^K 1\wedge \EE_{\ths}^{\pi^k} \cE^{\theta^k;\ths}(\tau_{h-1}) \\
    \leqsim&~
    \sqrt{DW\tilde{\iota}\brac{KAM+\sum_{k=1}^K\sum_{t<k}\paren{(\cE^{\theta^k;\ths}_0 )^2+\sum_{h=1}^{W-1}\EE_{\ths}^{\pi^t} \cE^{\theta^k;\ths}(\tau_{h-1})^2}}} \\
    \leqsim&~
    \sqrt{DW\iota \cdot K\cdot \Lamrev^2AWM},
\end{align*}
where the first inequality is \cref{eqn:rev-step-1}, the second inequality follows from Cauchy-Schwarz, and the last inequality follows from \cref{eqn:rev-step-2} and the given condition.
\end{proof}

\newcommand{\pomdp}[1]{\mathsf{pomdp}(#1)}

\paragraph{Proof of \cref{thm:psr}} Recall that $\Theta$ is a class of LMDP with common state space $\cS$. For each LMDP $\theta\in\Theta$, we construct a POMDP $\pomdp{\theta}$ with latent state space $\cZ=\cS\times\supp(\rho_\theta)$ and observation space $\cO=\cS$ as follows: %
\begin{itemize}
    \item The initial state is $\ts_1=(s_1,m)$, where $m\sim \rho_\theta$, $s_1\sim \mu_{\theta,m}$.
    \item The state $\ts=(s,m)$ always emits $o=s$ as the observation. After an action $a$ is taken, the next state is generated as $\ts'=(s',m)$ where $s'\sim\TT_{\theta,m}(\cdot|s,a)$.
\end{itemize}
The transition matrix of $\pomdp{\theta}$ specified above can also be written as
\begin{align*}
    \TT_{\pomdp{\theta}}=\diag\paren{ \TT_{\theta,m} }_{m\in\supp(\rho_\theta)},
\end{align*}
up to reorganization of coordinates. Therefore, we have $\rank(\TT_{\pomdp{\ths}})\leq Ld$.

Because $\cO=\cS$, any policy for the LMDP $\theta$ is a policy for the POMDP $\pomdp{\theta}$, and vice versa. Furthermore, it is easy to verify that for any policy $\pi$, the trajectory distribution $\PP_{\pomdp{\theta}}^{\pi}(\tau_H=\cdot)$ agrees with the distribution $\PP_{\theta}^{\pi}(\tau_H=\cdot)$. Hence, for each $\theta\in\Theta$, %
\begin{align*}
    \Em_{\pomdp{\theta}}=\diag\paren{ \MM_{*,\oW}^\theta(\piexp,s) }_{s\in\cS},
\end{align*}
where we denote
\begin{align*}
    \MM_{*,\oW}^\theta(\piexp,s)\defeq [\MM_{m,\oW}^\theta(\piexp,s)]_{m\in\supp(\rho_\theta)}\in\RR^{(\cA\times\cS)^{\oW-1}\times\supp(\rho_\theta)}.
\end{align*}
By \cref{lem:DB-inv}, as long as $\om(\oW)\geq \log(2L)$, for each $(s,m)\in\cZ$, there exists a left inverse of $\MM_{*,\oW}^\theta(\piexp,s)$ with $\ell_1$ norm bounded by 2. In particular, we apply \cref{lem:DB-inv} to conclude the existence of a left inverse with the desired norm bound for each block  of the block diagonal matrix $\Em_{\pomdp{\theta}}$. Therefore, there exists a left inverse of $\Em_{\pomdp{\theta}}$ with $\ell_1$ norm bounded by 2, and hence $\Lamrev\leq 2$. 

Therefore, we can now apply \cref{prop:rev-eluder} to complete the proof of \cref{thm:psr}. \qed

\subsection{A sufficient condition for Assumption~\ref{assmp:pi-exp}}
\label{appdx:pi-exp-example}

The following proposition indicates that \cref{assmp:pi-exp} is not that strong as it may seem: it holds for a broad class of LMDPs under relatively mild assumptions on the support of each MDP instance. 

\begin{proposition}\label{prop:reg-cond-example}
Suppose that there is a policy $\pi_0$ and parameter $W_0\geq \om^{-1}(3\log(L/\alpha_0))$, such that for each $\theta\in\Theta$, the LMDP $M_\theta$ is $\om$-separated under $\pi_0$, and there exists $\mu_\theta:\cS\to\Delta(\cS)$ so that
\begin{align*}
    \TT_{\theta,m}^{\pi_0}(s_{W_0}=s'|s_1=s)\geq \alpha_0\mu_\theta(s'|s), \qquad \forall m\in\supp(\rho_\theta), s,s'\in\cS.
\end{align*}
Let $\piexp=\pi_0\circ_{W_0}\pi_0$. Then \cref{assmp:pi-exp} holds with $\Wexp=2W_0$ and $\alpha=\frac{\alpha_0}{32}$.
\end{proposition}

For the sake of notational simplicity, we first prove a more abstract version of \cref{prop:reg-cond-example}.

\newcommand{\QQsq}{\QQ^{\otimes 2}}
\begin{proposition}\label{prop:reg-cond-abstract}
For measurable spaces $\cX, \cY$ and $\cZ:=\cY\times\cX$, consider the class of transition kernels from $\cX$ to $\cZ$:
\begin{align*}
    \cQ=\set{ \QQ: \cX\to\Delta(\cZ) }.
\end{align*}
For any $\QQ\in\cQ$, we define $\QQsq:\cX\to\Delta(\cZ\times\cZ)$ as follows: for any $x_0\in\cX$, $\QQsq(\cdot|x_0)$ is the probability distribution of $(z,z')$, where $z=(Y,x)\sim \QQ(\cdot|x), z'=(Y',x')\sim \QQ(\cdot|x)$.

Suppose that $\QQ_m\in\cQ$ are transition kernels such that for all $m\neq l$,
\begin{align*}
    \DB{ \QQ_m(\cdot|x), \QQ_l(\cdot|x) }\geq 3\log(L/\alpha), \qquad \forall x\in\cX.
\end{align*}
Further assume that there exists $\mu:\cX\to\Delta(\cX)$ such that
\begin{align}\label{eqn:reg-cond-constraint}
    \QQ_m(x|x_0)\geq \alpha\mu(x|x_0),\qquad\forall m\in[L].
\end{align}
Then for any $\QQ\in\cQ$, $x_0\in\cX$, and $p\in\Delta([L])$, we have
\begin{align*}
    \DTV{ \EE_{m\sim p} \QQsq_m(\cdot|x_0), \QQsq(\cdot|x_0) }\geq \frac{\alpha}{32}(1-\max_m p_m).
\end{align*}
\end{proposition}

\begin{proof}
In the following, we fix any given $\QQ\in\cQ$, $x_0\in\cX$, and $p\in\Delta([L])$. Let $\tPP$ be the probability distribution of $(m,z,z')$, where $m\sim p$, $z=(Y,x)\sim \QQ_m(\cdot|x_0)$, and $z'=(Y',x')\sim \QQ_m(\cdot|x_0)$ (i.e. $(z,z')\sim \QQsq_m(\cdot|x_0)$). Also, let $\PP=\EE_{m\sim p} \QQsq_m(\cdot|x_0)$ be the marginal distribution of $(z,z')\sim\tPP$. We also omit $x_0$ from the conditional probabilities when it is clear from the context.

By \cref{lemma:TV-cond}, it holds that
\begin{align*}
    \EE_{(Y,x)\sim \PP}\brac{ \DTV{ \PP(z'=\cdot|Y,x), \QQsq(z'=\cdot|Y,x) } }\leq 2\DTV{\PP,\QQsq}.
\end{align*}
We also have
\begin{align*}
    \EE_{x'\sim \PP}\brac{ \DTV{ \PP(z'=\cdot|x), \QQsq(z'=\cdot|x) } }\leq 2\DTV{\PP,\QQsq}.
\end{align*}
Notice that the conditional distribution $\QQsq(z'=\cdot|Y,x)=\QQ(z'=\cdot|x)$ only depends on $x$, and hence by triangle inequality,
\begin{align*}
    \EE_{(Y,x)\sim \PP}\brac{ \DTV{ \PP(z'=\cdot|Y,x), \PP(z'=\cdot|x) } }\leq 4\DTV{\PP,\QQsq}
\end{align*}
Further notice that
\begin{align*}
    \PP(z'=\cdot|Y,x)=\EE_{m|Y,x}\brac{ \QQ_m(z'=\cdot|x) }, \qquad
    \PP(z'=\cdot|x)=\EE_{m|x}\brac{ \QQ_m(z'=\cdot|x) }.
\end{align*}
Hence, by \cref{lem:DB-mixture-dist}, we have
\begin{align*}
    \DTV{ \PP(z'=\cdot|Y,x), \PP(z'=\cdot|x) }
    \geq \frac{1}{2}\DTV{ \tPP(m=\cdot|Y,x), \tPP(m=\cdot|x) }.
\end{align*}
Next, using the definition of TV distance (which is a $f$-divergence, see e.g. \citet{polyanskiy2014lecture}), we can show that
\begin{align*}
    \EE_{(Y,x)\sim\tPP}\brac{\DTV{ \tPP(m=\cdot|Y,x), \tPP(m=\cdot|x) }}
    =
    \EE_{(m,x)\sim\tPP}\brac{\DTV{ \tPP(Y=\cdot|m,x), \tPP(Y=\cdot|x) }}.
\end{align*}
We know
\begin{align*}
    \PP(Y=\cdot|m,x)=\QQ_m(Y=\cdot|x_0,x), \qquad
    \PP(Y=\cdot|x)=\EE_{m|x}\brac{ \QQ_m(zy=\cdot|x_0,x) }, 
\end{align*}
and hence combining the inequalities above gives
\begin{align}\label{eqn:cond-cond-cond}
    4\DTV{\PP,\QQsq}
    \geq \EE_{(m,x)\sim\PP}\brac{ \DTV{ \QQ_m(Y=\cdot|x_0,x), \EE_{m'|x}\brac{\QQ_{m'}(Y=\cdot|x_0,x)} } }.
\end{align}
Consider the set
\begin{align*}
    \cX_+=\set{ x\in\cX: \DB{\QQ_m(Y=\cdot|x_0,x), \QQ_l(Y=\cdot|x_0,x)}\geq \log L, ~~\forall m\neq l }.
\end{align*}
For any $x\in\cX_+$, by \cref{lem:DB-mixture-dist}, we have
\begin{align*}
    \DTV{ \QQ_m(Y=\cdot|x_0,x), \EE_{m'|x}\brac{\QQ_{m'}(Y=\cdot|x_0,x)} }
    \geq \frac12\paren{1-\tPP(m|x)}.
\end{align*}
Therefore, combining the above inequality with \cref{eqn:cond-cond-cond} gives
\begin{align*}
    4\DTV{\PP,\QQsq}
    \geq&~\EE_{(m,x)\sim\tPP}\brac{ \DTV{ \QQ_m(Y=\cdot|x_0,x), \EE_{m'|x}\brac{\QQ_{m'}(Y=\cdot|x_0,x)} } } \\
    \geq&~
    \frac12\EE_{(m,x)\sim\tPP}\brac{ \indic{x\in\cX_+}\paren{1-\tPP(m|x)} } \\
    \geq&~
    \frac12\EE_{x}\brac{ \indic{x\in\cX_+}\min_{m}\paren{1-\tPP(m|x)} }
\end{align*}
By definition,
\begin{align*}
    1-\tPP(m|x)=\sum_{l\neq m} \tPP(l|x)
    =\frac{\sum_{l\neq m} p_l\QQ_l(x|x_0)}{\PP(x)}.
\end{align*}
Therefore,
\begin{align*}
    \EE_{x}\brac{ \indic{x\in\cX_+}\min_{m}\paren{1-\tPP(m|x)} }
    =&~
    \sum_{x\in\cX_+} \min_{m} \sum_{l\neq m} p_l\QQ_l(x|x_0) \\
    \stackrel{\text{\cref{eqn:reg-cond-constraint}}}{\geq}&~ \sum_{x\in\cX_+} \min_{m} \sum_{l\neq m} p_l\cdot \alpha\mu(x) \\
    =&~ 
    \alpha\mu(\cX_+) (1-\max_m p_m).
\end{align*}
It remains to prove that $\mu(\cX_+)\geq \frac12$. For each pair of $m\neq l$, consider the set
\begin{align*}
    \cX_{m,l}\defeq \set{ x\in\cX: \DB{\QQ_m(Y=\cdot|x_0,x), \QQ_l(Y=\cdot|x_0,x)}< \log L }.
\end{align*}
By definition,
\begin{align*}
    &~\exp\paren{ -\DB{\QQ_m(z=\cdot|x_0), \QQ_l(z=\cdot|x_0)} } \\
    =&~\sum_{x\in\cX} \sqrt{\QQ_m(x|x_0)\QQ_l(x|x_0)}\exp\paren{-\DB{\QQ_m(Y=\cdot|x_0,x), \QQ_l(Y=\cdot|x_0,x)}} \\
    >&~ \sum_{x\in\cX_{m,l}} \sqrt{\QQ_m(x|x_0)\QQ_l(x|x_0)}\cdot \frac{1}{L} \\
    \geq&~ \alpha\mu(\cX_{m,l})\cdot \frac{1}{L}.
\end{align*}
Therefore, by the fact that $\DB{\QQ_m(z=\cdot|x_0), \QQ_l(z=\cdot|x_0)}\geq 3\log(L/\alpha)$, we know that $\mu(\cX_{m,l})\leq \frac{1}{L}$ for all $m\neq l$, and hence
\begin{align*}
    1-\mu(\cX_+)\leq \sum_{m<l} \mu(\cX_{m,l})\leq \frac12.
\end{align*}
The proof is completed by combining the inequalities above.
\end{proof}

\begin{proofof}{\cref{prop:reg-cond-example}}
We only need to demonstrate how to apply \cref{prop:reg-cond-abstract}. We abbreviate $W=W_0$ in the following proof. Take $\cX=\cS$, $\cY=\cA\times(\cS\times\cA)^{W-2}$, with variable $x_0=s_1, Y=(a_1,s_2,\cdots,a_{W-1}), x=s_W$. Let
\begin{align*}
    \QQ_m=\TT_{\theta,m}^{\piexp}((a_1,s_2,\cdots,s_W)=\cdot|s_1=\cdot) \in\cQ, \qquad m\in[L].
\end{align*}
Then, we can identify $\QQsq_m$ as
\begin{align*}
    \QQsq_m=\TT_{\theta,m}^{\piexp}((a_1,s_2,\cdots,s_{2W-1})=\cdot|s_1=\cdot).
\end{align*}
We also have $\QQ_m(x|x_0)=\TT_{\theta,m}^{\piexp}(s_W=s'|s_0=s)$. Therefore, we can indeed apply \cref{prop:reg-cond-abstract} and the proof is hence completed. %
\end{proofof}

\section{Proofs for Section~\ref{sec:comp}}\label{appdx:comp}

\subsection{Proof of Theorem~\ref{thm:plan}}\label{appdx:proof-plan}

We first prove the following lemma.

\begin{lemma}\label{lem:plan-optimism}
Suppose that the policy $\hpi$ is returned by \cref{alg:plan}. Then for any policy $\pi$, it holds that
\begin{align*}
    V(\hpi)\geq V(\pi)-\PP^{\pi}( \ms \neq m(\otau_W) ).
\end{align*}
\end{lemma}

\begin{proof}
For any policy $\pi$ and trajectory $\otau_h$, we consider the value $\pi$ given the trajectory $\otau_h$:
\begin{align*}
    V^{\pi}(\otau_h):=\EE^{\pi} \brcond{ \sum_{h'=h}^H R_h(s_h,a_h) }{ \otau_h }.
\end{align*}
In particular, for trajectory $\otau_W=(s_1,a_1,\cdots,s_W)$, we have
\begin{align*}
    V^{\pi}(\otau_W)
    =&~\EE^{\pi}\brcond{\sum_{h=W}^H R_h(s_h,a_h)}{\otau_W} \\
    =&~ \sum_{m\in[L]} \tPP(m|\otau_W) \cdot \EE_m^{\pi(\cdot|\otau_W)}\brcond{\sum_{h=W}^H R_h(s_h,a_h)}{s_W},
\end{align*}
where the expectation $\EE_m^{\pi(\cdot|\otau_W)}$ is taken over the probability distribution of $(s_{W+1:H},a_{W:H})$ induced by executing the policy $\pi(\cdot|\otau_W)$ in MDP $M_m$ with starting state $s_W$. Therefore, because $\Vind{m}{W}$ is exactly the optimal value function in MDP $M_m$ (at step $W$), we know that
\begin{align*}
    \EE_m^{\pi(\cdot|\otau_W)}\brcond{\sum_{h=W}^H R_h(s_h,a_h)}{s_W} \leq \Vind{m}{W}(s_W).
\end{align*}
Hence, we have
\begin{align*}
    V^{\pi}(\otau_W)
    \leq&~ \sum_{m\in[L]} \tPP(m|\otau_W) \Vind{m}{W}(s_W) \\
    \leq&~ \tPP(m(\otau_W)|\otau_W)\cdot\Vind{m(\otau_W)}{W}(s_W) + \sum_{m\neq m(\otau_W)} \tPP(m|\otau_W) \\
    =&~ \hV(\otau_W)+\tPP(\ms\neq m(\otau_W)|\otau_W),
\end{align*}
where the last line follows from the definition of $\hV$ in \cref{alg:plan}. On the other hand, we also have
\begin{align*}
    V^{\hpi}(\otau_W)
    =&~ \sum_{m\in[L]} \tPP(m|\otau_W) \cdot \EE_m^{\hpi(\cdot|\otau_W)}\brcond{\sum_{h=W}^H R_h(s_h,a_h)}{s_W} \\
    \geq&~ \tPP(m(\otau_W)|\otau_W) \cdot \EE_{m(\otau_W)}^{\hpi(\cdot|\otau_W)}\brcond{\sum_{h=W}^H R_h(s_h,a_h)}{s_W} \\
    =&~ \tPP(m(\otau_W)|\otau_W) \cdot \EE_{m(\otau_W)}\brcond{\sum_{h=W}^H R_h(s_h,a_h)}{s_W, \text{for each }h\geq W, a_h=\pi^{(m(\otau_W))}_h(s_h)} \\
    =&~ \tPP(m(\otau_W)|\otau_W) \cdot \Vind{m(\otau_W)}{W}(s_W) 
    =\hV(\otau_W),
\end{align*}
where the last line is because $\Vind{m}{W}(s_W)$ is exactly the expected cumulative reward if the agent starts at step $W$ and state $s_W$, and executes $\pi_m$ afterwards.
Combining the inequalities above, we obtain
\begin{align*}
    V^{\pi}(\otau_W)-\tPP(\ms\neq m(\otau_W)|\otau_W) \leq V^{\hpi}(\otau_W).
\end{align*}
By recursively using the definition of $\hpi$, we can show that for each step $h=W,W-1,\cdots,1$,
\begin{align*}
    V^{\pi}(\otau_h)-\tPP(\ms\neq m(\otau_W)|\otau_h) \leq V^{\hpi}(\otau_h).
\end{align*}
The desired result follows as
\begin{align*}
    V(\pi)-\tPP^\pi(\ms\neq m(\otau_W)) = \EE\brac{V^{\pi}(\otau_1)-\tPP(\ms\neq m(\otau_W)|\otau_1)}
    \leq \EE\brac{V^{\hpi}(\otau_1)}=V(\hpi).
\end{align*}
\end{proof}

\paragraph{Proof of \cref{thm:plan}}
Let $\pis$ be an optimal policy such that $M$ is $\om$-separated under $\pis$. By \cref{prop:latent-MLE}, we know that $\PP^{\pis}(\ms\neq m(\otau_W))\leq L\exp(-\om(W))\leq \eps$. Therefore, \cref{lem:plan-optimism} implies $V(\hpi)\geq V(\pis)-\eps=V^\star-\eps$. The time complexity follows immediately from the definition of \cref{alg:plan}.
\qed

\subsection{Embedding 3SAT problem to LMDP}\label{appdx:3SAT-to-LMDP}

\newcommand{\smh}[1]{s^{#1}_{\ominus}}
\newcommand{\tMp}{\widetilde{M}_{\Phi}}

Suppose that $\Phi$ is a 3SAT formula with $n$ variables $x_1,\cdots,x_n$ and $N$ clauses $C_1,\cdots,C_N$, and $\cA=\set{0,1}^w$. Consider the corresponding LMDP $M_\Phi$ constructed as follows.
\begin{itemize}
\item The horizon length is $H=\ceil{n/w}+1$.
\item The state space is $\cS=\set{ \smh{1},\smh{2},\cdots,\smh{H-1},\sp }$, and the action space is $\cA$.
\item $L=N$, and the mixing weight is $\rho=\unif([N])$.
\item For each $m\in[N]$, the MDP $M_m$ is given as follows. 
\begin{itemize}
    \item The initial state is $\smh{1}$.
    \item At state $\smh{h}$, taking action $a\in\cA_{m,h}$ leads to $\sp$, where
    \begin{align*}
        \cA_{m,h}\defeq&~ \set{ a\in\set{0,1}^w: \text{$\exists j\in[w]$ such that $a[j]=1$ and the clause $C_m$ contains $x_{w(h-1)+j}$} } \\
        &~ \bigcup \set{ a\in\set{0,1}^w: \text{$\exists j\in[w]$ such that $a[j]=0$ and the clause $C_m$ contains $\neg x_{w(h-1)+j}$} }.
    \end{align*}
    For action $a\not\in\cA_{m,h}$, taking action $a$ leads to $\smh{\min\set{h+1,H-1}}$.
\end{itemize}
\item The reward function is given by $R_h(s,a)=\indic{s=\sp,h=H}$.
\end{itemize}

The basic property of $M_\Phi$ is that, the optimal value of the LMDP $M_\Phi$ encodes the satisfiability of the formula $\Phi$. More concretely, if taking an action sequence $a_{1:H-1}$ leads to $\sp$ for all $m\in[N]$, then the first $n$ bits of the sequence $(a_1,\cdots,a_{H-1})$ gives a satisfying assignment of $\Phi$. Conversely, any satisfying assignment of $\Phi$ gives a corresponding action sequence such that taking it leads to $\sp$ always. On the other hand, if $\Phi$ is not satisfiable, then for any action sequence $a_{1:H-1}$, there must be a latent index $m\in[N]$ such that taking $a_{1:H-1}$ leads to $\smh{H-1}$ in MDP $M_m$. To summarize, we have the following fact.

\textbf{Claim.} The optimal value $V^\star$ of $M_\Phi$ equals 1 if and only if $\Phi$ is satisfiable. Furthermore, when $\Phi$ is not satisfiable, $V^\star\leq 1-\frac{1}{m}$.

Based on the LMDP $M_\Phi$ described above, we construct a ``perturbed'' version $\tMp$ that is $\delta$-strongly separated.
\begin{itemize}
\item Pick $d=\ceil{11 \log (2N)}$ and invoke \cref{lem:compute-net-2} to generates a sequence $\bx_1,\bx_2,\cdots,\bx_N\in\set{-1,+1}^d$, such that for all $i\neq j$, $i,j\in[N]$,
\begin{align*}
    &\lone{\bx_i-\bx_j}\geq \frac{d}{2}, \qquad
    \lone{\bx_i+\bx_j}\geq \frac{d}{2}.
\end{align*}
We also set $\odelta=4\delta$, and for each $m\in[N]$, we define %
\begin{align*}
    \mu_m^+=&~\brac{ \frac{1+\odelta\bx_m[1]}{2d}; \frac{1-\odelta\bx_m[1]}{2d}; \cdots; \frac{1+\odelta\bx_m[d]}{2d}; \frac{1-\odelta\bx_m[d]}{2d} } \in\Delta([2d]), \\
    \mu_m^-=&~\brac{ \frac{1-\odelta\bx_m[1]}{2d}; \frac{1+\odelta\bx_m[1]}{2d}; \cdots; \frac{1-\odelta\bx_m[d]}{2d}; \frac{1+\odelta\bx_m[d]}{2d} } \in\Delta([2d]).
\end{align*}
\item The state space is $\tcS=\cS\times [2d]$, the action space is $\cA$, and the horizon length is $H$.
\item $L'=2N$, and the mixing weight is $\rho'=\unif([2N])$
\item For each $m\in[N]$, we set $\tM_{2m-1}=M_{m}\otimes \mu_{m}^+$ and $\tM_{2m}=M_{m}\otimes \mu_{m}^-$ (recall our definition in \cref{def:MDP-tensor}).
\item The reward function is given by $R_h((s,o),a)=\indic{s=\sp,h=H}$.
\end{itemize}

\begin{proposition}\label{prop:3SAT-log-eps}
In the LMDP $\tMp$ described above, for any policy class $\Pi$ that contains $\cA^{H}$, we have
\begin{align*}
    \max_{\pi\in\Pi} V(\pi)=\begin{cases}
        1, & \Phi\text{ is satisifiable},\\
        \leq 1 - \frac{(1-\odelta^2)^{(H-1)/2}}{N}, & \text{otherwise}.
    \end{cases}
\end{align*}
\end{proposition}

\newcommand{\goode}{\ts_H[1]=\sp} %
\newcommand{\bade}{\ts_H[1]=\smh{H-1}} %
\begin{proof}
By our construction, regardless of the actions taken, we always have $\goode$ or $\bade$.
Therefore, for any policy $\pi$,
\begin{align*}
    V(\pi)
    =
    \PP^{\pi}(\goode) 
    =
    1-\PP^{\pi}(\bade).
\end{align*}
By construction, any reachable trajectory that ends with $\bade$ must take the form %
\begin{align*}
    (\smh{1},o_1),a_1,\cdots,(\smh{H-1},o_{H-1}),a_{H-1}, (\smh{H-1},o_H).
\end{align*}
Further, for each $m\in[N]$, in the MDP $\tM_{2m-1}$ and $\tM_{2m}$, $\bade$ if and only if $a_{1:H-1}\not\in\Asat{m}$, where we define
\begin{align*}
    \Asat{m}=\set{ a_{1:H-1}\in\cA^{H-1}: \text{for some }h\in[H-1], a_h\in\cA_{m,h} }\subset \cA^{H-1}.
\end{align*}
Therefore, for any reachable trajectory $\tau_{H-1}$ that leads to $\bade$, we have
\begin{align*}
    \tau_{H-1}=&~ ((\smh{1},o_1),a_1,\cdots,(\smh{H-1},o_{H-1}),a_{H-1}), \\
    \PP^{\pi}(\tau_{H-1})
    =&~
    \frac{1}{2N}\sum_{l=1}^{2N} \PP_{\tM_l}^{\pi}(\tau_{H-1}) \\
    =&~
    \frac{1}{N}\sum_{m=1}^N \indic{a_{1:H-1}\in\Asat{m}}\cdot \pi(\tau_{H-1}) \cdot \paren{ \prod_{h=1}^{H-1} \mu_m^+(o_h) + \prod_{h=1}^{H-1} \mu_m^-(o_h) }
\end{align*}
where by convention we write
\begin{align*}
    \pi(\tau_{H-1})=\prod_{h=1}^{H-1} \pi(a_h|(\smh{1},o_1),a_1,\cdots,(\smh{h},o_{h})),
\end{align*}
and we abbreviate this quantity as $p_{\pi}(a_{1:H}|o_{1:H})$. Then, we have
\begin{align*}
    1-V(\pi)
    =&~
    \PP^{\pi}(\bade) \\
    =&~
    \sum_{\substack{\text{reachable $\tau_{H-1}$ that} \\ \text{leads to }\bade }} \PP^{\pi}(\tau_{H-1})\\
    =&~\sum_{(o_{1:H-1},a_{1:H-1})} \frac{1}{2m}\sum_{i=1}^m  \indic{a_{1:H-1}\in\Asat{m}}\cdot p_{\pi}(a_{1:H}|o_{1:H}) \cdot \paren{ \prod_{h=1}^{H-1} \mu_m^+(o_h) + \prod_{h=1}^{H-1} \mu_m^-(o_h) }.
\end{align*}
By \cref{lem:delta-square}, it holds that
\begin{align*}
    \prod_{h=1}^{H-1} \mu_m^+(o_h) + \prod_{h=1}^{H-1} \mu_m^-(o_h)\geq \frac{2(1-\odelta^2)^{\floor{(H-1)/2}}}{(2d)^{H-1}}.
\end{align*}
Hence, we have
\begin{align*}
    1-V(\pi)
    \geq
    &~\frac{1}{m}\sum_{i=1}^m \sum_{(o_{1:H-1},a_{1:H-1})} \indic{a_{1:H-1}\in\Asat{m}}\cdot p_{\pi}(a_{1:H}|o_{1:H}) \cdot  \frac{2(1-\odelta^2)^{\floor{(H-1)/2}}}{(2d)^{H-1}} \\
    =&~ (1-\odelta^2)^{\floor{(H-1)/2}} \sum_{a_{1:H-1}} \frac{\#\set{m\in[N]: a_{1:H-1}\not\in\Asat{m}}}{N} \times \frac{1}{(2d)^H}\sum_{o_{1:H-1}} p_{\pi}(a_{1:H}|o_{1:H}) \\
    \geq&~ (1-\odelta^2)^{\floor{(H-1)/2}} \cdot  \min_{a_{1:H-1}} \frac{\#\set{m\in[N]: a_{1:H-1}\not\in\Asat{m}}}{N} \cdot \sum_{a_{1:H-1}} \frac{1}{(2d)^H}\sum_{o_{1:H-1}} p_{\pi}(a_{1:H}|o_{1:H})\\
    =&~ (1-\odelta^2)^{\floor{(H-1)/2}} \cdot  \min_{a_{1:H-1}} \frac{\#\set{m\in[N]: a_{1:H-1}\not\in\Asat{m}}}{N}, 
\end{align*}
where the last line is because
\begin{align*}
    \sum_{a_{1:H-1}} \sum_{o_{1:H-1}} p_{\pi}(a_{1:H}|o_{1:H}) = (2d)^H.
\end{align*}
Therefore, if $\Phi$ is not satisfiable, then for any action sequence $a_{1:H}$, there must exist $m\in[N]$ such that $a_{1:H}\not\in\Asat{m}$. This is because if $a_{1:H}\in\Asat{m}$ for all $m\in[N]$, then the first $n$ bits of the sequence $(a_1,\cdots,a_{H-1})$ gives a satisfying assignment of $\Phi$. Thus, in this case, for any policy $\pi$,
\begin{align*}
    1-V(\pi)\geq \frac{(1-\odelta^2)^{\floor{(H-1)/2}}}{m}.
\end{align*}
On the other hand, if $\Phi$ is satisfiable, then there is an action sequence $a_{1:H-1}\in\Asat{m}$ for all $m\in[N]$, and hence $V(a_{1:H-1})=1$. Combining these complete the proof.
\end{proof}

\begin{lemma}\label{lem:delta-square}
    For any reals $\lambda_1,\cdots,\lambda_k\in[-1,1]$ and $\delta\in[0,1)$, it holds that
    \begin{align*}
        \prod_{i=1}^k (1+\delta\lambda_i)+\prod_{i=1}^k (1-\delta\lambda_i)\geq 2(1-\delta^2)^{\floor{k/2}}.
    \end{align*}
\end{lemma}

\begin{proof}
    Notice that the LHS is a linear function of $\lambda_i$ for each $i$ (fixing other $\lambda_j$'s). Therefore, we only need to consider the case $\lambda_i\in\set{-1,1}$. Suppose that $\lambda_1,\cdots,\lambda_k$ has $r$ many 1's and $s$ many $-1$'s ($r+s=k$), and w.l.o.g $r\geq s$. Then for $t=r-s\geq 0$,
    \begin{align*}
        \prod_{i=1}^k (1+\delta\lambda_i)+\prod_{i=1}^k (1-\delta\lambda_i)
        =&~(1+\delta)^r(1-\delta)^s+(1+\delta)^s(1-\delta)^r\\
        =&~(1-\delta^2)^s\brac{ (1+\delta)^t+(1-\delta)^t }\\
        \geq&~ 2(1-\delta^2)^s\geq 2(1-\delta^2)^{\floor{k/2}}.
    \end{align*}
\end{proof}

\subsection{Proof of Proposition~\ref{prop:NP-hard}}

Suppose that a 3SAT formula $\Phi$ with $n$ variables and $N$ clauses are given. Then, we can pick $w=1$, $\delta=\frac{1}{\sqrt{n}}$, $\eps=\frac{c}{N}$ for some small constant $c$, and the LMDP $\tMp$ constructed above has $H=n+1$, $L=2N$, $S=Hd$, $A=2$, and it is $\delta$-strongly separated. Further, we have $\max\set{L,S,A,H,\eps^{-1},\delta^{-1}}\leq \cO(n+N)$, and $\tMp$ can be computed in $\poly(n,N)$ time.
Therefore, if we can solve any given $\delsep$-strong separated LMDP in polynomial time, we can determine the satisfiability of any given 3SAT formula $\Phi$ in polynomial time by solving $\tMp$, which implies that NP=P.
\qed

\subsection{Proof of Theorem~\ref{thm:comp-log-eps}}\label{appdx:proof-comp-log-eps}

\newcommand{\epsv}[1]{\eps_t}
\newcommand{\logA}[1]{\floor{\log A_t}}
\newcommand{\logL}[1]{\log L_t}
\newcommand{\delv}[1]{\delta_t}

Suppose that there is an algorithm $\fA$ that contradicts the statement of \cref{thm:comp-log-eps}.

Fix a given 3-SAT formula $\Phi$ with $n$ variables and $N$ clauses is given (we assume $N\leq n^3$ without loss of generality), we proceed to determine the satisfiability of $\Phi$ in $2^{o(n)}$-time using $\fA$.

Pick $t=t_n\in\NN$ to be the minimal integer such that %
\begin{align}\label{eqn:pick-t-eps}
    200n\leq \frac{\log(1/\epsv{t})\cdot\logA{t}}{\delv{t}^2}.
\end{align}
We then consider $\eps=\epsv{t}$, $w=\logA{t}$, $A=2^{w}$, $\delsep=\frac{1}{\delv{t}}$, and $\cA=\set{0,1}^{w}$. %

Now, consider the LMDP $\tMp$ constructed in \cref{appdx:3SAT-to-LMDP} based on $(\Phi,\cA,\delta)$. We know that $\tMp$ is $\delta$-strongly separated, and we also have
\begin{align*}
    L=2N\leq 2n^3, \qquad
    S=nd\leq \cO(n\log n), \qquad
    H=\ceil{\frac{n}{w}}+1\leq n+1.
\end{align*}
In the following, we show that \cref{eqn:comp-log-eps-constraints} and \cref{eqn:pick-t-eps} (with suitably chosen $C$) ensures that
\begin{align*}
    \eps<\eps'\defeq \frac{(1-\odelta^2)^{(H-1)/2}}{3N}.
\end{align*}
By definition,
\begin{align*}
    \log(1/\eps')= \frac{(H-1)\log\frac{1}{1-\odelta^2}}{2}+\log(3N)
    \leq \frac{2\odelta^2}{1-\odelta^2} \ceil{\frac{n}{w}}+\log(3N)
    \leq \frac{128\delta^2}{3} \frac{n}{w}+3\log(n)+4.
\end{align*}
Therefore, by \cref{eqn:pick-t-eps}, we have $\log(1/\eps')<\log(1/\eps)$ if we have $\frac{3}{4}\log(1/\eps)>3\log n + 4$, or equivalently $e^6n^4\leq \eps^{-1}$. This is indeed insured by \cref{eqn:comp-log-eps-constraints}. 

Next, consider running $\fA$ on $(\tMp,\eps)$, and let $\hV$ be the value returned by $\fA$. By \cref{prop:3SAT-log-eps}, we have the follow facts: (a) If $\hV\geq 1-\eps$, then $\Phi$ is satisfiable. (b) If $\hV<1-\eps$, then $\Phi$ is not satisfiable. Therefore, we can use $\fA$ to determine the satisfiability of $\Phi$ in time $\Alogeps+\poly(n)$.
Notice that our choice of $t$ ensures that $\log(1/\epsv{t})w\delv{t}^{-2}\leq 3200n$, and hence we actually determine the satisfiability of $\Phi$ in $2^{o(n)}$-time, which contradicts \cref{ETH}.
\qed

\subsection{Proof of Theorem~\ref{thm:comp-log-L}}\label{appdx:proof-comp-log-L}

Suppose that there is an algorithm $\fA$ that contradicts the statement of \cref{thm:comp-log-L}.

Fix a given 3-SAT formula $\Phi$ with $n$ variables and $N$ clauses is given (we assume $N\leq n^3$ without loss of generality), we proceed to determine the satisfiability of $\Phi$ in $2^{o(n)}$-time using $\fA$.

Pick $t=t_n\in\NN$ to be the minimal integer such that
\begin{align}\label{eqn:pick-t}
    Cn\ceil{\log_2 N}\leq \frac{\logL{t}\cdot\logA{t}}{\delv{t}^2},
\end{align}
where $C$ is a large absolute constant.
We then consider $L=2^{\logL{t}}$, $w=\logA{t}$, $A=2^{w}$, $\delsep=\frac{1}{\delv{t}}$, and $\cA=\set{0,1}^{w}$. 

Let $M_\Phi$ be the LMDP with action set $\cA$, horizon $H=\ceil{n/w}+1$ constructed in \cref{appdx:3SAT-to-LMDP}.

Further, we choose $r=\ceil{\log_2 N}$, $d=\floor{\frac{\logL{t}}{r}}$. By our choice \cref{eqn:pick-t}, we can ensure the presumption $d\geq C_0H\delta^2$ of \cref{lem:compute-family} holds, which implies that we can construct a $(N,H,\delta,r2^{-c_0d},2^{dr})$-family over $[2d]^r$ in time $\poly(2^{dr})\leq\poly(L)$. Denote $\cQ$ be such a family, and we consider $M_\Phi\otimes\cQ$, which is a \sepstr~LMDPs family with $S=(2d)^rH$ and hence $\log S\leq \cO(\log\logL{t})$ by \cref{eqn:comp-log-L-constraints} (because $n\leq \poly\logL{t}$ using \cref{eqn:pick-t}).

Consider running $\fA$ on $M_\Phi\otimes\cQ$ with $\eps=\frac{1}{3N}$, and let $\hV$ be the value returned by $\fA$. Let $V_\Phi$ be the optimal value of $M_\Phi$, $V_{M,\Phi}$ be the optimal value of $M_\Phi\otimes\Phi$. Then by \cref{prop:property-tensor}, it holds that
\begin{align*}
    V_\Phi\leq V_{M,\Phi} \leq r2^{-c_0d}+V_\Phi.
\end{align*}
Hence, as long as $r2^{-c_0d}<\frac{1}{3N}$ (which is ensured by condition \cref{eqn:comp-log-L-constraints}), we have the follow facts: (a) If $V_\Phi=1$, then $\hV\geq 1-\frac{1}{3N}$. (b) If $V_\Phi\leq 1-\frac{1}{N}$, then $\hV<1-\frac{1}{3N}$. Notice that a special case of \cref{prop:3SAT-log-eps} is that, when $\Phi$ is satisfiable, then $V_\Phi=1$, and otherwise $V_\Phi\leq 1-\frac{1}{N}$. Therefore, we can use $\fA$ to determine the satisfiability of $\Phi$ in time $\AlogL+\poly(L)$.
Notice that our choice of $t$ ensures that $(\logL{t})(\logA{t})\delv{t}^{-2}\leq 16Cn\ceil{\log_2 N}$, and hence $\log L=o(n)$, and
\begin{align*}
    \frac{ \log A \log L }{ \delta^2 \log\log L } = \cO(n).
\end{align*}
Therefore, given $\fA$, we can construct a $2^{o(n)}$-time algorithm for 3SAT, a contradiction.
\qed

\subsection{Technical lemmas}

\begin{lemma}\label{lem:compute-net}
There is a procedure such that, for any input integer $N\geq 2$ and $d\geq \ceil{11\log N}$, compute a sequence $\bx_1,\cdots,\bx_N\in\set{-1,+1}^d$ such that $\lone{\bx_i-\bx_j}\geq \frac{d}{2} \forall i\neq j$, with running time $\poly(2^d)$.
\end{lemma}

\begin{proof}
Consider the following procedure: We maintain two set $\cU,\cV$, and we initialize $\cU=\set{}, \cV=\set{-1,1}^d$. At each step, we pick a $\bx\in\cV$, add $\bx$ to $\cU$, and remove all $\by\in\cV$ such that $\lone{\by-\bx}<\frac{d}{2}$. The procedure ends when $\cV$ is empty or $\abs{\cU}=N$. 

We show that this procedure must end with $\abs{\cU}=N$. Notice that for any $\bx, \by \in\set{-1,1}^d$, we have $\lone{\bx-\by}<\frac{d}{2}$ only when $\bx, \by$ differs by at most $i<\frac{d}{4}$ coordinates. Therefore, at each step, we remove at most
\begin{align*}
    M=\sum_{i=0}^{\ceil{d/4}-1} \binom{d}{i}
\end{align*}
elements in $\cV$. Hence, it remains to show that $\frac{2^d}{M}\geq N$.

Denote $k=\ceil{d/4}-1$. Then we have
\begin{align*}
    M=\sum_{i=0}^{d} \binom{d}{i}
    \leq \paren{\frac{ed}{k}}^k 
    \leq \paren{\frac{ed}{d/4}}^{d/4}
    =\exp\paren{ \frac{1+2\log 2}{4}d },
\end{align*}
and hence $\frac{2^d}{M}>\exp(d/11) \geq N$ as claimed.
\end{proof}

Repeating the argument above, we can also prove the following result.
\begin{lemma}\label{lem:compute-net-2}
There is a procedure such that, for any input integer $N\geq 2$ and $d\geq \ceil{11\log (2N)}$, compute a sequence $\bx_1,\cdots,\bx_N\in\set{-1,+1}^d$ such that for any $i\neq j$, 
\begin{align*}
    \lone{\bx_i-\bx_j}\geq \frac{d}{2}, \qquad
    \lone{\bx_i+\bx_j}\geq \frac{d}{2}
\end{align*}
with running time $\poly(2^d)$.
\end{lemma}

\begin{lemma}\label{lem:compute-family}
There is a procedure such that, for any input $r,d,H\geq 2$ and $\delta\in(0,\frac{1}{4}]$ satisfying $d\geq C_0H\delta^2$,
compute a $(2^r,H,\delta,\gamma,2^{dr})$-family over $[2d]^r$, with $\gamma\leq r2^{-c_0d}$, with running time $\poly(2^{dr})$.
\end{lemma}

\begin{proof}
We first invoke the procedure of \cref{lem:compute-net} to compute $\bx_1,\cdots,\bx_N\in\set{-1,1}^d$ such that $\lone{\bx_i-\bx_j}\geq \frac{d}{2}$ and $N> \exp(d/11)$. Consider the distribution $\mu_i=\QQ_{\odelta\bx_i}\in\Delta([2d])$ for each $i\in[N]$, where we set $\odelta=4\delta$. Clearly, we have $\DTV{\mu_i,\mu_j}\geq \delta$ for $i\neq j$.

Notice that for $K=\ceil{d/60}$, we have $N>\binom{K+d-1}{d}+1$, and hence by \cref{cor:prob-matching}, there exists $\xi_0,\xi_1\in\Delta([N])$ such that $\supp(\xi_0)\cup\supp(\xi_1)=\emptyset$ and
\begin{align*}
    \DTVt{ \EE_{i\sim \xi_0}\brac{ \mu_i^{\otimes n} }, \EE_{i\sim \xi_1}\brac{ \mu_i^{\otimes n} } }
    \leq \sum_{k=K}^H \paren{\frac{eH\odelta^2}{K}}^{k}.
\end{align*}
Therefore, as long as $d\geq 120eH\odelta^2$, $\cQ=\set{(\xi_0,\xi_1),(\mu_1,\cdots,\mu_N)}$ is a $(2,H,\delta,2^{-\frac{K-1}{2}},N)$-family over $[2d]$. Further, invoking \cref{lem:family-tensor} yields $\cQ'$, a $(2^r,H,\delta,r2^{-\frac{K-1}{2}},N^r)$-family over $[2d]^r$. 

By the proof of \cref{cor:prob-matching}, $\xi_0,\xi_1$ can be computed in $\poly(N)$ time, and $\cQ'$ can also be computed from $\cQ$ in time $\poly(2^{dr})$ by going through the proof of \cref{lem:family-tensor}. Combining the results above completes the proof.
\end{proof}

\end{document}